\documentclass[12pt,reqno]{amsart}

\usepackage{amsthm,amsmath,amsfonts,amssymb}
\usepackage{mathrsfs}
\usepackage{ascmac}
\usepackage{bm}
\usepackage[numbers]{natbib}
\usepackage[dvipdfmx]{graphicx}
\usepackage{graphicx,here,comment,subcaption}
\usepackage{amsaddr}
\usepackage{fancyhdr}
\usepackage{url}
\usepackage{booktabs}
\usepackage{mathrsfs}
\usepackage{fullpage}
\usepackage{hyperref}					
\hypersetup{colorlinks,					
	linkcolor=blue,
	citecolor=blue}
\newtheorem{theorem}{Theorem}
\newtheorem{corollary}{Corollary}
\newtheorem{lemma}{Lemma}

\theoremstyle{definition}

\newtheorem{remark}{Remark}
\newtheorem{condition}{Condition}

\usepackage{amsaddr}

\title{Convergence analysis of wide shallow neural operators within the framework of Neural Tangent Kernel}
\author{Xianliang Xu$^1$, Ye Li$^{2}$ \and Zhongyi Huang$^1$} 
\address{1 Tsinghua University, Beijing, China. \\
	2 Nanjing University of Aeronautics and Astronautics, Nanjing, China.}

\thanks{$^*$ This work was partially supported by the NSFC Projects No. 12025104, 11871298, 81930119, 62106103.}

\begin{document}
	\maketitle

\begin{abstract}
Neural operators are aiming at approximating operators mapping between Banach spaces of functions, achieving much success in the field of scientific computing. Compared to certain deep learning-based solvers, such as Physics-Informed Neural Networks (PINNs), Deep Ritz Method (DRM), neural operators can solve a class of Partial Differential Equations (PDEs). Although much work has been done to analyze the approximation and generalization error of neural operators, there is still a lack of analysis on their training error. In this work, we conduct the convergence analysis of gradient descent for the wide shallow neural operators and physics-informed shallow neural operators within the framework of Neural Tangent Kernel (NTK). The core idea lies on the fact that over-parameterization and random initialization together ensure that each weight vector remains near its initialization throughout all iterations, yielding the linear convergence of gradient descent. In this work, we demonstrate that under the setting of over-parametrization, gradient descent can find the global minimum regardless of whether it is in continuous time or discrete time.
\end{abstract}
	
\section{Introduction}

Partial Differential Equations (PDEs) are essential for modeling a wide range of phenomena in physics, biology, and engineering. Nonetheless, the numerical solution of PDEs has always been a significant challenge in the field of scientific computation. Traditional numerical approaches, such as finite difference, finite element, finite volume, and spectral methods, can encounter difficulties due to the curse of dimensionality when applied to PDEs with a high number of dimensions. In recent years, the impressive achievements of deep learning in various domains, including computer vision, natural language processing, and reinforcement learning, have led to an increased interest in utilizing machine learning techniques to tackle PDE-related problems.

For scientific problems, neural network-based methods are primarily divided into two categories: neural solvers and neural operators. Neural solvers, such as PINNs \cite{14}, DRM \cite{15}, utilize neural networks to represent the solutions of PDEs, minimizing some form of residual to enable the neural networks to approximate the true solutions closely. There are two potential advantages of neural solvers. First, this is an unsupervised learning approach, which means it does not require the costly process of obtaining a large number of labels as in supervised learning. Second, as a powerful represention tool, neural networks are known to be effictive for approximating continuous functions \cite{16}, smooth functions \cite{17}, Sobolev functions \cite{18}. This presents a potentially viable avenue for addressing the challenges of high-dimensional PDEs. Nevertheless, existing neural solvers face several limitations when compared to classical numerical solvers like FEM and FVM, particularly in terms of accuracy, and convergence issues. In addition, neural solvers are typically limited to solving a fixed PDE. If certain parameters of the PDE change, it becomes necessary to retrain the neural network.

Neural operator (also called operator learning) aims to approximate unknown operator, which often takes the form of the solution operator associated with a differential equation. Unlike most supervised learning methods in the field of machine learning, where both inputs and outputs are of finite dimensions, operator learning can be regarded as a form of supervised learning in function spaces. Because the inputs and outputs of neural networks are of finite dimensions, some operator learning methods, such as PCA-net and DeepONet, use an encoder to convert infinite-dimensional inputs into finite-dimensional ones, and a decoder to convert finite-dimensional outputs back into infinite-dimensional outputs. The PCA-Net architecture was proposed as an operator learning framework in \cite{19}, where principal component analysis (PCA) is employed to obtain data-driven encoders and decoders, combining with a neural network mapping between the finite-dimensional latent spaces. Building on early work by \cite{4}, DeepONet \cite{20} consists of a deep neural network
for encoding the discrete input function space  and another deep neural network for encoding the domain of the output functions. The encoding network is conventionally
referred to as the “branch-net”, while the decoding network is referred to as the “trunk-net”.  In contrast to the PCA-Net and DeepONet architectures mentioned above,, Fourier Neural Operator (FNO), introduced in \cite{21}, does not follow the encoder-decoder-net paradigm. Instead, FNO is a composition of linear integral operators and nonlinear activation functions, which can be seen as a generalization the structure of finite-dimensional neural networks to a function space setting.  

The theoretical research on neural operators mostly focuses on the study of approximation errors and generalization errors. As we know, the theoretical basis for the application of neural networks lies in the fact that neural networks are universal approximators. This also holds for neural operators. Regarding the analysis of approximation errors in neural operators, the aim is to identify whether neural operator also possess a universal approximation property, i.e. the ability to approximate a wide class of operators to any
given accuracy. As shown in \cite{4}, (shallow) operator networks can approximate continuous operators mapping between spaces of continuous
functions with arbitrary accuracy. Building on this result, DeepONets have also been proven to be universal approximators. For neural operators following the encoder-decoder paradigm like DeepONet and PCA-Net, Lemma 22 in \cite{22} provides a consistent approximation result, which states that if two Banach spaces have the approximation property, then continuous maps between them can be approximated in a finite-dimensional manner. The universal approximation capability of the FNO was initially established in \cite{23}, drawing on concepts from Fourier analysis, and specifically leveraging the density of Fourier series to demonstrate the FNO's ability to approximate a broad spectrum of operators. For a more quantitative analysis of the approximation error of neural operators, see \cite{3}. In addition to approximation errors, the error analysis of encoder-decoder style neural operators also includes encoding and reconstruction errors. \cite{1} has provided both lower and upper bounds on the
total error for DeepONets by using the spectral decay properties of the covariance operators associated with the underlying measures.  By employing tools from non-parametric regression, \cite{2} has provided an analysis of the generalization error for neural operators with basis encoders and decoders. The results in \cite{2} holds on neural operators with some popular encoders and decoders, such as those using Legendre polynomials, trigonometric functions, and PCA. For more details on the recent advances and theoretical research in operator learning, refer to the review \cite{3}.

Up to this point, the theoretical exploration of the convergence and optimization aspects of neural operators has received relatively little attention. To our best knowledge, only \cite{5} and \cite{6} have touched upon the optimization of neural operators. Based on restricted strong convexity (RSC), \cite{5} has presented a unified framework for gradient descent and apply the framework to DeepONets and FNOs, establishing convergence guarantees for both. \cite{6} has briefly analyzed the training of physics-informed DeepONets and derived a weighting scheme guided by NTK theory to balance the data and the PDE residual terms in the loss function. In this paper, we focus on the training error of shallow neural operator in \cite{4} with the framework of NTK, showing that gradient descent converges at a global linear rate to the global optimum. 

\subsection{Notations}
We denote $[n]=\{1,2,\cdots, n\}$ for $n\in \mathbb{N}$. Given a set $S$, we denote the uniform distribution on $S$ by $Unif\{S\}$. We use $I\{E\}$ to denote the indicator function of the event $E$. We use $A\lesssim B$ to denote an estimate that $A\leq cB$, where $c$ is a universal constant. A universal constant means a constant independent of any variables.

\section{Preliminaries}
The neural operator considered in this paper was originally introduced in [4], aimining to approximate a non-linear operator. Specifically, suppose that $\sigma$ is a continuous and non-polynomial function, $X$ is a Banach space, $K_1\subset X$, $K_2\subset \mathbb{R}^d$ are two compact sets in $X$ and $\mathbb{R}^d$, respectively, $V$ is a compact set in $C(K_1)$. Assume that $G^{*}: V\rightarrow C(K_2)$ is a nonlinear continuous operator. Then, an operator net can be formulated in terms of two shallow neural networks. The first is the so-called branch net $\beta(u)=(\beta_1(u),\cdots, \beta_m(u))$, defined for $1\leq r\leq m$ as
\[ \beta_r(u):=\sum\limits_{k=1}^p a_{rk}\sigma\left(\sum\limits_{s=1}^q  \xi_{rk}^su(x_s)+\theta_{rk} \right),\]
where $\{x_s\}_{1\leq s\leq q} \subset K_1$ are the so-called sensors and $a_{rk}, \xi_{rk}^s,\theta_{rk}$ are weights of the neural network.

The second neural network is the so-called trunk net $\tau(y)=(\tau_1(y), \cdots, \tau_m(y))$, defined as
\[ \tau_r(y) :=\sigma(w_r^T y+\zeta_r), 1\leq r \leq m ,\]
where $y\in K_2$ and $w_r, \zeta_r$ are weights of the neural network. Then the branch net and trunk net are combined to approximate the non-linear operator $G^{*}$, i.e.,
\[G^{*}(u)(y) \approx \sum\limits_{r=1}^m \beta_r(u)\tau_r(y):=G(u)(y), u\in V,y\in K_2. \]
 
As shown in [4], (shallow) operator networks can approximate,
to arbitrary accuracy, continuous operators mapping between spaces of continuous functions. Specifically, for any $\epsilon > 0$, there are positive integers $p,m$ and $q$, constants $a_{rk}, \xi_{rk}^s,\theta_{rk},\zeta_r \in \mathbb{R}$, $w_r\in \mathbb{R}^d$, $x_s \in K_1, s=1,\cdots q$, $r=1,\cdots, m$ and $k=1,\cdots, m$, such that
\[\left|G^{*}(u)(y) - \sum\limits_{r=1}^m \sum\limits_{k=1}^{p} a_{rk}\sigma\left(\sum\limits_{s=1}^q  \xi_{rk}^su(x_s)+\theta_{rk} \right)\sigma(w_r^T y+\zeta_r) \right|<\epsilon\]
holds for all $u\in V$ and $y\in K_2$.

The training of neural networks is performed using a supervised learning process. It involves minimizing the mean-squared error between the predicted output $G(u)(y)$ and the actual output $G^{*}(u)(y)$. Specifically, assume we have samples $\{(u_i, G^{*}(u_i))\}_{i=1}^N, u_i\sim \mu$, where $\mu$ is a probability supported on $V$. The aim is to minimize the following loss function:
\[ \frac{1}{2}\sum\limits_{i=1}^N \sum\limits_{j=1}^n |G(u_i)(y_j)-G^{*}(u_i)(y_j)|^2.\]

In this paper, we primarily focus on the shallow neural operators with ReLU activation functions. Formally, we consider a shallow operator of the following form.
\begin{equation}
G(u)(y) = \frac{1}{\sqrt{m}}\sum\limits_{r=1}^m \left[\frac{1}{\sqrt{p}} \sum\limits_{k=1}^p \tilde{a}_{rk} \sigma(\tilde{w}_{rk}^T u ) \right]\sigma(w_r^T y),	
\end{equation}

where we equate the function $u$ with its value vector $(u(x_1), \cdots, u(x_q))^T \in \mathbb{R}^q$ at the points $\{x_s\}_{s=1}^q$.

We denote the loss function by $L(W, \tilde{W}, a)$. The main focus of this paper is to analyze the gradient descent in training the shallw neural operator. We fix the weights $\{a_{rk}\}_{r=1,k=1}^{m,p}$ and apply gradient (GD) to optimize the weights $\{w_r\}_{r=1}^m, \{\tilde{w}_{rk}\}_{r=1,k=1}^{m,p}$. Specifically, 
\[ w_r(t+1)=w_r(t)-\eta \frac{\partial L(W(t), \tilde{W}(t))}{\partial w_r} ,\ \tilde{w}_{rk}(t+1)=\tilde{w}_{rk}(t)-\eta \frac{\partial L(W(t), \tilde{W}(t))}{\partial \tilde{w}_{rk}},\]
where $\eta>0$ is the learning rate and $L(W,\tilde{W})$ is an abbreviation of $L(W, \tilde{W}, a)$.

At this point, the loss function is
\[L(W,\tilde{W})=\frac{1}{2}\sum\limits_{i=1}^{n_1}\sum\limits_{j=1}^{n_2} (G(u_i)(y_j)-z_j^i)^2, \]
where $z_j^i = G^{*}(u_i)(y_j)$. Throughout this paper, we consider
the initialization
\begin{equation}
 w_r(0)\sim \mathcal{N}(\bm{0},\bm{I}),\tilde{w}_{rk}(0)\sim \mathcal{N}(\bm{0},\bm{I}), \bm{a}_{rk}(0)\sim Unif\{-1,1\}
\end{equation}
and assume that $\|u_i\|_2=\mathcal{O}(1), \|y_j\|_2=\mathcal{O}(1)$ for all $j\in [n]$. Note that here we treat vector $u_i=(u_i(x_1),\cdots, u_{x_q})$ and function $u_i$ as equivalent.
\section{Continuous Time Analysis}

In this section, we present our result for gradient flow, which can be viewed as a continuous form of gradient descent with an infinitesimal time step size. The analysis of gradient flow in continuous time serves as a foundational step for comprehending discrete gradient descent algorithms. We prove that the gradient flow converges to the global optima of the loss under over-parameterization and some mild conditions on training samples. The time continuous form can be characterized as the
following dynamics
\[\frac{d w_r(t)}{dt}=-\frac{\partial L(W(t), \tilde{W}(t))}{\partial w_r},\ \frac{d}{dt} \tilde{w}_{rk}(t)=-\frac{\partial L(W(t), \tilde{W}(t))}{\partial \tilde{w}_{rk}} \]
for $r\in [m], k\in [p]$. We denote $G^{t}(u_i)(y_j)$ the prediction on $y_j$ at time $t$ under $u_i$, i.e., with weights $w_r(t), \tilde{w}_{rk}(t)$. 

Thus, we can deduce that 
\begin{equation}
 \frac{dw_r(t)}{dt}=-\frac{\partial L(W(t),\tilde{W}(t))}{\partial w_r} = -\sum\limits_{i=1}^{n_1} \sum\limits_{j=1}^{n_2} \frac{\partial G^{t}(u_i)(y_j)}{\partial w_r}(G(u_i)(y_j)-z_j^i) 
\end{equation}

and 
\begin{equation}
\frac{d\tilde{w}_{rk}(t)}{dt}=-\frac{\partial L(W(t), \tilde{W}(t))}{\partial \tilde{w}_{rk}}=-  \sum\limits_{i=1}^{n_1} \sum\limits_{j=1}^{n_2}  \frac{\partial G^{t}(u_i)(y_j)}{\partial \tilde{w}_{rk}} (G^{t}(u)(y_j)-z_j^i). 	
\end{equation}

Then, the dynamics of each prediction can be calculated as follows.
\begin{equation}
\begin{aligned}
\frac{d G^{t}(u_i)(y_j)}{dt} &= \sum\limits_{r=1}^m \left\langle \frac{\partial G^{t}(u_i)(y_j)}{\partial w_r} ,\frac{dw_r(t)}{dt} \right\rangle+ \sum\limits_{r=1}^m \sum\limits_{k=1}^p\left\langle \frac{\partial G^{t}(u_i)(y_j)}{\partial \tilde{w}_{rk}} ,\frac{d\tilde{w}_{rk}(t)}{dt} \right\rangle \\
&=-\sum\limits_{r=1}^m \sum\limits_{i_1=1}^{n_1} \sum\limits_{j_1=1}^{n_2} \left\langle \frac{\partial G^{t}(u_i)(y_j)}{\partial w_r}, \frac{\partial G^{t}(u_{i_1})(y_{j_1})}{\partial w_r} \right\rangle (G^{t}(u_{i_1})(y_{j_1})-z_{j_1}^{i_1}) \\
&\quad -\sum\limits_{r=1}^m \sum\limits_{k=1}^p \sum\limits_{i_1=1}^{n_1} \sum\limits_{j_1=1}^{n_2} \left\langle \frac{\partial G^{t}(u_i)(y_j)}{\partial \tilde{w}_{rk}}, \frac{\partial G^{t}(u_{i_1})(y_{j_1})}{\partial\tilde{w}_{rk}} \right\rangle (G^{t}(u_{i_1})(y_{j_1})-z_{j_1}^{i_1}).
\end{aligned}
\end{equation}

We let $G^{t}(u_i):=(G^{t}(u_i)(y_1), \cdots, G^{t}(u_i)(y_{n_2})) \in \mathbb{R}^{n_2}$ and $G^{t}(u):=(G^{t}(u_1), \cdots, G^{t}(u_{n_1}))\in \mathbb{R}^{n_1n_2}$. Then, we have
\begin{equation}
	\frac{d G^{t}(u_i)}{dt} =\left[ (H_1^i(t),\cdots, H_{n_1}^i(t))+(\tilde{H}_1^i(t),\cdots, \tilde{H}_{n_1}^i(t))\right] (z-G^{t}(u)),
\end{equation}
where $H_j^i(t) \in \mathbb{R}^{n_2\times n_2}$, whose $(i_1,j_1)$-th entry is defined as
\[\sum\limits_{r=1}^m  \left\langle  \frac{ \partial G^{t}(u_i)(y_{i_1})  }{\partial w_r} , \frac{G^{t}(u_j)(y_{j_1}) }{ \partial w_r}\right\rangle   \]
and $\tilde{H}_j^i(t) \in \mathbb{R}^{n_2\times n_2}$, whose $(i_1,j_1)$-th entry is defined as
\[\sum\limits_{r=1}^m \sum\limits_{k=1}^p \left\langle  \frac{ \partial G^{t}(u_i)(y_{i_1})  }{\partial \tilde{w}_{rk}} , \frac{G^{t}(u_j)(y_{j_1}) }{ \partial \tilde{w}_{rk}}\right\rangle .  \]

Thus, we can write the dynamics of predictions as follows.
\[ \frac{d G^{t}(u)}{dt}= \left[H(t)+\tilde{H}(t)\right](z-G^t(u)),\]
where $H(t),\tilde{H}(t)\in \mathbb{R}^{n_1n_2\times n_2n_2}$. We can divide $H(t), \tilde{H}(t)$ into $n_1\times n_1$ blocks, the $(i,j)$-th block of $H(t)$ is $H_j^i(t)$ and  the $(i,j)$-th block of $\tilde{H}(t)$ is $\tilde{H}_j^i(t)$. From the form of $H(t),\tilde{H}(t)$, we can derive the Gram matrices induced by the random initialization, which we denote by $H^{\infty}$ and $\tilde{H}^{\infty}$. Note that although $H^{\infty}$ and $\tilde{H}^{\infty}$ are large matrices, we can divide $H^{\infty}$ and $\tilde{H}^{\infty}$ into $n_1\times n_1$ blocks, where each block is a $n_2\times n_2$ matrix. Following the notation above, the $(i_1,j_1)$-th entry of $(i,j)$-th block of $H^{\infty}$ is 
\[ \mathbb{E}[\sigma(\tilde{w}^T u_i)\sigma(\tilde{w}^T u_j)] \mathbb{E}[ y_{i_1}^T y_{j_1} I\{w^T y_{i_1}\geq 0, w^T y_{j_1}\geq 0\}].\]

Thus, the $(i,j)$-th block can be written as 
\[\mathbb{E}[\sigma(\tilde{w}^T u_i)\sigma(\tilde{w}^T u_j)] H_2^{\infty},\]
where $H_2^{\infty} \in \mathbb{R}^{n_2\times n_2}$ and the $(i_1,j_1)$-th entry of $H_2^{\infty}$ is
\[\mathbb{E}[ y_{i_1}^T y_{j_1} I\{w^T y_{i_1}\geq 0, w^T y_{j_1}\geq 0\}].\]

Thus, $H^{\infty}$ can be seen as a Kronecker product of matrices $H_1^{\infty}$ and $H_2^{\infty}$, where $H_1^{\infty}\in \mathbb{R}^{n_1\times n_1}$ and the $(i,j)$-th entry of $H_1^{\infty}$ is $\mathbb{E}[\sigma(\tilde{w}^T u_i)\sigma(\tilde{w}^T u_j)]$. Similarly, we have that $\tilde{H}^{\infty}$ is a Kronecker product of matrices $\tilde{H}_1^{\infty}$ and $\tilde{H}_2^{\infty}$, where $\tilde{H}_1^{\infty}\in \mathbb{R}^{n_1\times n_1}$ and the $(i,j)$-th entry of $\tilde{H}_1^{\infty}$ is $\mathbb{E}[u_i^T u_j I\{\tilde{w}^T u_i\geq 0, \tilde{w}^T u_j \geq 0 \}]$, the $(i_1,j_1)$-th entry of $\tilde{H}_2^{\infty}$ is $\mathbb{E}[\sigma(w^T y_{i_1})\sigma(w^T y_{j_1})]$ .

Similar to the situation in $L^2$ regression, we can show two essential facts: (1) $\|H(0)-H^{\infty}\|_2=\mathcal{O}(1/\sqrt{m})$, $\|\tilde{H}(0)-\tilde{H}^{\infty}\|_2=\mathcal{O}(1/\sqrt{m})$ and (2) for all $t\geq 0$, $\|H(t)-H(0)\|_2=\mathcal{O}(1/\sqrt{m})$, $\|\tilde{H}(t)-\tilde{H}(0)\|_2=\mathcal{O}(1/\sqrt{m})$.

Therefore, roughly speaking, as $m\to \infty$, the dynamics of the predictions can be written as
\begin{equation*}
\frac{d}{dt} G^{t}(u) = \left( H^{\infty}+\tilde{H}^{\infty}\right)(z-G^{t}(u)),
\end{equation*}
which results in the linear convergence.

We first show that the Gram matrices are strictly positive definite under mild assumptions.

\begin{lemma}
If no two samples in $\{y_j\}_{j=1}^{n_2}$ are parallel and no two samples in $\{u_i\}_{i=1}^{n_1}$ are parallel, then $H^{\infty}$ and $\tilde{H}^{\infty}$ are strictly positive definite. We denote the least eigenvalue of $H^{\infty}$ and $\tilde{H}^{\infty}$ as $\lambda_0$ and $\tilde{\lambda}_0$ respectively.
\end{lemma}

\begin{remark}
In fact, when we consider neural networks with bias, it is natural that Lemma 1 holds. Specifically, for $r\in [m], j\in [n_2]$, we can replace $w_r$ and $y_j$ by $(w_r^T, 1)^T$ and $(y_j^T, 1)^T$. Thus Lemma 1 holds under the condition that no two samples in $\{y_j\}_{j=1}^{n_2}$ are identical, which holds naturally.	
\end{remark}

Then we can verify the two facts that $H(0), \tilde{H}(0)$ are close to $H^{\infty}, \tilde{H}^{\infty}$ and $H(t),\tilde{H}(t)$ are close to $H(0), \tilde{H}(0)$ by following two lemmas.

\begin{lemma}
If $m=\Omega\left(\frac{n_1^2n_2^2}{min(\lambda_0^2, \tilde{\lambda}_0^2) } \log\left(\frac{n_1n_2}{\delta}\right) \right)$, we have with probability at least $1-\delta$, 
$\|H(0)-H^{\infty}\|_2\leq \frac{\lambda_0}{4}$, $\|\tilde{H}(0)-\tilde{H}^{\infty}\|_2\leq \frac{\tilde{\lambda}_0}{4}$ and 
$\lambda_{min}(H(0))\geq \frac{3}{4}\lambda_0$, $\lambda_{min}(\tilde{H}(0))\geq \frac{3}{4}\tilde{\lambda}_0$.
\end{lemma}

\begin{lemma}
Let $R,\tilde{R}\in (0,1)$. If $w_1(0), \cdots, w_m(0), \tilde{w}_{00}(0),\cdots \tilde{w}_{mp}(0)$ are i.i.d. generated $\mathcal{N}(\bm{0}, \bm{I})$. For any set of weight vectors $w_1, \cdots, w_m, \tilde{w}_{00},\cdots \tilde{w}_{mp}\in \mathbb{R}^d$ that satisfy for any $r\in [m], k\in [p]$, 	$\|w_r-w_r(0)\|_2 \leq R$ and $\|\tilde{w}_{rk}-\tilde{w}_{rk}(0)\|_2\leq \tilde{R}$, then we have
with probability at least $1-\delta-n_2exp(-mR)$,
\begin{equation}
\|H-H(0)\|_F\lesssim  n_1n_2p\tilde{R}^2+ n_1n_2\sqrt{p}\tilde{R} \sqrt{\log \left(\frac{mn_1}{\delta} \right)}+ n_1n_2R\log \left(\frac{mn_1}{\delta} \right)
\end{equation}
and
with probability at least $1-\delta-n_2exp(-mp\tilde{R})$, 
\begin{equation}
 \|\tilde{H}-\tilde{H}(0)\|_F\lesssim n_1n_2R\sqrt{\log \left(\frac{n_2}{\delta} \right)}+n_1n_2\tilde{R}\log \left(\frac{mn_2}{\delta} \right) ,
\end{equation}
where the $(i_1,j_1)$-th entry of the $(i,j)$-th block of $H$ is
\[H_{i,j}^{i_1,j_1} = \frac{1}{m}\sum\limits_{r=1}^m \left[\frac{1}{\sqrt{p}} \sum\limits_{k=1}^p \tilde{a}_{rk} \sigma(\tilde{w}_{rk}^T u_i )  \right] \left[\frac{1}{\sqrt{p}} \sum\limits_{k=1}^p \tilde{a}_{rk} \sigma(\tilde{w}_{rk}^T u_j )  \right]y_{i_1}^T y_{j_1} I\{w_r^T y_{i_1} \geq 0, w_r^T y_{j_1} \geq 0\}\]
and the $(i_1,j_1)$-th entry of the $(i,j)$-th block of $\tilde{H}$ is
\[\tilde{H}_{i,j}^{i_1,j_1} = \frac{1}{m}\frac{1}{p}\sum\limits_{r=1}^m \sum\limits_{k=1}^p u_i^T u_jI\{\tilde{w}_{rk}^T u_i\geq 0, \tilde{w}_{rk}^T u_j\geq 0 \}\sigma(w_r^T y_{i_1})\sigma(w_r^T y_{j_1}) . \]
\end{lemma}

With these preparations, we come to the final conclusion.

\begin{theorem}
Suppose the condition in Lemma 1 holds and under initialization as described in (2), then with probability at least $1-\delta$, we have
\[\|z-G^{t}(u)\|_2^2\leq exp(-(\lambda_0+\tilde{\lambda}_0)t) \|z-G^{0}(u)\|_2^2,\]
where
\[m=\Omega\left( \frac{ n_1^4n_2^4\log\left( \frac{n_1n_2}{\delta}\right)\log^3\left( \frac{m}{\delta}\right) }{(min(\lambda_0, \tilde{\lambda}_0))^2(\lambda_0+\tilde{\lambda}_0)^2} \right).\]
\end{theorem}

\textbf{Proof Sketch:} Note that 
\begin{equation}
\begin{aligned}
\frac{d}{dt} \|z-G^{t}(u)\|_2^2 &= - 2(z-G^{t}(u))^T (H(t)+\tilde{H}(t))(z-G^{t}(u)),
\end{aligned}
\end{equation}
thus if $\lambda_{min}(H(t))\geq \lambda_0/2$ and $\lambda_{min}(\tilde{H}(t))\geq \tilde{\lambda}_0/2$, we have
\[\frac{d}{dt} \|z-G^{t}(u)\|_2^2 \leq -(\lambda_0+\tilde{\lambda}_0)\|z-G^{t}(u)\|_2^2.\]
This yields that $\frac{d}{dt} \left( exp((\lambda_0+\tilde{\lambda}_0)t)\|z-G^{t}(u)\|_2^2 \right)\leq 0$, i.e., $exp((\lambda_0+\tilde{\lambda}_0)t)\|z-G^{t}(u)\|_2^2$ is non-increasing, thus we have
\begin{equation}
\|z-G^{t}(u)\|_2^2\leq  exp(-(\lambda_0+\tilde{\lambda}_0)t)\|z-G^{0}(u)\|_2^2.
\end{equation}

On the other hand, roughly speaking, the continous dynamics of $w_r(t)$ and $\tilde{w}_{rk}(t)$, i.e., (3) and (4), show that
\[ \left\|\frac{dw_r(t)}{dt} \right\|_2 \sim \frac{\|z-G^{t}(u)\|_2}{\sqrt{m}}, \ \left\|\frac{d\tilde{w}_{rk}(t)}{dt} \right\|_2 \sim \frac{\|z-G^{t}(u)\|_2}{\sqrt{mp}}.\]

Thus if the prediction decays like (10), we can deduce that
\[  \|w_r(t)-w_r(0)\|_2\sim \frac{1}{\sqrt{m}},\ \|\tilde{w}_{rk}(t)-\tilde{w}_{rk}(0)\|_2\sim \frac{1}{\sqrt{mp}}.\]

Combining this with the stability of the descrete Gram matrices, i.e., Lemma 3, we have $\|H(t)-H^{\infty}\|_2\leq \lambda_0/4$, $\|\tilde{H}(t)-\tilde{H}^{\infty}\|_2\leq \tilde{\lambda}_0/4$ and  $\lambda_{min}(H(t))\geq \lambda_0/2$, $\lambda_{min}(\tilde{H}(t))\geq \tilde{\lambda}_0/2$, when $m$ is sufficiently large. 

From such equivalence, we can arrive at the desired conclusion.

\begin{remark}
The result in Theorem 1 indicates that $m= \Omega(Poly(1/min(\lambda_0, \tilde{\lambda}_0)))$, which may lead to strict requirement for $m$. In fact, from (11), we can see that  $\lambda_{min}(H(t))\geq \lambda_0/2$ or $\lambda_{min}(\tilde{H}(t))\geq \tilde{\lambda}_0/2$ is enough. Thus, $m= \Omega(Poly(1/max(\lambda_0, \tilde{\lambda}_0)))$ is sufficient.
\end{remark}

\section{Discrete Time Analysis}

In this section, we are going to demonstrate that randomly initialized gradient in training shallow neural operators converges to the golbal minimum at a linear rate. Unlike the continuous time case, the discrete time case requires a more refined analysis. In the following, we first present our main result and then outline the proof's approach.

\begin{theorem}
Under the setting of Theorem 1, if we set $\eta=\mathcal{O}\left(\frac{1}{\| H^{\infty}\|_2+\|\tilde{H}^\infty\|_2}\right)$, then with probability at least $1-\delta$, we have
\[ \|z-G^{t}(u)\|_2^2 \leq \left(1-\eta \frac{\lambda_0+\tilde{\lambda}_0}{2} \right)^{t}\|z-G^{0}\|_2^2,\]
where 
\[m=\Omega\left( \frac{ n_1^4n_2^4\log\left( \frac{n_1n_2}{\delta}\right)\log^3\left( \frac{m}{\delta}\right) }{(min(\lambda_0, \tilde{\lambda}_0))^2(\lambda_0+\tilde{\lambda}_0)^2} \right).\]
\end{theorem}

For the $L^2$ regression problems, \cite{7} has demonstrated that if the learning rate $\eta=\mathcal{O}(\lambda_0/n^2)$, then randomly initialized gradient descent converges to a globally optimal solution at a linear convergence rate when $m$ is large enough. The requirement of $\eta$ is derived from the decomposition for the residual in the $(k+1)$-th iteration, i.e.,
\begin{equation*}
	y-u(k+1)=y-u(k)-(u(k+1)-u(k)),	
\end{equation*}
where $u(t)$ is the prediction vector at $t$-iteration under the shallow neural network and $y$ is the true prediction. Although using the method from \cite{7} can also yield linear convergence of gradient descent in training shallow neural operators, the requirements for the learning rate $\eta$ would be very stringent due to the dependency on $\lambda_0,\tilde{\lambda}_0$ and $n$. Thus, instead of decomposing the residual into the two terms as above, we write it as follows, which serves as a recursion formula.

\begin{lemma}
For all $t\in \mathbb{N}$, we have
\[z-G^{t+1}(u)= \left[I-\eta\left(H(t)+\tilde{H}(t)\right)\right] \left(z-G^{t}(u)\right)-I(t),\]
where $I(t)\in \mathbb{R}^{n_1n_2}$ is the residual term. We can divide it into $n_1$ blocks, each block belongs to $\mathbb{R}^n_2$ and the $j$-th component of $i$-th block is defined as
\begin{equation}
I_{i,j}(t)= G^{t+1}(u_i)(y_j) - G^{t}(u_i)(y_j) -\left\langle \frac{\partial G^{t}(u_i)(y_j)}{\partial w},  w(t+1)-w(t) \right\rangle  - \left\langle \frac{\partial G^{t}(u_i)(y_j)}{\partial \tilde{w} },  \tilde{w}(t+1)-\tilde{w}(t) \right\rangle.
\end{equation} 
\end{lemma}

Just as in the case of $L^2$ regression, we prove our conclusion by induction. From the recursive formula above, it can be seen that both the estimation of $H(t)$ and $\tilde{H}(t)$, as well as the estimation of the residual $I(t)$, depend on $\|w_r(t)-w_r(0)\|_2$ and $\|\tilde{w}_{rk}(t)-\tilde{w}_{rk}(0)\|_2$. Therefore, our inductive hypothesis is the following differences between weights and their initializations.

\begin{condition}
At the $s$-th iteration, we have
\begin{equation}
\|w_r(s)-w_r(0)\|_2 \lesssim \frac{ \sqrt{n_1n_2}\|z-G^{0}(u)\|_2}{\sqrt{m}(\lambda_0+\tilde{\lambda}_0)}\sqrt{ \log \left(\frac{m}{\delta}\right)}:=R^{'}
\end{equation}
and
\begin{equation}
\|\tilde{w}_{rk}(s)-\tilde{w}_{rk}(0)\|_2\lesssim \frac{\sqrt{n_1n_2}\|z-G^{0}(u)  \|_2}{\sqrt{m} \sqrt{p}(\lambda_0+\tilde{\lambda}_0)}\sqrt{ \log \left(\frac{m}{\delta}\right)}:=\tilde{R}^{'}
\end{equation}
and  $|w_r(s)^Ty_j|\leq B$ for all $r\in [m]$ and $j\in [n_2]$, where $B=2\sqrt{\log \left(\frac{mn_2}{\delta} \right) }$.	
\end{condition}

This condition can lead to the linear convergence of gradient descent, i.e., result in Theorem 2.

\begin{corollary}
If Condition 1 holds for $s=0,\cdots, t$, then we have that \[\|z-G^{s}(u)\|_2^2\leq \left(1-\eta\frac{\lambda_0+\tilde{\lambda}_0  }{2}\right)^s \|z-G^{0}(u)\|_2^2\]
holds for $s=0,\cdots, t$, where $m$ is required to satisfy that 
\[m=\Omega\left( \frac{ n_1^4n_2^4\log\left( \frac{n_1n_2}{\delta}\right)\log^3\left( \frac{m}{\delta}\right) }{(min(\lambda_0, \tilde{\lambda}_0))^2(\lambda_0+\tilde{\lambda}_0)^2} \right).\]
\end{corollary}

\textbf{Proof Sketch:} Under the setting of over-parameterization, we can show that the weights $w_r(s)$, $\tilde{w}_{rk}(s)$ stay close to the initialization $w_r(0)$, $\tilde{w}_{rk}(0)$. Thus, with the stability of the discrete Gram matrices, i.e., Lemma 3, we can deduce that $\lambda_{min}(H(s))\geq \lambda_0/2 $ and $\lambda_{min}(\tilde{H}(s))\geq \tilde{\lambda}_0/2 $. Then combining with the Lemma 4, we have

\begin{equation}
\begin{aligned}
&\|z-G^{s+1}(u)\|_2^2 \\ 
&=\left\| \left(I-\eta\left(H(s)+\tilde{H}(s)\right)\right) (z-G^{s}(u)) \right\|_2^2+\|I(s)\|_2^2-2\left\langle \left(I-\eta(H(s)+\tilde{H}(s))\right) (z-G^{s}(u)), I(s) \right\rangle \\
&\leq \left(1-\eta\frac{\lambda_0+\tilde{\lambda}_0 }{2}\right)^2 \|z-G^{s}(u)\|_2^2+\|I(s)\|_2^2+2\left(1-\eta\frac{\lambda_0+\tilde{\lambda}_0 }{2}\right)\|z-G^{s}(u)\|_2 \|I(s)\|_2,
\end{aligned}
\end{equation}
where the inequality requires that $I-\eta(H(s)+\tilde{H}(s))$ is positive definite. Since $\|H(s)-H^{\infty}\|_2=\mathcal{O}(1/\sqrt{m})$ and $\|\tilde{H}(s)-\tilde{H}^{\infty}\|_2=\mathcal{O}(1/\sqrt{m})$, $\eta=\mathcal{O}\left(\frac{1}{\|H^{\infty}\|_2+\|\tilde{H}^{\infty}\|_2}\right)$ is sufficient to ensure that $I-\eta(H(s)+\tilde{H}(s))$ is positive definite, when $m$ is large enough.

From (14), if $\|I(s)\|_2\lesssim \eta(\lambda_0+\tilde{\lambda}_0) \|z-G^{s}(u)\|_2$, which can be obtained from the following lemma, we can obtain that
\[ \|z-G^{s+1}(u)\|_2^2\leq \left(1-\eta\frac{\lambda_0+\tilde{\lambda}_0}{2}\right)\|z-G^{s}(u)\|_2 ,\]
which directly yields the desired conclusion.

\begin{lemma}
Under Condition 1, for $s=0,\cdots, t-1$, we have
\begin{equation}
	\|I(s)\|_2\lesssim \bar{R}\|z-G^s(u)\|_2,
\end{equation}
where 
\begin{equation}
	\bar{R}= \frac{\eta (n_1n_2)^{\frac{3}{2}}\|z-G^{0}(u)\|_2}{\sqrt{m}(\lambda_0+\tilde{\lambda}_0)} \log^{\frac{3}{2}}\left( \frac{m}{\delta}\right).
\end{equation}
\end{lemma}

\section{Physics-Informed Neural Operators}

Let $\Gamma$ be a bounded open subset of $\mathbb{R}^d$, in this section, we consider the PDE with following form:
\begin{equation}
\begin{aligned}
	& \mathcal{L} u(y) = f(y), \ y\in (0, T)\times \Gamma \\
	& u(\tilde{y}) = g(\tilde{y}), \ \tilde{y}\in \{0\}\times \Gamma \cup [0,T]\times \partial \Gamma,
\end{aligned}
\end{equation}
where $y=(y_0,y_1\cdots, y_d)$, $y_0\in (0,T)$, $(y_1\cdots, y_d)\in \bar{\Gamma}$ and $\mathcal{L}$ is a differential operator, 
\[ \mathcal{L} u(y)=\frac{\partial u}{\partial y_0}(y)-\sum\limits_{i=1}^d \frac{\partial^2 u}{\partial y_i^2}(y)+u(y).\]

In this section, we consider the shallow neural operato with following form
\[G(u)(y) = \frac{1}{\sqrt{m}}\sum\limits_{r=1}^m \left[\frac{1}{\sqrt{p}} \sum\limits_{k=1}^p \tilde{a}_{rk} \sigma(\tilde{w}_{rk}^T u) \right]\sigma_3(w_r^T y),\]
where $\sigma_3(\cdot), \sigma_2(\cdot), \sigma(\cdot)$ are the $\text{ReLU}^3, \text{ReLU}^2, \text{ReLU}$ activation functions, respectively.

Given samples $y_1,\cdots, y_{n_2}$ in the interior and $\tilde{y}_1,\cdots, \tilde{y}_{n_3}$ on the boundary, the loss function of PINN is
\[ L(W, \tilde{W}):=\sum\limits_{i=1}^{n_1} \sum\limits_{j_1=1}^{n_2} \frac{1}{n_2} (\mathcal{L}G(u_i)(y_{j_1})-f(y_{j_1}) )^2 +\sum\limits_{i=1}^{n_1} \sum\limits_{j_2=1}^{n_3} \frac{1}{n_3} (G(u_i)(\tilde{y}_{j_2})-g(\tilde{y}_{j_2}) )^2. \]

Let 
\[ s(u_i)(y_j) = \frac{1}{\sqrt{n_2}} (\mathcal{L}G(u_i)(y_{j_1}) -f(y_{j_1}))\]
and
\[h(u_i)(\tilde{y}_{j_2}) =\frac{1}{\sqrt{n_3}} (G(u_i)(\tilde{y}_{j_2}) -g(\tilde{y}_{j_2})), \]
then the loss function can be written as
\[L(W, \tilde{W})= \sum\limits_{i=1}^{n_1}\|s(u_i)\|_2^2+\|h(u_i)\|_2^2,\]
where $s(u_i)=(s(u_i)(y_1),\cdots, s(u_i)(y_{n_2}))$, $h(u_i)=(h(u_i)(\tilde{y}_1),\cdots, h(u_i)(\tilde{y}_{n_3}))$.

We first consider the continuous setting, which is a stepping stone towards understanding discrete algorithms. For $w_r(t)$ and $\tilde{w}_{rt}$, we have
\begin{equation}
	\begin{aligned}
		\frac{d w_r(t)}{dt} &= -\frac{\partial L(W(t), \tilde{W}(t))}{\partial w_r}\\
		&=  -\sum\limits_{i=1}^{n_1} \sum\limits_{j_1=1}^{n_2} s^{t}(u_i)(y_{j_1})\frac{\partial s^{t}(u_i)(y_{j_1})}{\partial w_r}-\sum\limits_{i=1}^{n_1} \sum\limits_{j_2=1}^{n_3} h^{t}(u_i)(\tilde{y}_{j_2})\frac{\partial h^{t}(u_i)(\tilde{y}_{j_2})}{\partial w_r}
	\end{aligned}
\end{equation}
and
\begin{equation}
	\begin{aligned}
		\frac{d \tilde{w}_{rk}(t)}{dt} &= -\frac{\partial L(W(t), \tilde{W}(t))}{\partial \tilde{w}_{rk}}\\
		&=  -\sum\limits_{i=1}^{n_1} \sum\limits_{j_1=1}^{n_2} s^{t}(u_i)(y_{j_1})\frac{\partial s^{t}(u_i)(y_{j_1})}{\partial \tilde{w}_{rk}}-\sum\limits_{i=1}^{n_1} \sum\limits_{j_2=1}^{n_3} h^{t}(u_i)(\tilde{y}_{j_2})\frac{\partial h^{t}(u_i)(\tilde{y}_{j_2})}{\partial \tilde{w}_{rk}}.
	\end{aligned}
\end{equation}

Thus, for the predictions, we have
\begin{equation}
	\begin{aligned}
		&\frac{d s^{t}(u_i)(y_j)}{dt} \\
		&= \sum\limits_{r=1}^m \langle \frac{\partial s^{t}(u_i)(y_j) }{\partial w_r} , \frac{d w_r(t)}{dt}\rangle +\sum\limits_{r=1}^m \sum\limits_{k=1}^{q} \langle \frac{\partial s^{t}(u_i)(y_j) }{\partial \tilde{w}_{rk} } , \frac{d \tilde{w}_{rk}(t)}{dt}\rangle \\
		&= -\sum\limits_{r=1}^m \sum\limits_{i_1=1}^{n_1} \sum\limits_{j_1=1}^{n_2} \langle \frac{\partial s^{t}(u_i)(y_j) }{\partial w_r} , \frac{\partial s^{t}(u_{i_1})(y_{j_1}) }{\partial w_r}\rangle s^{t}(u_{i_1})(y_{j_1}) - \sum\limits_{r=1}^m \sum\limits_{i_1=1}^{n_1} \sum\limits_{j_2=1}^{n_3} \langle \frac{\partial s^{t}(u_i)(y_j) }{\partial w_r} , \frac{\partial h^{t}(u_{i_1})(\tilde{y}_{j_2}) }{\partial w_r}\rangle h^{t}(u_{i_1})(\tilde{y}_{j_2}) \\
		&- \sum\limits_{r,k=1,1}^{m,p} \sum\limits_{i_1=1}^{n_1} \sum\limits_{j_1=1}^{n_2} \langle \frac{\partial s^{t}(u_i)(y_j) }{\partial \tilde{w}_{rk} } , \frac{\partial s^{t}(u_{i_1})(y_{j_1}) }{\partial \tilde{w}_{rk} }\rangle s^{t}(u_{i_1})(y_{j_1}) - \sum\limits_{r,k=1,1}^{m,p}  \sum\limits_{i_1=1}^{n_1} \sum\limits_{j_2=1}^{n_3} \langle \frac{\partial s^{t}(u_i)(y_j) }{\partial \tilde{w}_{rk} } , \frac{\partial h^{t}(u_{i_1})(\tilde{y}_{j_2}) }{\partial \tilde{w}_{rk} }\rangle h^{t}(u_{i_1})(\tilde{y}_{j_2}),
	\end{aligned}
\end{equation}
and
\begin{equation}
	\begin{aligned}
		& \frac{d h^{t}(u_i)(\tilde{y}_j) }{dt} \\
		&=\sum\limits_{r=1}^m \langle \frac{\partial h^{t}(u_i)(\tilde{y}_j) }{\partial w_r} , \frac{d w_r(t)}{dt}\rangle +\sum\limits_{r=1}^m \sum\limits_{k=1}^{l} \langle \frac{\partial h^{t}(u_i)(\tilde{y}_j) }{\partial \tilde{w}_{rk} } , \frac{d \tilde{w}_{rk}(t)}{dt}\rangle \\
		&= -\sum\limits_{r=1}^m \sum\limits_{i_1=1}^{n_1} \sum\limits_{j_1=1}^{n_2} \langle \frac{\partial h^{t}(u_i)(\tilde{y}_j) }{\partial w_r} , \frac{\partial s^{t}(u_{i_1})(y_{j_1}) }{\partial w_r}\rangle s^{t}(u_{i_1})(y_{j_1}) - \sum\limits_{r=1}^m \sum\limits_{i_1=1}^{n_1} \sum\limits_{j_2=1}^{n_3} \langle \frac{\partial h^{t}(u_i)(\tilde{y}_j) }{\partial w_r} , \frac{\partial h^{t}(u_{i_1})(\tilde{y}_{j_2}) }{\partial w_r}\rangle h^{t}(u_{i_1})(\tilde{y}_{j_2}) \\
		&- \sum\limits_{r,k=1,1}^{m,p} \sum\limits_{i_1=1}^{n_1} \sum\limits_{j_1=1}^{n_2} \langle \frac{\partial h^{t}(u_i)(\tilde{y}_j) }{\partial \tilde{w}_{rk} } , \frac{\partial s^{t}(u_{i_1})(y_{j_1}) }{\partial \tilde{w}_{rk} }\rangle s^{t}(u_{i_1})(y_{j_1}) - \sum\limits_{r,k=1,1}^{m,p}  \sum\limits_{i_1=1}^{n_1} \sum\limits_{j_2=1}^{n_3} \langle \frac{\partial h^{t}(u_i)(\tilde{y}_j) }{\partial \tilde{w}_{rk} } , \frac{\partial h^{t}(u_{i_1})(\tilde{y}_{j_2}) }{\partial \tilde{w}_{rk} }\rangle h^{t}(u_{i_1})(\tilde{y}_{j_2}).
	\end{aligned}
\end{equation}	

Let $G^{t}(u) = ((s(u_1), h(u_1)),\cdots, (s(u_{n_1}), h(u_{n_1})))$ and $z=((f,g),\cdots, (f,g))$, then
\[ \frac{d}{dt} G^t(u) =(H(t)+\tilde{H}(t)) (z-G^{t}(u)),\]
where $H(t),\tilde{H}(t) \in \mathbb{R}^{n_1(n_2+n_3)\times n_1(n_2+n_3)}$ are Gram matrices at time $t$. We can divide them into $n_1\times n_1$ blocks and each block is a matrix in $\mathbb{R}^{(n_2+n_3)\times (n_2+n_3)}$. Specifically, the $(i,j)$-th block of $H(t)$ and $\tilde{H}(t)$ are $H_{i,j}(t):=D_i(t)^T D_j(t)$ and $\tilde{H}_{i,j}(t):=\tilde{D}_i(t)^T \tilde{D}_j(t)$ respectively, where
\[ D_i(t)= \left[\frac{\partial s^{t}(u_i)(y_1)}{\partial w}, \cdots,\frac{\partial s^{t}(u_i)(y_{n_2})}{\partial w}, \frac{\partial h^{t}(u_i)(\tilde{y}_1)}{\partial w},\cdots, \frac{\partial h^{t}(u_i)(\tilde{y}_{n_3})}{\partial w} \right]\]
and
\[ \tilde{D}_i(t)= \left[\frac{\partial s^{t}(u_i)(y_1)}{\partial \tilde{w} }, \cdots,\frac{\partial s^{t}(u_i)(y_{n_2})}{\partial \tilde{w}}, \frac{\partial h^{t}(u_i)(\tilde{y}_1)}{\partial \tilde{w}},\cdots, \frac{\partial h^{t}(u_i)(\tilde{y}_{n_3})}{\partial \tilde{w}} \right].\]

Recall that
\[\frac{\partial s(u)(y)}{\partial w_r} = \frac{1}{\sqrt{m} }\left[\frac{1}{\sqrt{p}} \sum\limits_{k=1}^p \tilde{a}_{rk}\sigma(\tilde{w}_{rk}^T u)  \right] \frac{1}{\sqrt{n_2}} \frac{\partial \mathcal{L}(\sigma_3(w_r^T y)) }{\partial w_r}  ,       \]
\[\frac{\partial s(u)(y)}{\partial \tilde{w}_{rk} } = \frac{1}{\sqrt{m} }\frac{\tilde{a}_{rk}}{\sqrt{p}} u I\{\tilde{w}_{rk}^T u \geq 0\} \frac{\mathcal{L}(\sigma_3(w_r^T y)) }{\sqrt{n_2}}   ,    \]
\[\frac{\partial h(u)(\tilde{y})}{\partial w_r} = \frac{1}{\sqrt{m} }\left[\frac{1}{\sqrt{p}} \sum\limits_{k=1}^p \tilde{a}_{rk}\sigma(\tilde{w}_{rk}^T u)  \right] \frac{1}{\sqrt{n_3}} \frac{\partial(\sigma_3(w_r^T \tilde{y})) }{\partial w_r}    ,     \]
and
\[\frac{\partial h(u)(\tilde{y})}{\partial \tilde{w}_{rk} } = \frac{1}{\sqrt{m} }\frac{\tilde{a}_{rk}}{\sqrt{p}}I\{\tilde{w}_{rk}^T u \geq 0\}  \frac{\sigma_3(w_r^T \tilde{y}) }{\sqrt{n_3}}    .   \]

Then, the $(j_1,j_2)$-th ($j_1\in [n_1],j_2\in [n_1]$) entry of $H_{i,j}(t)$ is
\[ \frac{1}{m}\sum\limits_{r=1}^m \left[\frac{1}{\sqrt{p}} \sum\limits_{k=1}^p \tilde{a}_{rk}\sigma(\tilde{w}_{rk}^T u_i)  \right]\left[\frac{1}{\sqrt{p}} \sum\limits_{k=1}^p \tilde{a}_{rk}\sigma(\tilde{w}_{rk}^T u_j)  \right]  \frac{1}{n_2}\left\langle \frac{\partial \mathcal{L}(\sigma_3(w_r^T y_{j_1})) }{\partial w_r}, \frac{\partial \mathcal{L}(\sigma_3(w_r^T y_{j_2})) }{\partial w_r}  \right\rangle.\]
and the $(j_1,j_2)$-th ($j_1\in [n_1],j_2\in [n_1]$) entry of $\tilde{H}_{i,j}(t)$ is
\[\frac{1}{m} \sum\limits_{r=1}^m \left[ \frac{1}{p} \sum\limits_{k=1}^{p}u_i^Tu_j I\{\tilde{w}_{rk}^T u_i\geq 0, \tilde{w}_{rk}^T u_j\geq 0\} \right] \frac{1}{n_2} \mathcal{L}(\sigma(w_r^T y_{j_1}))  \mathcal{L}(\sigma(w_r^T y_{j_2})) ,\]
where we omit the index $t$ for simplicity.

From the forms of $H(t)$ and $\tilde{H}(t)$, we can derive the corresponding Gram matrices that are induced by the random initialization, which we denote by $H^{\infty}$ and $\tilde{H}^{\infty}$, respectively. Specifically, $H^{\infty}$ is a Kronecker product of $H_1^{\infty}$ and $H_2^{\infty}$, where $H_1^{\infty}\in \mathbb{R}^{n_1\times n_1}$, $H_1^{\infty}\in \mathbb{R}^{(n_2+n_3)\times (n_2+n_3)}$, the $(i,j)$-th entry of $H^{\infty}$ is $\mathbb{E}[\sigma(\tilde{w}^T u_i)\sigma(\tilde{w}^T u_j)]$, the $(j_1,j_2)$-th entry of $H_2^{\infty}$ is
\[\frac{1}{n_2}\mathbb{E}\left[\left\langle \frac{\partial \mathcal{L}(\sigma_3(w^T y_{j_1})) }{\partial w}, \frac{\partial \mathcal{L}(\sigma_3(w^T y_{j_2})) }{\partial w}  \right\rangle\right] .\]
And $\tilde{H}^{\infty}$ is a Kronecker product of $\tilde{H}_1^{\infty}$ and $\tilde{H}_2^{\infty}$, where $\tilde{H}_1^{\infty}\in \mathbb{R}^{n_1\times n_1}$, $\tilde{H}_1^{\infty}\in \mathbb{R}^{(n_2+n_3)\times (n_2+n_3)}$, the $(i,j)$-th entry of $\tilde{H}^{\infty}$ is $\mathbb{E}[u_i^T u_jI\{\tilde{w}^T u_i\geq 0,\tilde{w}^T u_j \geq 0\}]$, the $(j_1,j_2)$-th entry of $\tilde{H}_2^{\infty}$ is
\[\frac{1}{n_2}\mathbb{E}[\mathcal{L}(\sigma_3(w^T y_{j_1}))\mathcal{L}(\sigma_3(w^T y_{j_2})) ].\]

The Gram matrices play important roles in the convergence analysis. Similar to the setting of Section 4, we can demonstrate the strict positive definiteness of the Gram matrices under mild conditions.

\begin{lemma}
If no two samples in $\{u_i\}_{i=1}^{n_1}$ are parallel and no two samples in $\{y_j\}_{j=1}^{n_2}\cup \{\tilde{y}_j\}_{j=1}^{n_3}$ are parallel, then $H^{\infty}$ and $\tilde{H}^{\infty}$ are both strictly positive definite. We denote their least eigenvalues by $\lambda_0$ and $\tilde{\lambda}_0$, respectively.
\end{lemma}

Similar to Section 4, the convergence of gradient descent relies on the stability of the Gram matrices, which is demonstrated by the following two lemmas.

\begin{lemma}
If $m=\Omega\left(\frac{d^4n_1^2}{\min\{\lambda_0^2, \tilde{\lambda}_0^2 \}}\log^2\left(\frac{n_1(n_2+n_3)}{\delta} \right)  \right)$, we have with probability at least $1-\delta$, $\|H(0)-H^{\infty}\|_2\leq \frac{\lambda_0}{4}$, $\|\tilde{H}(0)-\tilde{H}^{\infty}\|_2\leq \frac{\tilde{\lambda}_0}{4}$ and $\lambda_{min}(H(0))\geq \frac{3}{4}\lambda_0$, $\lambda_{min}(\tilde{H}(0))\geq \frac{3}{4}\tilde{\lambda}_0$.
\end{lemma}

\begin{lemma}
Let $R,\tilde{R}\in (0,1)$. If $w_1(0), \cdots, w_m(0), \tilde{w}_{00}(0),\cdots \tilde{w}_{mp}(0)$ are i.i.d. generated $\mathcal{N}(\bm{0}, \bm{I})$. For any set of weight vectors $w_1, \cdots, w_m, \tilde{w}_{00},\cdots \tilde{w}_{mp}\in \mathbb{R}^d$ that satisfy for any $r\in [m], k\in [p]$, 	$\|w_r-w_r(0)\|_2 \leq R$ and $\|\tilde{w}_{rk}-\tilde{w}_{rk}(0)\|_2\leq \tilde{R}$, then we have
with probability at least $1-\delta-n_2\exp(-mR)$,
\begin{equation}
\begin{aligned}
\|H-H(0)\|_F&\lesssim  n_1d\log\left(\frac{m}{\delta}\right)\sqrt{\log\left(\frac{m(n_2+n_3)}{\delta}\right) }R\log\left(\frac{mn_1}{\delta}\right) \\
&\quad +n_1d\log\left(\frac{m}{\delta}\right)\log\left(\frac{m(n_2+n_3)}{\delta}\right)\left(p\tilde{R}^2+\sqrt{p}\tilde{R}\sqrt{\left( \frac{mn_1}{\delta}\right) } \right)
\end{aligned}
\end{equation}
and
with probability at least $1-\delta-n_1exp(-mp\tilde{R})$, 
\begin{equation}
\|\tilde{H}-\tilde{H}(0)\|_F\lesssim n_1d\log\left(\frac{m}{\delta}\right)\sqrt{\log\left(\frac{m(n_2+n_3)}{\delta}\right) }\tilde{R}+n_1 \left(d\log\left(\frac{m}{\delta}\right)\right)^{\frac{3}{2}}\log\left(\frac{m(n_2+n_3)}{\delta}\right) R .
\end{equation}
\end{lemma}

Similar to training neural operators, we can derive the training dynamics of physics-informed neural operators.

\begin{lemma}
For all $t\in \mathbb{N}$, we have
\[ G^{t+1}(u)=[I-\eta(H(t)+\tilde{H}(t))]G^{t}(u)+I(t),\]
where $I(t)\in \mathbb{R}^{n_1(n_2+n_3)}$, $I(t)$ can be divided into $n_1$ blocks, where each block is an $(n_2+n_3)$dimensional vector. The $j_1$-th ($j_1\in [n_2]$) component of $i$-th block is 
\[ s^{t+1}(u_i)(y_{j_1})-s^{t}(u_i)(y_{j_1})-\left\langle \frac{\partial s^{t}(u_i)(y_{j_1})}{\partial w}, w(t+1)-w(t) \right\rangle-\left\langle \frac{\partial s^{t}(u_i)(y_{j_1})}{\partial \tilde{w}}, \tilde{w}(t+1)-\tilde{w}(t) \right\rangle,\]  
The $n_2+j_2$-th ($j_1\in [n_3]$) component of $i$-th block is
\[ h^{t+1}(u_i)(\tilde{y}_{j_2})-h^{t}(u_i)(\tilde{y}_{j_2})-\left\langle \frac{\partial h^{t}(u_i)(\tilde{y}_{j_2})}{\partial w}, w(t+1)-w(t) \right\rangle-\left\langle \frac{\partial h^{t}(u_i)(\tilde{y}_{j_2})}{\partial \tilde{w}}, \tilde{w}(t+1)-\tilde{w}(t) \right\rangle.\]
\end{lemma}

With these preparations in place, we can now arrive at the final convergence theorem.

\begin{theorem}
If we set $\eta=\mathcal{O}\left(\frac{1}{\|H^{\infty}\|_2+\|\tilde{H}^{\infty}\|_2}\right)$, then with probability at least $1-\delta$, we have
\[ \|G^{t}(u)\|_2^2\leq \left( 1-\frac{\eta(\lambda_0+\tilde{\lambda}_0)}{2}\right)^t \|G^{0}(u)\|_2^2,\]
where
\[ m=\tilde{\Omega}\left( \frac{n_1^4d^7}{(\lambda_0+\tilde{\lambda}_0)^2 min(\lambda_0^2, \tilde{\lambda}_0^2)} \right)\]
and $\tilde{\Omega}$ indicates that some terms involving $\log(n_1)$, $\log(n_2)$ and $\log(m)$ are omitted. 

\end{theorem}

We prove Theorem 3 by induction. Our induction hypothesis is just the following condition:

\begin{condition}
At the $s$-th iteration, we have
\[\|G^{s}(u)\|_2^2\leq \left(1-\frac{\eta(\lambda_0+\tilde{\lambda}_0)}{2} \right)^{s}\|G^{0}(u)\|_2^2\]
and $\|w_r(s)\|_2\leq B_1$, $|w_r(s)^T y_j|\leq B_2$ and $|w_r(s)^T \tilde{y}_{j_1}|\leq B_2$ holds for all $r\in [m], j\ in [n_2], j_1\in [n_3] $, where
\[B_1=2\sqrt{d\log(m/\delta)}, \ B_2=2\sqrt{\log(m(n_2+n_3)/\delta)}.\]
\end{condition}

This condition directly yields the following bound of deviation from the initialization.

\begin{corollary}
If Condtion 2 holds for $s=0,\cdots, T$, then we have that
\[\|\tilde{w}_{rk}(T+1)-\tilde{w}_{rk}(0)\|_2\lesssim \frac{\sqrt{n_1}B_1^2B_2\|G^{0}(u)\|_2}{\sqrt{mp}(\lambda_0+\tilde{\lambda}_0)} \]
and
\[\|w_r(T+1)-w_r(0)\|_2\lesssim \frac{\sqrt{n_1}B_1B_2\|G^{0}(u)\|_2}{\sqrt{mp}(\lambda_0+\tilde{\lambda}_0)}\sqrt{\log\left(\frac{mn_1}{\delta}\right)}\]
holds for all $r\in [m], k \in [p]$.  
\end{corollary}

\begin{lemma}
If Condtion 2 holds for $s=0,\cdots, T$, then we have that
\[\|I(s)\|_2\lesssim \frac{\eta  (n_1)^{3/2}B_1^6B_2^3\|G^{0}(u)\|_2}{\sqrt{m}(\lambda_0+\tilde{\lambda}_0)}\log^{3/2}\left( \frac{mn_1}{\delta}\right) \|G^t(u)\|_2.\]
holds for $s=0,\cdots, T-1$.
\end{lemma}

\section{Conclusion and Future Work}
In this paper, we have analyzed the convergence of gradient descent (GD) in training wide shallow neural operators within the framework of NTK, demenstrating the linear convergence of GD. The core idea is that over-parameterization ensures that all weights are close to their initializations for all iterations, which is similar to performing a certain kernel method. There are some future works. Firstly,
the extension of our theory to other neural operators, like FNO. The main difficulty could be that how to meet the requirements of the NTK theory. Secondly, the extension to DeepONets, which we think may be similar to the extension from the results in \cite{7} to \cite{9}.

\bibliographystyle{IEEEtran}
\bibliography{NeuralOperator.bib}

\clearpage

\section*{Appendix}

Before the proofs, we first define the events 
\begin{equation}
	A_{jr}:=\{\exists w: \|w-w_r(0)\|_2\leq R , I\{w^T y_j\geq 0\}\neq I\{w_r(0)^T y_j\geq 0\} \}
\end{equation}
and 
\begin{equation}
\tilde{A}_{rk}^{i}:=\{\exists w: \|w-w_{rk}(0)\|_2\leq \tilde{R} , I\{w^T u\geq 0\}\neq I\{w_{rk}(0)^T u_i\geq 0\} \}
\end{equation}
for all $i\in [n_1],j\in [n_2], r\in [m], p\in [k]$. 

Note that the event $A_{jr}$ happens if and only if $|w_r(0)^T y_j|<\|y_j\|_2R$, thus by the anti-concentration inequality of Gaussian distribution (Lemma 10),  we have
\begin{equation}
	P(A_{ir})=P_{z\sim \mathcal{N}(0,\|y_j\|_2^2) }( |z|<\|y_j\|_2R)=P_{z\sim \mathcal{N}(0,1) }( |z|<R)\lesssim R.	
\end{equation}
Similarly, we have $P(\tilde{A}_{rk}^{i}) \lesssim \tilde{R}.$

Moreover, we let $S_j:=\{r\in [m]: I\{A_{jr}\}=0\}$, $S_j^{\perp}:= [m]\backslash S_j$ and $\tilde{S}_i:=\{(r,k)\in [m]\times [p]: I\{\tilde{A}_{rk}^{i}\}=0\}$, $\tilde{S}^{\perp}_{i}:=[m]\times [p] \backslash \tilde{S}_i$. 

\section{Proof of Continuous Time Analysis}

\subsection{Proof of Lemma 1}
\begin{proof}
First, recall that $H^{\infty}$ is a Kronecker product of $H_1^{\infty}$ and $H_2^{\infty}$. The $(i,j)$-th entry of $H_1^{\infty}$ is $\mathbb{E}[\sigma(\tilde{w}^T u_i)\sigma(\tilde{w}^T u_j)]$ and the $(i_1,j_1)$-th entry of $H_2^{\infty}$ is $\mathbb{E}[y_{i_1}^T y_{j_1} I\{w^T y_{i_1}\geq 0, w^T y_{j_1}\geq 0\}]$. As we know, the Kronecker product of two strictly positive definite matrices is also strictly positive definite. Thus, it suffices to demonstrate that $H_1^{\infty}$ and $H_2^{\infty}$ are both strictly positive definite.

The proof relies relies on standard functional analysis. Let $\mathcal{H}$ be the Hilbert space of integrable $d$-dimensional vector fields on $\mathbb{R}^d$, i.e., $f\in \mathcal{H}$ if $\mathbb{E}_{w\sim \mathcal{N}(\bm{0},\bm{I}) }[\|f(w)\|_2^2]<\infty.$ Then the inner product of this Hilbert space is $\langle f,g\rangle_{\mathcal{H}} = \mathbb{E}_{w\sim \mathcal{N}(\bm{0},\bm{I}) }[f(w)^T g(w)]$ for $f,g \in \mathcal{H}$. With these preparations in place, to prove $H_2^{\infty}$ is strictlt positive definite, it is equivalent to show $\psi(y_1), \cdots, \psi(y_{n_2})\in \mathcal{H}$ are linearly independent, where $\psi(y_j) = y_jI\{w^T y_j \geq 0\}$ for $j \in [n_2]$. This is exactly the result of Therem 3.1 in \cite{7}. Similarly, Theorem 2.1 in \cite{8} shows that $\{\sigma(\tilde{w}^T u_i)\}_{i=1}^{n_1}$ are linearly independent if no two samples in $\{u_i\}_{i=1}^{n_1}$ are parallel. Thus it can be directly deduced that $H^{\infty}$ is strictly positive definite. Similarly, we can deduce that $\tilde{H}^{\infty}$ is also strictly positive definite.

As for other activation functions, such as $\text{ReLU}^p$ or smooth activation functions, similar conclusions hold true. For specific details, refer to \cite{9} and \cite{10}.

\end{proof}

\subsection{Proof of Lemma 2}
\begin{proof}
First, let 
\[ X_r = \left[\frac{1}{\sqrt{p}} \sum\limits_{k=1}^p \tilde{a}_{rk} \sigma(\tilde{w}_{rk}(0)^T u_i )  \right] \left[\frac{1}{\sqrt{p}} \sum\limits_{k=1}^p \tilde{a}_{rk} \sigma(\tilde{w}_{rk}(0)^T u_j )  \right]y_{i_1}^T y_{j_1} I\{w_r(0)^T y_{i_1} \geq 0, w_r(0)^T y_{j_1} \geq 0\},\]
then $|H_{i,j}^{i_1,j_1}(0)-H_{i,j}^{i_1,j_1,\infty}| = \left| \frac{1}{m}\sum\limits_{r=1}^m X_r- \mathbb{E}[X_1]\right| $.

Note that Lemma 13 implies that for all $i\in [n_1]$, $\left\|\frac{1}{\sqrt{p}} \sum\limits_{k=1}^p \tilde{a}_{rk} \sigma(\tilde{w}_{rk}(0)^T u_i )  \right\|_{\psi_2} =\mathcal{O}(1)$, which yields that 
\[\|X_r\|_{\psi_1}\lesssim \left\| \frac{1}{\sqrt{p}} \sum\limits_{k=1}^p \tilde{a}_{rk} \sigma(\tilde{w}_{rk}(0)^T u_i )   \right\|_{\psi_2} \left\| \frac{1}{\sqrt{p}} \sum\limits_{k=1}^p \tilde{a}_{rk} \sigma(\tilde{w}_{rk}(0)^T u_j )   \right\|_{\psi_2} =\mathcal{O}(1).\]

Thus, by applying Lemma 12, we can deduce for fixed $(i,j)$ and $(i_1,j_1)$, with probability at least $1-\delta$,
\[ |H_{i,j}^{i_1,j_1}(0)-H_{i,j}^{i_1,j_1,\infty}| \lesssim \sqrt{\frac{\log(\frac{1}{\delta})}{m} } + \frac{\log(\frac{1}{\delta})}{m}.\]

Taking a union bound yields that with probability at least $1-\delta$,
\begin{align*}
\|H(0)-H^\infty\|_F^2 &= \sum\limits_{i,j=1}^{n_1} \sum\limits_{i_1,j_1=1}^{n_2}|H_{i,j}^{i_1,j_1}(0)-H_{i,j}^{i_1,j_1,\infty}|^2 \\
&\lesssim n_1^2 n_2^2 \left(\sqrt{\frac{\log(\frac{n_1n_2}{\delta})}{m} } + \frac{\log(\frac{n_1n_2}{\delta})}{m}\right)^2 \\
&\lesssim \frac{n_1^2n_2^2}{m}\log(\frac{n_1n_2}{\delta}).
\end{align*}

Thus, if $m=\Omega(\frac{n_1^2n_2^2\log(n_1n_2/\delta)}{\lambda_0^2})$, we have 
$\|H(0)-H^{\infty}\|_2\leq \|H(0)-H^{\infty}\|_F\leq \lambda_0/4$, resulting in that $\lambda_{min}(H(0)) \geq 3\lambda_0/4$.

On the other hand, 
\[ \left\|\frac{1}{p}\sum\limits_{k=1}^p u_i^T u_j I\{\tilde{w }_{rk}^T u_i\geq 0,\tilde{w }_{rk}^T u_j\geq 0 \}\sigma(w_r^T y_i)\sigma(w_r^T y_j)  \right\|_{\psi_1}\leq \|\sigma(w_r^T y_{i_1})\sigma(w_r^T y_{j_1})  \|_{\psi_1}= \mathcal{O}(1).\]

Similarly, applying Lemma 12 yields that with probability at least $1-\delta$,
\[ \|\tilde{H}(0)-\tilde{H}^{\infty}\|_2\leq \tilde{H}(0)-\tilde{H}^{\infty}\|_F\lesssim n_1n_2\sqrt{\frac{\log(\frac{n_1n_2}{\delta})}{m} },\]
which leads to the desired conclusion.

\end{proof}

\subsection{Proof of Lemma 3}
\begin{proof}
	
First, for $H$, recall that the $(i_1,j_1)$-th entry of $(i,j)$-th block is 	
\[\frac{1}{m} \sum\limits_{r=1}^m \left[\frac{1}{\sqrt{p}}\sum\limits_{k=1}^p \tilde{a}_{rk}\sigma(\tilde{w}_{rk}^T u_i) \right] \left[\frac{1}{\sqrt{p}}\sum\limits_{k=1}^p \tilde{a}_{rk}\sigma(\tilde{w}_{rk}^T u_j) \right] y_{i_1}^T y_{j_1} I\{ w_r^Ty_{i_1}\geq 0, w_r^Ty_{j_1}\geq 0\}.\]	
	
We let 
\[ a=\frac{1}{\sqrt{p}}\sum\limits_{k=1}^p \tilde{a}_{rk}\sigma(\tilde{w}_{rk}^T u_i),b=\frac{1}{\sqrt{p}}\sum\limits_{k=1}^p \tilde{a}_{rk}\sigma(\tilde{w}_{rk}^T u_j),c =y_{i_1}^T y_{j_1} I\{ w_r^Ty_{i_1}\geq 0, w_r^Ty_{j_1}\geq 0\}\]
and let $a(0),b(0),c(0)$ be the initialized parts corresponding to $a,b,c$  respectively.

Note that we can decompose $abc-a(0)b(0)c(0)$ as follows:
\[ abc-a(0)b(0)c(0)=(ab-a(0)b(0))c+a(0)b(0)(c-c(0)). \]	
	
For the first part $(ab-a(0)b(0))c$, from the bounedness of $c$ and $c(0)$, we have
\begin{equation*}
|ab-a(0)b(0)|\leq |a-a(0)||b|+|b-b(0)||a(0)|\leq |a-a(0)||b-b(0)|+|a-a(0)||b(0)|+|b-b(0)||a(0)|.
\end{equation*}	
	
For $a-a(0)$ and $b-b(0)$, we have
\[\left| \frac{1}{\sqrt{p}}\sum\limits_{k=1}^p \tilde{a}_{rk}\sigma(\tilde{w}_{rk}^T u_i) -\frac{1}{\sqrt{p}}\sum\limits_{k=1}^p \tilde{a}_{rk}\sigma(\tilde{w}_{rk}(0)^T u_i) \right| \leq \sqrt{p}\tilde{R},\]	
i.e., $|a-a(0)|,|b-b(0)|\lesssim \sqrt{p}\tilde{R}$. Moreover, Lemma 13 shows that $|a(0)|,|b(0)|\lesssim \sqrt{\log (mn_1/\delta)}$. Thus, combining these facts yields that
\begin{equation}
|(ab-a(0)b(0)) c| \lesssim p\tilde{R}^{2}+\sqrt{p}\tilde{R} \sqrt{\log \left(\frac{mn_1}{\delta} \right)}.
\end{equation}	
	
For the second part $a(0)b(0)(c-c(0))$, note that
\begin{equation}
\begin{aligned}
&\left|I\{w_r^T y_{i_1} \geq 0, w_r^T y_{j_1} \geq 0\}-I\{w_r(0)^T y_{i_1} \geq 0, w_r(0)^T y_{j_1} \geq 0\} \right| \\
&\leq 	\left|I\{w_r^T y_{i_1} \geq 0\}-I\{w_r(0)^T y_{i_1} \geq 0\} \right|+\left|I\{w_r^T y_{j_1}\geq 0\}-I\{w_r(0)^T y_{j_1} \geq 0\} \right|\\
&\leq I\{A_{i_1,r}\}+I\{A_{j_1,r}\}.
\end{aligned}
\end{equation}

From the Bernstein inequality (Lemma 11), we have that with probability at least $1-n_2 exp(-mR)$,
\begin{equation}
\frac{1}{m}\sum\limits_{r=1}^m I\{A_{i,r}\} \lesssim R
\end{equation}
holds for any $i\in [n_1]$.

Therefore, we can deduce that 
\begin{equation}
\begin{aligned}
&|H_{i,j}^{i_1,j_1}-H_{i,j}^{i_1,j_1}(0)|\\
&\lesssim p\tilde{R}^{2}+\sqrt{p}\tilde{R} \sqrt{\log \left(\frac{m}{\delta} \right)} + \frac{1}{m} \sum\limits_{r=1}^m \log \left(\frac{m}{\delta} \right) (I\{A_{i_1,r}\}+I\{A_{j_1,r}\}) \\
&\lesssim p\tilde{R}^{2}+\sqrt{p}\tilde{R} \sqrt{\log \left(\frac{mn_1}{\delta} \right)} +R \log \left(\frac{mn_1}{\delta} \right).
\end{aligned}
\end{equation}

Summing $i,j,i_1,j_1$ yields that
\begin{equation}
\begin{aligned}
\|H-H(0)\|_F &= \sqrt{\sum\limits_{i,j=1}^{n_1} \sum\limits_{i_1,j_1=1}^{n_2} |H_{i,j}^{i_1,j_1}-H_{i,j}^{i_1,j_1}(0)|^2} \\
&\lesssim n_1n_2p\tilde{R}^{2}+n_1n_2\sqrt{p}\tilde{R} \sqrt{\log \left(\frac{mn_1}{\delta} \right)}+n_1n_2R\log \left(\frac{mn_1}{\delta} \right)
\end{aligned}
\end{equation}
holds with probability at least $1-\delta-n_2exp(-mR)$. 

Second, for $\tilde{H}$, recall that $(i_1,j_1)$-th entry of $(i,j)$-th block is
\[ \frac{1}{m}\frac{1}{p} \sum\limits_{r=1}^m \sum\limits_{k=1}^{p} u_i^T u_j I\{\tilde{w}_{rk}^T u_i\geq 0,\tilde{w}_{rk}^T u_j\geq 0 \}  \sigma(w_r^Ty_{i_1}) \sigma(w_r^Ty_{j_1}).\]

Let $a=\sigma(w_r^Ty_{i_1}), b=\sigma(w_r^Ty_{j_1}),c = u_i^T u_j I\{\tilde{w}_{rk}^T u_i\geq 0,\tilde{w}_{rk}^T u_j\geq 0 \}$ and $a(0),b(0),c(0)$ be the corresponding initialized parts. 

Similarly, we decompose $abc-a(0)b(0)c(0)$ as follows:
\[ abc-a(0)b(0)c(0)=(ab-a(0)b(0))c+a(0)b(0)(c-c(0)). \]	

For the first part $(ab-a(0)b(0))c$, we have
\begin{equation}
\begin{aligned}
&| \sigma(w_r^T y_{i_1})\sigma(w_r^T y_{j_1})-\sigma(w_r(0)^T y_{i_1})\sigma(w_r(0)^T y_{j_1}) |\\
&=|[(\sigma(w_r^T y_{i_1})-\sigma(w_r(0)^T y_{i_1})) ]\sigma(w_r^T y_{j_1}) +\sigma(w_r(0)^T y_{i_1}) [\sigma(w_r^T y_{j_1})-\sigma(w_r(0)^T y_{j_1})]  | \\
&\lesssim R(|\sigma(w_r^T y_{j_1})|+|\sigma(w_r(0)^T y_{i_1})|) \\
&\lesssim R(|\sigma(w_r(0)^T y_{j_1})|+|\sigma(w_r(0)^T y_{i_1})|+R),
\end{aligned}
\end{equation}
thus we have
\begin{equation}
	|(ab-a(0)b(0))c|\lesssim \frac{1}{m}\sum\limits_{r=1}^m R(|\sigma(w_r(0)^T y_{j_1})|+|\sigma(w_r(0)^T y_{i_1})|)+R^{2}.
\end{equation}

Note that $\left \|\sigma(w_r(0)^T y_j) \right\|_{\psi_2} =\mathcal{O}(1) $ for all $j\in [n_2]$, then applying Lemma 12 yields that with probability at least $1-\delta$,
\[ \frac{1}{m} \sum\limits_{r=1}^m |\sigma(w_r(0)^T y_j)|\lesssim \mathbb{E}[|\sigma(w_1(0)^T y_j)| ]+ \sqrt{\log\left(\frac{n_2}{\delta}\right)} \lesssim \sqrt{\log\left(\frac{n_2}{\delta}\right)}\]
holds for all $j \in [n_2]$.

For the second part $a(0)b(0)(c-c(0))$, we cannot directly apply the Bernstein inequality. Instead, we first truncate $|\sigma(w_r(0)^T y_{i_1})\sigma(w_r(0)^T y_{j_1})|$. Note that for $w_r(0)^T y_j$, we have $P(|w_r(0)^T y_j|> \|y_j\|_2 t)\leq 2e^{-t^2/2}$, i.e., with probability at least $1-\delta$, $ |w_r(0)^T y_j|\leq \sqrt{2\log(\frac{2}{\delta})}$. Thus, taking a union bound yields that with probability at least $1-\delta$,
\[ |\sigma(w_r(0)^T y_j)|\leq |w_r(0)^T y_j|\lesssim \sqrt{\log\left(\frac{mn_2}{\delta}\right)}:=M\] 
holds for any $r\in [m], j \in [n_2]$.

Therefore, under this event, 
\begin{equation}
\begin{aligned}
&|\sigma(w_r(0)^T y_{i_1})\sigma(w_r(0)^T y_{j_1}) (I\{\tilde{w }_{rk}^T u_i\geq 0,\tilde{w }_{rk}^T u_j\geq 0\}-I\{\tilde{w }_{rk}(0)^T u_i\geq 0,\tilde{w }_{rk}^T u_j\geq 0\})| \\
&\lesssim M^2 |I\{\tilde{w }_{rk}^T u_i\geq 0,\tilde{w }_{rk}^T u_j\geq 0\}-I\{\tilde{w }_{rk}(0)^T u_i\geq 0,\tilde{w }_{rk}^T u_j\geq 0\}|\\
&\leq M^2 |I\{\tilde{w }_{rk}^T u_i\geq 0\}-I\{\tilde{w }_{rk}(0)^T u_i\geq 0\}|+M^2 |I\{\tilde{w }_{rk}^T u_i\geq 0\}-I\{\tilde{w }_{rk}(0)^T u_i\geq 0\}|\\
&\leq M^2 (I\{\tilde{A}_{rk}^{i}\}+I\{\tilde{A}_{rk}^{j}\}).
\end{aligned}
\end{equation}

From the Bernstein inequality, we have that with probability at least $1-n_1exp(-mp\tilde{R})$,
\[\frac{1}{m}\frac{1}{p}\sum\limits_{r=1}^m \sum\limits_{k=1}^p I\{\tilde{A}_{rk}^{i}\} \lesssim \tilde{R}. \]

Thus, with probability at least $1-\delta-n_1exp(-mp\tilde{R})$,
\begin{align*}
|\tilde{H}_{i,j}^{i_1,j_1}-\tilde{H}_{i,j}^{i_1,j_1}(0)|
&\lesssim R\sqrt{\log \left(\frac{n_2}{\delta} \right)}+ \tilde{R}\log \left(\frac{mn_2}{\delta} \right).
\end{align*}

Summing $i,j,i_1,j_1$ yields that
\begin{equation}
\begin{aligned}
\|\tilde{H}-\tilde{H}(0)\|_F &= \sqrt{\sum\limits_{i,j=1}^{n_1} \sum\limits_{i_1,j_1=1}^{n_2} |\tilde{H}_{i,j}^{i_1,j_1}-\tilde{H}_{i,j}^{i_1,j_1}(0)|^2} \\
&\lesssim n_1n_2R\sqrt{\log \left(\frac{n_2}{\delta} \right)}+n_1n_2\tilde{R}\log \left(\frac{mn_2}{\delta} \right)
\end{aligned}
\end{equation}
holds with probability at least $1-\delta-n_1exp(-mp\tilde{R})$.

\end{proof}

\subsection{Proof of Theorem 1}

The proof of Theorem 1 consists of the following Lemma 6, Lemma 7, Lemma 8 and Lemma 9. First, we assume that the following lemmas are considered in the setting of events in Lemma 13, Lemma 14 and $\{|w_r(0)^Ty_j| \leq B, \forall j \in [n_2], \forall r\in [m]\}$, where $B=2\sqrt{\log{(mn/\delta)}}$.

\begin{lemma}
If for $0\leq s \leq t $, $\lambda_{min}(H(s))\geq \frac{\lambda_0}{2} $, $\lambda_{min}(\tilde{H}(s))\geq \frac{\tilde{\lambda}_0}{2} $, then we have 
\[  \|z-G^{t}(u)\|_2^2 \leq exp(-(\lambda_0+\tilde{\lambda}_0)t)\|z-G^{0}(u)\|_2^2.\]
\end{lemma}

\begin{proof}
From the conditions  $\lambda_{min}(H(s))\geq \frac{\lambda_0}{2} $ and $\lambda_{min}(\tilde{H}(s))\geq \frac{\tilde{\lambda}_0}{2} $, we can deduce that 
\begin{align*}
\frac{d}{dt} \|z-G^{t}(u)\|_2^2 &= -2(z-G^{t}(u))^T (H(t)+\tilde{H}(t)) (z-G^{t}(u)) \\
&\leq -(\lambda_0+\tilde{\lambda}_0) \|z-G^{t}(u)\|_2^2.
\end{align*}
From this, we have
\[   \frac{d}{dt} \left(exp((\lambda_0+\tilde{\lambda}_0)t)  \|z-G^{t}(u)\|_2^2   \right) \leq 0,\]
which yields that 
\[exp((\lambda_0+\tilde{\lambda}_0)t)  \|z-G^{t}(u)\|_2^2  \leq \|z-G^{0}(u)\|_2^2,\]
i.e., 
\[  \|z-G^{t}(u)\|_2^2 \leq exp(-(\lambda_0+\tilde{\lambda}_0)t)\|z-G^{0}(u)\|_2^2.\]
\end{proof}

\begin{lemma}
Suppose for $0\leq s\leq t$,  $\lambda_{min}(H(s))\geq \frac{\lambda_0}{2} $, $\lambda_{min}(\tilde{H}(s))\geq \frac{\tilde{\lambda}_0}{2} $ and $\|\tilde{w}_{rk}(s)-\tilde{w}_{rk}(0)\|_2\leq \tilde{R}$ holds for any $r\in [m], k\in [p]$, then we have that
\[\|w_r(s)-w_r(0)\|_2\leq \frac{C\sqrt{n_1n_2}\|z-G^{0}(u)\|_2}{\sqrt{m}(\lambda_0+\tilde{\lambda}_0)}\left(\sqrt{p} \tilde{R}+\sqrt{\log \left( \frac{mn_1}{\delta}\right) } \right):=R^{'}\]
holds for any $r\in [m]$, where $C$ is a universal constant.	
\end{lemma}

\begin{proof}
For $0\leq s \leq t$, we have
\begin{equation}
\begin{aligned}
&\left \|\frac{d}{dt} w_r(s)\right\|_2 \\
&=\left\|\frac{\partial L(W(s), \tilde{W}(s))}{\partial w_r} \right\|_2\\
&= \left\|\sum\limits_{i=1}^{n_1}\sum\limits_{j=1}^{n_2}  \left(  \frac{1}{\sqrt{m}} \left[\frac{1}{\sqrt{p}}\sum\limits_{k=1}^p \tilde{a}_{rk}\sigma(\tilde{w}_{rk}(s)^T u_i)\right]y_jI\{w_r(s)^Ty_j\geq 0\}  \right)(G^{s}(u_i)(y_j)-z_i^j) \right\|_2 \\
&\leq \frac{\sqrt{n_1n_2}}{\sqrt{m}}\left(\sqrt{p} \tilde{R}+\sqrt{\log \left( \frac{mn_1}{\delta}\right) } \right)\|G^{s}(u)-z\|_2\\
&\leq \frac{\sqrt{n_1n_2}}{\sqrt{m}}\left(\sqrt{p} \tilde{R}+\sqrt{\log \left( \frac{mn_1}{\delta}\right) } \right) exp(-\frac{( \lambda_0+\tilde{\lambda}_0  )s}{2})\|G^{0}(u)-z\|_2,
\end{aligned}
\end{equation}
where the last inequality follows from Lemma 6 and the first inequality follows from that
\begin{align*}
&\left|\frac{1}{\sqrt{p}}\sum\limits_{k=1}^p \tilde{a}_{rk}\sigma(\tilde{w}_{rk}(s)^T u_i)\right| \\
&=\left|\frac{1}{\sqrt{p}}\sum\limits_{k=1}^p \left[ \tilde{a}_{rk}\sigma(\tilde{w}_{rk}(s)^T u_i)-\tilde{a}_{rk}\sigma(\tilde{w}_{rk}(0)^T u_i) \right]+\tilde{a}_{rk}\sigma(\tilde{w}_{rk}(0)^T u_i) \right| \\
&\leq \left|\frac{1}{\sqrt{p}}\sum\limits_{k=1}^p \left[ \tilde{a}_{rk}\sigma(\tilde{w}_{rk}(s)^T u_i)-\tilde{a}_{rk}\sigma(\tilde{w}_{rk}(0)^T u_i) \right]\right| +\left|\frac{1}{\sqrt{p}}\sum\limits_{k=1}^p \tilde{a}_{rk}\sigma(\tilde{w}_{rk}(0)^T u_i) \right| \\
&\lesssim \sqrt{p}\tilde{R}+ \sqrt{\log \left( \frac{mn_1}{\delta}\right) }.
\end{align*}

Therefore, we have
\[ \|w_r(t)-w_r(0)\|_2\leq \int_{0}^t  \left\| \frac{d}{ds} w_r(s) \right\|_2 ds \leq \frac{C\sqrt{n_1n_2}\|z-G^{0}(u)\|_2}{\sqrt{m}(\lambda_0+\tilde{\lambda}_0)}\left(\sqrt{p} \tilde{R}+\sqrt{\log \left( \frac{mn_1}{\delta}\right) } \right):=R^{'},    \]
where $C$ is a universal constant.
\end{proof}

\begin{lemma}
Suppose for $0\leq s\leq t$,  $\lambda_{min}(H(s))\geq \frac{\lambda_0}{2} $, $\lambda_{min}(\tilde{H}(s))\geq \frac{\tilde{\lambda}_0}{2} $ and $\|w_r(s)-w_r(0)\|_2\leq R$ holds for any $r\in [m]$, then we have that
\[ \|\tilde{w}_{rk}(t)-\tilde{w}_{rk}(0) \|_2\leq \frac{C\sqrt{n_1n_2}(R+B)\|z-G^{0}(u)\|_2}{\sqrt{mp}(\lambda_0+\tilde{\lambda}_0)}:=\tilde{R}^{'}\]	
holds for any $r\in [m], p \in [k]$, where $C$ is a universal constant.
\end{lemma}

\begin{proof}
For $0\leq s\leq t$, we have
\begin{equation}
\begin{aligned}
\left \|\frac{d}{dt} \tilde{w}_{rk}(s)\right\|_2 &=\left\|\frac{\partial L(W(s), \tilde{W}(s))}{\partial \tilde{w}_{rk}} \right\|_2\\
&= \left\|\sum\limits_{i=1}^{n_1}\sum\limits_{j=1}^{n_2}  \frac{1}{\sqrt{m}} \frac{\tilde{a}_{rk}}{\sqrt{p}}  u_iI\{\tilde{w}_{rk}(s)^Tu_i\geq 0 \} \sigma(w_r(s)^T y_j)  (G^{s}(u)(y_j)-z_j^i) \right\|_2 \\
&\lesssim \frac{1}{\sqrt{mp}} \sum\limits_{i=1}^{n_1}\sum\limits_{j=1}^{n_2} \left|  \sigma(w_r(s)^T y_j)  (G^{s}(u_i)(y_j)-z_j^i)\right|\\
&\leq \frac{1}{\sqrt{mp}} \sum\limits_{i=1}^{n_1}\sum\limits_{j=1}^{n_2} \left(\left|  \sigma(w_r(s)^T y_j)-\sigma(w_r(0)^T y_j) \right|+ \left|  \sigma(w_r(0)^T y_j)  \right|\right) \left| G^{s}(u_i)(y_j)-z_j^i\right| \\
&\leq \frac{1}{\sqrt{mp}} \sum\limits_{i=1}^{n_1}\sum\limits_{j=1}^{n_2} (R+B)\left| G^{s}(u_i)(y_j)-z_j^i\right|\\
&\leq \frac{\sqrt{n_1n_2}(R+B)}{\sqrt{mp}}\|G^{s}(u)-z\|_2\\
&\leq  \frac{\sqrt{n_1n_2}(R+B)}{\sqrt{mp}}exp(-(\lambda_0+\tilde{\lambda}_0)s/2)\|z-G^{0}(u)\|_2,
\end{aligned}
\end{equation}
where the last inequality follows from Lemma 6. Then, similar to that in Lemma 7, the conclusion holds.

\end{proof}

\begin{lemma}
If $R^{'}< R$ and $\tilde{R}^{'} < \tilde{R}$, we have that for all $t \geq 0$, the following two conclusions hold:
\begin{itemize}
\item $\lambda_{min}(H(t))\geq \frac{\lambda_0}{2} $ and $\lambda_{min}(\tilde{H}(t))\geq \frac{\tilde{\lambda}_0}{2} $;

\item $\|w_r(t)-w_r(0)\|_2\leq R^{'}$ and $\|\tilde{w}_{rk}(t)-\tilde{w}_{rk}(0)\|_2\leq \tilde{R}^{'}$ for any $r\in [m], p\in [k]$.
\end{itemize}

\end{lemma}

\begin{proof}
The proof is based on contradiction. Suppose $t > 0$ is the smallest time that the two conclusions do not hold, then either conclusion 1 does not hold or conclusion 2 does not hold.

If conclusion 1 does not hold, i.e., either $\lambda_{min}(H(t))< \frac{\lambda_0}{2} $ or $\lambda_{min}(\tilde{H}(t))< \frac{\tilde{\lambda}_0}{2} $, then Lemma 3 implies that there exists $r\in [m]$, $\|w_r(s)-w_r(0)\|_2 > R>R^{'}$ or there exists $r\in [m], k\in [p]$, $\|\tilde{w}_{rk}(s)-\tilde{w}_{rk}(0)\|_2 > \tilde{R}>\tilde{R}^{'}$. This fact shows that conlusion 2 does not hold and then, this contradicts with the minimality of $t$.

If conclusion 2 does not hold, then either there exists $r\in [m]$, $\|w_r(t)-w_r(0)\|_2 >R^{'}$ or there exists $r\in [m], k\in [p]$, $\|\tilde{w}_{rk}(t)-\tilde{w}_{rk}(0)\|_2 >\tilde{R}^{'}$. If $\|w_r(t)-w_r(0)\|_2 >R^{'}$, then Lemma 7 implies that there exist $s<t$ such that  $\lambda_{min}(H(s))< \frac{\lambda_0}{2} $ or $\lambda_{min}(\tilde{H}(s))< \frac{\tilde{\lambda}_0}{2} $ or there exists $r\in [m], k\in [p]$, $\|\tilde{w}_{rk}(s)-\tilde{w}_{rk}(0)\|_2 > \tilde{R}^{'}$, which contradicts with the minimality of $t$. And the last case is similar to this case.

\end{proof}

\begin{proof}[Proof of Theorem 1]
Theorem 1 is a direct corollary of Lemma 6 and Lemma 9. Thus, it remains only to clarify the requirements for $m$ so that Lemma 6, Lemma 7 and Lemma 8 hold. First, $R$ and $\tilde{R}$ should ensure that $\|H(0)-H^{\infty}\|_2\leq \lambda_0/4$ and $\|\tilde{H}(0)-\tilde{H}^{\infty}\|_2\leq \tilde{\lambda}_0/4$, i.e.,
\begin{equation}
R\lesssim \frac{\min(\lambda_0, \tilde{\lambda}_0)}{n_1n_2 \log\left( \frac{m}{\delta}\right) }, \ \tilde{R}\lesssim \frac{\min(\lambda_0, \tilde{\lambda}_0)}{n_1n_2\sqrt{p} \log\left( \frac{m}{\delta}\right) }.
\end{equation}

Combining this with the requirement that $R^{'}<R, \tilde{R}^{'}<\tilde{R}$, we can deduce that
\[ m=\Omega\left( \frac{ n_1^4n_2^4\log\left( \frac{n}{\delta}\right)\log^3\left( \frac{m}{\delta}\right) }{(min(\lambda_0, \tilde{\lambda}_0))^2(\lambda_0+\tilde{\lambda}_0)^2} \right).\]

Moreover, the requirement for $m$ also leads to that 
\[ n_2exp(-mR)\lesssim \delta, n_2exp(-mp\tilde{R})\lesssim \delta,\]
which are confidences in Lemma 3.

\end{proof}

\section{Proof of Descrete Time Analysis}

\subsection{Proof of Lemma 4}

\begin{proof}
First, we can decompose $G^{t+1}(u_i)(y_j) - G^{t}(u_i)(y_j)$ as follows.
\begin{equation}
\begin{aligned}
&G^{t+1}(u_i)(y_j) - G^{t}(u_i)(y_j) \\
&= G^{t+1}(u_i)(y_j) - G^{t}(u_i)(y_j) -\left\langle \frac{\partial G^{t}(u_i)(y_j)}{\partial w},  w(t+1)-w(t) \right\rangle  - \left\langle \frac{\partial G^{t}(u_i)(y_j)}{\partial \tilde{w} },  \tilde{w}(t+1)-\tilde{w}(t) \right\rangle \\
&\quad + \left\langle \frac{\partial G^{t}(u_i)(y_j)}{\partial w},  w(t+1)-w(t) \right\rangle  + \left\langle \frac{\partial G^{t}(u_i)(y_j)}{\partial \tilde{w} },  \tilde{w}(t+1)-\tilde{w}(t) \right\rangle .\\
\end{aligned}	
\end{equation}
Note that
\begin{equation}
\begin{aligned}
&\left\langle \frac{\partial G^{t}(u_i)(y_j)}{\partial w},  w(t+1)-w(t) \right\rangle \\
&= \sum\limits_{r=1}^m  \left\langle \frac{\partial G^{t}(u_i)(y_j)}{\partial w_r},  w_r(t+1)-w_r(t) \right\rangle \\
&=-\eta\sum\limits_{i_1=1}^{n_1} \sum\limits_{j_1=1}^{n_2}\sum\limits_{r=1}^m  \left\langle \frac{\partial G^{t}(u_i)(y_j)}{\partial w_r},  \frac{\partial G^{t}(u_{i_1})(y_{j_1})}{\partial w_r} \right\rangle (G^{t}(u_{i_1})(y_{j_1})-z_{j_1}^{i_1})
\end{aligned}
\end{equation}
and
\begin{equation}
\begin{aligned}
&\left\langle \frac{\partial G^{t}(u_i)(y_j)}{\partial \tilde{w} },  \tilde{w}(t+1)-\tilde{w}(t) \right\rangle\\
&= \sum\limits_{r=1}^m \sum\limits_{k=1}^{p} \left\langle \frac{\partial G^{t}(u_i)(y_j)}{\partial \tilde{w}_{rk} },  \tilde{w}_{rk}(t+1)-\tilde{w}_{rk}(t) \right\rangle \\
&= -\eta\sum\limits_{i_1=1}^{n_1} \sum\limits_{j_1=1}^{n_2}\sum\limits_{r=1}^m  \sum\limits_{k=1}^p \left\langle \frac{\partial G^{t}(u_i)(y_j)}{\partial \tilde{w}_{rk} },  \frac{\partial G^{t}(u_{i_1})(y_{j_1})}{\partial \tilde{w}_{rk} } \right\rangle (G^{t}(u_{i_1})(y_{j_1})-z_{j_1}^{i_1}).
\end{aligned}
\end{equation}

Plugging (33) and (34) into (32 yields that
\begin{equation*}
\begin{aligned}
G^{t+1}(u_i)(y_j) - G^{t}(u_i)(y_j) &= I_{i,j}(t)-\eta [(H_1^i(t), \cdots, H_{n_1}^i)+(\tilde{H}_1^i(t), \cdots, \tilde{H}_{n_1}^i)]_j (G^t(u)-z),
\end{aligned}
\end{equation*}
where $[A]_j$ represents the $j$-th of the matrix $A$ and $I(t)\in \mathbb{R}^{n_1n_2}$, we can divide it into $n_1$ blocks, the $j$-th component of $i$-th block is defined as 
\begin{equation*}
I_{i,j}(t)= G^{t+1}(u_i)(y_j) - G^{t}(u_i)(y_j) -\left\langle \frac{\partial G^{t}(u_i)(y_j)}{\partial w},  w(t+1)-w(t) \right\rangle  - \left\langle \frac{\partial G^{t}(u_i)(y_j)}{\partial \tilde{w} },  \tilde{w}(t+1)-\tilde{w}(t) \right\rangle.	
\end{equation*}

Thus, we have
\begin{equation*}
G^{t+1}(u)-G^{t}(u)	= I(t)-\eta (H(t)+\tilde{H}(t))(G^{t}(u)-z).
\end{equation*}
By using a simple algebraic transformation, we have
\begin{equation}
\begin{aligned}
z-G^{t+1}(u) &= z-G^{t}(u)-I(t)-\eta (H(t)+\tilde{H}(t))(z-G^{t}(u)) \\
&= \left(I-\eta(H(t)+\tilde{H}(t)) \right) (z-G^{t}(u))-I(t).
\end{aligned}
\end{equation}	
\end{proof}

\subsection{Proof of Lemma 5}
\begin{proof}

We first express explicitly the $j$-component of the $i$-th bloack of the residual term $I(s)$ as follows.
\begin{equation}
\begin{aligned}
I_{i,j}(s)&=G^{s+1}(u_i)(y_j)-G^{s}(u_i)(y_j)-\left\langle \frac{\partial G^{s}(u_i)(y_j)}{\partial w} , w(s+1)-w(s)\right\rangle -\left\langle \frac{\partial G^{s}(u_i)(y_j)}{\partial \tilde{w} } , \tilde{w}(s+1)-\tilde{w}(s)\right\rangle \\
&= G^{s+1}(u_i)(y_j)-G^{s}(u_i)(y_j)-\sum\limits_{r=1}^m \left\langle \frac{\partial G^{s}(u_i)(y_j)}{\partial w_r} , w_r(s+1)-w_r(s)\right\rangle \\
&\quad -\sum\limits_{r=1}^m \sum\limits_{k=1}^p  \left\langle \frac{\partial G^{s}(u_i)(y_j)}{\partial \tilde{w}_{rk}} , \tilde{w}_{rk}(s+1)-\tilde{w}_{rk}(s)\right\rangle.
\end{aligned}
\end{equation}

From the forms of $G^{s}(u_i)(y_j)$,$\frac{\partial G^{s}(u_i)(y_j)}{\partial w_r}$ and $\frac{\partial G^{s}(u_i)(y_j)}{\partial \tilde{w}_{rk}}$, we have
\begin{equation}
\begin{aligned}
&G^{s+1}(u_i)(y_j)-G^{s}(u_i)(y_j) \\
&= \frac{1}{\sqrt{m}}\sum\limits_{r=1}^m  \left[ \frac{1}{\sqrt{p}} \sum\limits_{k=1}^p \tilde{a}_{rk}\sigma(\tilde{w}_{rk}(s+1)^T u_i) \right] \sigma(w_r(s+1)^T y_j) - \frac{1}{\sqrt{m}}\sum\limits_{r=1}^m\left[ \frac{1}{\sqrt{p}} \sum\limits_{k=1}^p \tilde{a}_{rk}\sigma(\tilde{w}_{rk}(s)^T u_i) \right] \sigma(w_r(s)^T y_j) \\	
&=\frac{1}{\sqrt{m}}\sum\limits_{r=1}^m  \left[ \frac{1}{\sqrt{p}} \sum\limits_{k=1}^p \tilde{a}_{rk}\sigma(\tilde{w}_{rk}(s+1)^T u_i) \right] (\sigma(w_r(s+1)^T y_j)-\sigma(w_r(s)^T y_j)) \\
&\quad+ \frac{1}{\sqrt{m}}\sum\limits_{r=1}^m  \left[ \frac{1}{\sqrt{p}} \sum\limits_{k=1}^p \tilde{a}_{rk}\left[\sigma(\tilde{w}_{rk}(s+1)^T u_i)-\sigma(\tilde{w}_{rk}(s)^T u_i)\right] \right] \sigma(w_r(s)^T y_j).
\end{aligned}
\end{equation}
and 
\begin{equation}
\begin{aligned}
&\left\langle \frac{\partial G^{s}(u_i)(y_j)}{\partial w_r},w_r(s+1)-w_r(s) \right\rangle \\
&=\frac{1}{\sqrt{m}} \left[ \frac{1}{\sqrt{p}} \sum\limits_{k=1}^p\tilde{a}_{rk}\sigma(\tilde{w}_{rk}(s)^T u_i) \right] I\{w_r(s)^T y_j\geq 0\}(w_r(s+1)-w_r(s))^T y_j
\end{aligned}
\end{equation}
and
\begin{equation}
\begin{aligned}
&\left\langle \frac{\partial G^{s}(u_i)(y_j)}{\partial \tilde{w}_{rk}} ,\tilde{w}_{rk}(s+1)-\tilde{w}_{rk}(s) \right\rangle \\
&=\frac{1}{\sqrt{m}} \frac{ \tilde{a}_{rk}}{\sqrt{p}} (\tilde{w}_{rk}(s+1)-\tilde{w}_{rk}(s))^T u_i I\{\tilde{w}_{rk}(s)^T u_i\geq 0 \}\sigma(w_r(s)^T y_j).	
\end{aligned}
\end{equation}

Thus, we can decompose $I_{i,j}(s)$ as follows
\[I_{i,j}(s)=\sum\limits_{r=1}^m \tilde{I}_{i,j}^r(s)+\bar{I}_{i,j}^r(s),\]
where
\begin{equation}
\begin{aligned}
\tilde{I}_{i,j}^r(s)&= \frac{1}{\sqrt{m}}\left[ \frac{1}{\sqrt{p}} \sum\limits_{k=1}^p \tilde{a}_{rk}\left[\sigma(\tilde{w}_{rk}(s+1)^T u_i)-\sigma(\tilde{w}_{rk}(s)^T u_i)-I\{ \tilde{w}_{rk}(s)^T u_i \geq 0 \}(\tilde{w}_{rk}(s+1)-\tilde{w}_{rk}(s))^T u_i\right] \right] \\
&\quad \cdot \sigma(w_r(s)^Ty_j)
\end{aligned}
\end{equation}
and 
\begin{equation}
\begin{aligned}
\bar{I}_{i,j}^r(s)&=\frac{1}{\sqrt{m}}\left[ \frac{1}{\sqrt{p}} \sum\limits_{k=1}^p \tilde{a}_{rk}\sigma(\tilde{w}_{rk}(s+1)^T u_i) \right](\sigma(w_r(s+1)^Ty_j)-\sigma(w_r(s)^Ty_j))\\
&\quad -\frac{1}{\sqrt{m}}\left[ \frac{1}{\sqrt{p}} \sum\limits_{k=1}^p \tilde{a}_{rk}\sigma(\tilde{w}_{rk}(s)^T u_i) \right]I\{w_r(s)^T y_j\geq 0\}(w_r(s+1)-w_r(s))^Ty_j .
\end{aligned}
\end{equation}

For $\tilde{I}_{i,j}^r(s)$, we replace $\tilde{R}$ in the definition of $\tilde{A}_{rk}^{i}$ by $\tilde{R}^{'}$ and still denote the event as $\tilde{A}_{rk}^{i}$ for simplicity, i.e.,
\[\tilde{A}_{rk}^{i}:=\{\exists w: \|w-\tilde{w}_{rk}(0)\|_2\leq \tilde{R}^{'} , I\{w^T u_i\geq 0\}\neq I\{\tilde{w}_{rk}(0)^T u_i\geq 0\} \}\]
and $\tilde{S}_i=\{ (r,k)\in [m]\times[p]: I\{\tilde{A}_{rk}\}=0\}$.

From the induction hypothesis, we know $\|\tilde{w}_{rk}(s+1)-\tilde{w}_{rk}(0)\|_2\leq \tilde{R}^{'}$ and $\|\tilde{w}_{rk}(s)-\tilde{w}_{rk}(0)\|_2\leq \tilde{R}^{'}$. Thus, $I\{\tilde{w}_{rk}(s+1)^T u_i\geq 0\}=I\{\tilde{w}_{rk}(s)^T u_i\geq 0\}$ holds for any $(r,k) \in \tilde{S}_i$. From this fact, we can deduce that for any $(r,k) \in \tilde{S}_i$,
\begin{equation}
\begin{aligned}
&\left|\sigma(\tilde{w}_{rk}(s+1)^T u_i)-\sigma(\tilde{w}_{rk}(s)^T u_i)-I\{ \tilde{w}_{rk}(s)^T u_i \geq 0 \}(\tilde{w}_{rk}(s+1)-\tilde{w}_{rk}(s))^T u_i\right| \\
&= | (\tilde{w}_{rk}(s+1)^T u_i) I\{\tilde{w}_{rk}(s+1)^T u_i\geq 0\}-(\tilde{w}_{rk}(s)^T u_i) I\{\tilde{w}_{rk}(s)^T u_i\geq 0\} \\
&\quad -I\{ \tilde{w}_{rk}(s)^T u_i \geq 0 \}(\tilde{w}_{rk}(s+1)-\tilde{w}_{rk}(s))^T u_i        |\\
&=| (\tilde{w}_{rk}(s+1)^T u_i) I\{\tilde{w}_{rk}(s)^T u_i\geq 0\}-(\tilde{w}_{rk}(s)^T u_i) I\{\tilde{w}_{rk}(s)^T u_i\geq 0\} \\
&\quad-I\{ \tilde{w}_{rk}(s)^T u_i \geq 0 \}(\tilde{w}_{rk}(s+1)-\tilde{w}_{rk}(s))^T u_i        |\\
&= 0.
\end{aligned}
\end{equation}

On the other hand, for any $(r,k)\in [m]\times[p]$, we have
\begin{equation}
|\sigma(\tilde{w}_{rk}(s+1)^T u_i)-\sigma(\tilde{w}_{rk}(s)^T u_i)-I\{ \tilde{w}_{rk}(s)^T u_i \geq 0 \}(\tilde{w}_{rk}(s+1)-\tilde{w}_{rk}(s))^T u_i|\lesssim \|\tilde{w}_{rk}(s+1)-\tilde{w}_{rk}(s) \|_2.
\end{equation}

Thus, combining (41), (42) and (43) yields that 
\begin{equation}
\begin{aligned}
\sum\limits_{r=1}^m|\tilde{I}_{i,j}^r(s)| &\leq \frac{B}{\sqrt{mp}} \sum\limits_{r=1}^m \sum\limits_{k=1 }^p\| \tilde{w}_{rk}(s+1)-\tilde{w}_{rk}(s) \|_2 I\{(r,k) \in \tilde{S}_i^{\perp}\}\\
&= \frac{B}{\sqrt{mp}}  \sum\limits_{r=1}^m \sum\limits_{k=1}^p\left\| -\eta \frac{\partial L(W(s), \tilde{W}(s))}{\partial \tilde{w}_{rk} } \right\|_2 I\{(r,k) \in \tilde{S}_i^{\perp}\}\\
&=\frac{B}{\sqrt{mp}}  \sum\limits_{r=1}^m \sum\limits_{k=1}^p\left\| -\eta \sum\limits_{i_1=1}^{n_1} \sum\limits_{j_1=1}^{n_2} \frac{\partial G^s(u_{i_1})(y_{j_1})}{\partial \tilde{w}_{rk} } (G^s(u_{i_1})(y_{j_1})-z_{j_1}^{i_1}) \right\|_2 I\{(r,k) \in \tilde{S}_i^{\perp}\}\\
&\leq \eta B^2  \sum\limits_{r=1}^m \sum\limits_{k=1}^p \frac{\sqrt{n_1n_2}}{mp} \|z-G^{s}(u)\|_2  I\{(r,k) \in \tilde{S}_i^{\perp}\}\\
&= \eta B^2\sqrt{n_1n_2}\|z-G^{s}(u)\|_2 \sum\limits_{r=1}^m \sum\limits_{k=1}^p \frac{1}{mp}  I\{(r,k) \in \tilde{S}_i^{\perp}\}\\
&= \eta B^2\sqrt{n_1n_2}\|z-G^{s}(u)\|_2 \sum\limits_{r=1}^m \sum\limits_{k=1}^p \frac{1}{mp}  I\{\tilde{A}_{r,k}^{i}\},
\end{aligned}
\end{equation}
where the second inequality follows from Cauchy's inequality and the form of $\frac{\partial G^s(u_i)(y_j)}{\partial \tilde{w}_{rk} }$, i.e.,
\[\frac{\partial G^s(u_i)(y_j)}{\partial \tilde{w}_{rk} }= \frac{1}{\sqrt{m}} \frac{ \tilde{a}_{rk}}{\sqrt{p}} u_i I\{\tilde{w}_{rk}(s)^T u_i\geq 0 \}\sigma(w_r(s)^Ty_j).\]

From the Bernstein inequality, we have that with probability at least $1-n_1exp(-mp\tilde{R}^{'})$,
\[\sum\limits_{r=1}^m \sum\limits_{k=1}^p \frac{1}{mp}  I\{\tilde{A}_{r,k}^{i}\} \lesssim \tilde{R}^{'}.\]

This leads to the final unpper bound:
\begin{equation}
\begin{aligned}
\sum\limits_{r=1}^m|\tilde{I}_{i,j}^r(s)| &\lesssim \eta B^2\sqrt{n_1n_2}\tilde{R}^{'} \|z-G^{s}(u)\|_2\\
&\lesssim \frac{\eta n_1n_2 \|z-G^{0}(u)\|_2}{\sqrt{mp}(\lambda_0+\tilde{\lambda}_0)}\log^{ \frac{3}{2}}\left(\frac{m}{\delta} \right) \|z-G^s(u)\|_2.
\end{aligned}
\end{equation}

It remains to bound $\bar{I}_{i,j}^r(s)$, which can be written as follows.
\begin{equation}
\begin{aligned}
&\bar{I}_{i,j}^r(s)=\frac{1}{\sqrt{m}}\left[ \frac{1}{\sqrt{p}} \sum\limits_{k=1}^p \tilde{a}_{rk}\sigma(\tilde{w}_{rk}(s+1)^T u_i) \right](\sigma(w_r(s+1)^Ty_j)-\sigma(w_r(s)^Ty_j))\\
&\quad -\frac{1}{\sqrt{m}}\left[ \frac{1}{\sqrt{p}} \sum\limits_{k=1}^p \tilde{a}_{rk}\sigma(\tilde{w}_{rk}(s)^T u_i) \right]I\{w_r(s)^T y_j\geq 0\}(w_r(s+1)-w_r(s))^Ty_j \\
&=\frac{1}{\sqrt{m}}\left[ \frac{1}{\sqrt{p}} \sum\limits_{k=1}^p \tilde{a}_{rk}(\sigma(\tilde{w}_{rk}(s+1)^T u_i)-\sigma(\tilde{w}_{rk}(0)^T u_i)) \right]\left(\sigma(w_r(s+1)^Ty_j)-\sigma(w_r(s)^Ty_j)\right) \\
&\quad - \frac{1}{\sqrt{m}} \left[ \frac{1}{\sqrt{p}} \sum\limits_{k=1}^p \tilde{a}_{rk}(\sigma(\tilde{w}_{rk}(s)^T u_i)-\sigma(\tilde{w}_{rk}(0)^T u_i)) \right]I\{w_r(s)^T y_j\geq 0\}(w_r(s+1)-w_r(s))^Ty_j \\
& +\frac{1}{\sqrt{m}}\left[ \frac{1}{\sqrt{p}} \sum\limits_{k=1}^p \tilde{a}_{rk}\sigma(\tilde{w}_{rk}(0)^T u_i) \right]\left(\sigma(w_r(s+1)^Ty_j)-\sigma(w_r(s)^Ty_j) - I\{w_r(s)^T y_j\geq 0\}(w_r(s+1)-w_r(s))^Ty_j   \right).
\end{aligned}
\end{equation} 

Note that
\[ \|\sigma(\tilde{w}_{rk}(s)^T u_i)-\sigma(\tilde{w}_{rk}(0)^T u_i)\|_2\lesssim \tilde{R}^{'},\ \|\sigma(\tilde{w}_{rk}(s+1)^T u_i)-\sigma(\tilde{w}_{rk}(0)^T u_i)\|_2\lesssim \tilde{R}^{'},  \]
thus we can bound the first term and second term by 
\begin{equation}
\frac{\sqrt{p}\tilde{R}^{'}}{\sqrt{m}}\|w_r(s+1)-w_r(s)\|_2\lesssim\frac{\sqrt{n}\|  z-G^{0}(u)\|_2}{m(\lambda_0+\tilde{\lambda}_0)} \sqrt{\log \left( \frac{m}{\delta}\right) } \|w_r(t+1)-w_r(t)\|_2.
\end{equation}

For the third term in (46), we also replace $R$ in the definition of $A_{jr}$ by $R^{'}$ and still denote the event by $A_{jr}$. Recall that 
\begin{equation}
	A_{jr}:=\{\exists w: \|w-w_r(0)\|_2\leq R , I\{w^T y_j\geq 0\}\neq I\{w_r(0)^T y_j\geq 0\} \}
\end{equation}
and $S_j=\{r\in [m]: I\{A_{jr}\}=0\}$. Note that $\|w_r(s+1)-w_r(0)\|_2\leq R^{'}$ and $\|w_r(s)-w_r(0)\|_2\leq R^{'}$, thus for $r\in S_j$, we have
$I\{w_r(s+1)^T y_j\geq 0\}=I\{w_r(s)^T y_j\geq 0\}$. Combining this fact with (46) and (47), we can deduce that 
\begin{equation}
\begin{aligned}
|\bar{I}_{i,j}^r(s)|\lesssim  \frac{\sqrt{n}\|  z-G^{0}(u)\|_2}{m(\lambda_0+\tilde{\lambda}_0)} \sqrt{\log \left( \frac{m}{\delta}\right) }\|w_r(s+1)-w_r(s)\|_2+ \frac{1}{\sqrt{m}} \sqrt{\log\left( \frac{m}{\delta}\right) }   \|w_r(s+1)-w_r(s)\|_2 I\{r \in S_j^{\perp}\}.
\end{aligned}
\end{equation}

Thus, we have to bound $\|w_r(s+1)-w_r(s)\|_2$. From the gradient descent update formula, we have
\begin{equation}
\begin{aligned}
\|w_r(s+1)-w_r(s)\|_2 &=\left\|-\eta\frac{\partial L(W(s), \tilde{W}(s))}{\partial w_r}\right\|_2 \\
&\leq \eta \|z-G^s(u)\|_2\sqrt{ \sum\limits_{i=1}^{n_1} \sum\limits_{j=1}^{n_2} \left\| \frac{\partial G^{s}(u_i)(y_j)}{\partial w_r} \right\|_2^2 }.
\end{aligned}
\end{equation}

Recall that
\[\frac{\partial G^s(u_i)(y_j)}{\partial w_r}=\frac{1}{\sqrt{m}} \left[ \frac{1}{\sqrt{p}} \sum\limits_{k=1}^p\tilde{a}_{rk}\sigma(\tilde{w}_{rk}(s)^T u_i) \right] y_jI\{w_r(s)^T y_j\geq 0\}.\]

Therefore,
\begin{equation}
\begin{aligned}
&\left\|\frac{\partial G^s(u_i)(y_j)}{\partial w_r} \right\|_2 \\
&= \left\|\frac{1}{\sqrt{m}} \left[ \frac{1}{\sqrt{p}} \sum\limits_{k=1}^p\tilde{a}_{rk}\left[\sigma(\tilde{w}_{rk}(s)^T u_i)-\sigma(\tilde{w}_{rk}(0)^T u_i)+ \sigma(\tilde{w}_{rk}(0)^T u_i)\right] \right] y_jI\{w_r(s)^T y_j\geq 0\}\right\|_2\\
&\lesssim  \frac{\sqrt{p}\tilde{R}^{'} }{\sqrt{m}}+ \frac{1}{\sqrt{m}} \sqrt{\log\left( \frac{m}{\delta}\right) }\\
&\lesssim  \frac{1}{\sqrt{m}}\frac{\sqrt{n}\|z-G^{0}(u) \|_2  }{\sqrt{m}(\lambda_0+\tilde{\lambda}_0)}\sqrt{\log\left( \frac{m}{\delta}\right) }+\frac{1}{\sqrt{m}} \sqrt{\log\left( \frac{mn_1}{\delta}\right) }\\
&\lesssim \frac{1}{\sqrt{m}}  \sqrt{\log\left( \frac{mn_1}{\delta}\right) },
\end{aligned}
\end{equation}
where the last inequality is due to us taking $m$ sufficiently large in the end.
 
Combining (50) and (51) yields that
\begin{equation}
\|w_r(s+1)-w_r(s)\|_2\lesssim  \frac{\eta \sqrt{n_1n_2}\|z-G^{s}(u)\|_2}{\sqrt{m}} \sqrt{\log\left( \frac{mn_1}{\delta}\right) }.
\end{equation}

Plugging this into (49) leads to that
\[|\bar{I}_{i,j}^r(s)|\lesssim \frac{\sqrt{n_1n_2}\|z-G^{0}(u)\|_2 }{m(\lambda_0+\tilde{\lambda}_0)} \frac{ \eta \sqrt{n_1n_2}\|z-G^{s}(u)\|_2}{\sqrt{m}} \log\left( \frac{m}{\delta}\right)  + \frac{\eta\sqrt{n_1n_2} \|z-G^s(u)\|_2}{m} I\{r\in S_j^{\perp}\}\log\left( \frac{m}{\delta}\right).\]

By applying the Bernstein's inequality, we have that with probability at least $1-n_2exp(-mR^{'})$,
\[\frac{1}{m}\sum\limits_{r=1}^m I\{r\in S_j^{\perp}\}=\frac{1}{m}\sum\limits_{r=1}^m I\{A_{jr}\}\lesssim R^{'}.\]

Therefore,
\begin{equation}
\begin{aligned}
\sum\limits_{r=1}^m |\bar{I}_{i,j}^r(s)| &\lesssim  \frac{\eta n_1n_2\|z-G^{0}(u)\|_2}{\sqrt{m}(\lambda_0+\tilde{\lambda}_0)} \log\left( \frac{m}{\delta}\right) \|z-G^{s}(u)\|_2+\frac{\eta n_1n_2 \|z-G^{0}(u)\|_2}{\sqrt{m}(\lambda_0+\tilde{\lambda}_0)}\log^{\frac{3}{2} }\left( \frac{m}{\delta}\right) \|z-G^{s}(u)\|_2 \\
&\lesssim \frac{\eta n_1n_2 \|z-G^{0}(u)\|_2}{\sqrt{m}(\lambda_0+\tilde{\lambda}_0)}\log^{\frac{3}{2} }\left( \frac{m}{\delta}\right) \|z-G^{s}(u)\|_2.
\end{aligned}
\end{equation}

From (45) and (53), we have
\begin{equation}
\begin{aligned}
|I_{i,j}(s)| &\leq \sum\limits_{r=1}^m |\tilde{I}_{i,j}^r(s)|+ \sum\limits_{r=1}^m |\bar{I}_{i,j}^r(s)| \\
&\lesssim \left(\frac{\eta n_1n_2\|z-G^{0}(u)\|_2}{\sqrt{mp}(\lambda_0+\tilde{\lambda}_0)} \log^{\frac{3}{2} }\left( \frac{m}{\delta}\right)+ \frac{\eta n_1n_2\|z-G^{0}(u)\|_2}{\sqrt{m}(\lambda_0+\tilde{\lambda}_0)} \log^{\frac{3}{2} }\left( \frac{m}{\delta}\right)\right) \|z-G^s(u)\|_2 \\
&\lesssim \frac{\eta n_1n_2\|z-G^{0}(u)\|_2}{\sqrt{m}(\lambda_0+\tilde{\lambda}_0)} \log^{\frac{3}{2} }\left( \frac{m}{\delta}\right)\|z-G^s(u)\|_2.
\end{aligned}
\end{equation}
 
Therefore, 
\begin{equation}
\|I(s)\|_2\lesssim \bar{R}\|z-G^s(u)\|_2,
\end{equation}
where 
\begin{equation}
\bar{R}= \frac{\eta (n_1n_2)^{\frac{3}{2}}\|z-G^{0}(u)\|_2}{\sqrt{m}(\lambda_0+\tilde{\lambda}_0)} \log^{\frac{3}{2}}\left( \frac{m}{\delta}\right).
\end{equation}

\end{proof}

\subsection{Proof of Corollary 1}

\begin{proof}
Note that when $\|I(s)\|_2\leq \eta(\lambda_0+\tilde{\lambda}_0)/12$, $\lambda_{min}(H(s)) \geq \lambda_0/2$, $\lambda_{min}(\tilde{H}(s)) \geq \tilde{\lambda}_0/2$ and $I-\eta(H(s)+\tilde{H}(s)) $ is positive definite, we have that for $s=0,\cdots, t-1$,
\begin{equation}
\begin{aligned}
&\|z-G^{s+1}(u)\|_2^2 \\ 
&=\| \left[I-\eta(H(s)+\tilde{H}(s))\right] (z-G^{s}(u)) \|_2^2+\|I(s)\|_2^2-2\left\langle \left(\eta(H(s)+\tilde{H}(s))\right) (z-G^{s}(u)), I(s) \right\rangle \\
&\leq (1-\eta\frac{\lambda_0+\tilde{\lambda}_0 }{2})^2 \|z-G^{s}(u)\|_2^2 +\|I(s)\|_2^2+2(1-\frac{\lambda_0+\tilde{\lambda}_0 }{2})\|z-G^{s}(u)\|_2 \|I(s)\|_2\\
&\leq \left(1-\eta (\lambda_0+\tilde{\lambda}_0)+\frac{\eta^2(\lambda_0+\tilde{\lambda}_0)^2}{4} +\left(\frac{\eta (\lambda_0+\tilde{\lambda}_0)}{12} \right)^2  + 2\frac{\eta(\lambda_0+\tilde{\lambda}_0)}{12}\right) \|z-G^s(u)\|_2^2\\ 
&\leq \left(1-\eta (\lambda_0+\tilde{\lambda}_0)+\frac{\eta(\lambda_0+\tilde{\lambda}_0)}{4} +\frac{\eta (\lambda_0+\tilde{\lambda}_0)}{12}   + 2\frac{\eta(\lambda_0+\tilde{\lambda}_0)}{12}\right) \|z-G^{s}(u)\|^2\\
&= \left(1-\eta \frac{\lambda_0+\tilde{\lambda}_0}{2} \right)\|z-G^{s}(u)\|^2.
\end{aligned}
\end{equation}
Thus, 
\[\|z-G^{s}(u)\|_2^2\leq \left(1-\eta\frac{\lambda_0+\tilde{\lambda}_0}{2} \right)^s \|z-G^{0}(u)\|_2^2\]
holds for $s=0,\cdots, t$.

Now, we have to derive the requirement for $m$ such that these conditions hold. First, from Lemma 3, when 
\[ R\lesssim \frac{\min(\lambda_0, \tilde{\lambda}_0)}{n_1n_2\log\left( \frac{m}{\delta}\right) }, \ \tilde{R}\lesssim \frac{\min(\lambda_0, \tilde{\lambda}_0)}{n_1n_2\sqrt{p}\sqrt{\log\left( \frac{m}{\delta}\right)} }, \]
we have $\|H(s)-H(0)\|_2\leq \frac{\lambda_0}{4}$, $\|\tilde{H}(s)-\tilde{H}(0)\|_2\leq \frac{\tilde{\lambda}_0}{4}$ and $\lambda_{min}(H(s))\geq \lambda_0/2$,$\lambda_{min}(\tilde{H}(s))\geq \tilde{\lambda}_0/2$. Thus, when $R^{'}<R$ and $\tilde{R}^{'}<\tilde{R}$, we have $\lambda_{min}(H(s))\geq \lambda_0/2$ and $\lambda_{min}(\tilde{H}(s))\geq \tilde{\lambda}_0/2$. Specifically, $m$ need to satisfy that
\begin{equation}
m=\Omega\left(  \frac{n_1^4n_2^4\log^3\left(\frac{m}{\delta}\right) \log \left(\frac{n_1n_2}{\delta}\right)}{(min(\lambda_0, \tilde{\lambda}_0))^2 (\lambda_0+\tilde{\lambda}_0)^2} \right) .	
\end{equation}

Moreover, at this point, we can deduce that 
\begin{align*}
\|H(s)\|_2 & \leq \|H(s)-H(0)\|_2+\|H(0)\|_2 \\
&\leq \|H(s)-H(0)\|_2+\|H(0)-H^{\infty}\|_2+\|H^{\infty}\|_2 \\
&\leq \frac{\lambda_0}{4}+\frac{\lambda_0}{4}+\|H^{\infty}\|_2 \\
&\leq \frac{3}{2} \|H^{\infty}\|_2
\end{align*}
and similarly, $\|\tilde{H}(s)\|_2 \leq \frac{3}{2} \|H^{\infty}\|_2$. Thus, $\eta=\mathcal{O}\left(\frac{1}{\|H^{\infty}\|_2+\|\tilde{H}^{\infty}\|_2}\right)$ is sufficient to ensure that $I-\eta(H(s)+\tilde{H}(s))$ is positive definite.

Second, we need to make sure that $\|I(s)\|_2\leq \eta(\lambda_0+\tilde{\lambda}_0)/12$. From (56), $\bar{R}\lesssim \eta(\lambda_0+\tilde{\lambda}_0)$ suffices, i.e.,
\begin{equation}
 m=\Omega\left( \frac{n_1^4n_2^4\log(\frac{n}{\delta}) \log^3(\frac{m}{\delta}) }{(\lambda_0+\tilde{\lambda_0})^4} \right).
\end{equation}

Combining these requirements for $m$, i.e., (58), (59) and the condition in Lemma 2, leads to the desired conclusion.

\end{proof}

\subsection{Proof of Theorem 2}
\begin{proof}
	
From Corollary 1, it remains only to verify that Condition 1 also holds for $s=t+1$. Note that in (52), we have proven that
\[\|w_r(s+1)-w_r(s)\|_2\lesssim \frac{\eta \sqrt{n_1n_2}\|z-G^s(u)\|_2}{\sqrt{m}}\sqrt{ \log \left(\frac{m}{\delta} \right) } \]
holds for $s=0,\cdots, t$ and $r\in[m]$.

Combining this with Corollary 1 yields that
\begin{equation*}
\begin{aligned}
\|w_r(t+1)-w_r(0)\|_2 &\leq \sum\limits_{s=0}^t \|w_r(s+1)-w_r(s)\|_2\\
&\lesssim \sum\limits_{s=0}^t\frac{\eta \sqrt{n_1n_2}\|z-G^s(u)\|_2}{\sqrt{m}}\sqrt{ \log \left(\frac{m}{\delta} \right) }  \\
&\lesssim \sum\limits_{s=0}^t\frac{\eta \sqrt{n_1n_2}}{\sqrt{m}}\sqrt{ \log \left(\frac{m}{\delta}\right)} \left(1-\eta\frac{\lambda_0+\tilde{\lambda}_0}{2} \right)^{s/2}  \|z-G^0(u)\|_2 \\
&\lesssim \frac{\sqrt{n_1n_2}\|z-G^{0}(u)\|_2}{\sqrt{m}(\lambda_0+\tilde{\lambda}_0)}\sqrt{ \log \left(\frac{m}{\delta}\right)}.
\end{aligned}
\end{equation*}

Similarly, in (44), we have proven that 
\[ \|\tilde{w}_{rk}(s+1)-\tilde{w}_{rk}(s)\|_2\lesssim \frac{\eta B\sqrt{n_1n_2}\|z-G^s(u)\|_2}{\sqrt{mp}},\]
which yields that
\[ \|\tilde{w}_{rk}(t+1)-\tilde{w}_{rk}(0)\|_2\lesssim \frac{\sqrt{n_1n_2}\|z-G^{0}(u)\|_2}{\sqrt{mp}(\lambda_0+\tilde{\lambda}_0)}\sqrt{ \log \left(\frac{m}{\delta}\right)}.\]

Moreover, from the triangle inequality, we have
\begin{equation*}
\begin{aligned}
|w_r(t+1)^T y_j| &\leq |(w_r(t+1)-w_r(0))^Ty_j|+|w_r(0)^Ty_j| \\
&\leq \|w_r(t+1)-w_r(0)\|_2+ \sqrt{2\log \left(\frac{2mn_2}{\delta} \right) } \\
&\leq 2\sqrt{\log \left(\frac{mn_2}{\delta} \right)}.
\end{aligned}
\end{equation*}

\end{proof}

\subsection{Proof of Lemma 6}
\begin{proof}
First, for $H^{\infty}=H_1^{\infty}\otimes H_2^{\infty}$, recall that the Kronecker product of two strictly definite matrices is also strictly positive definte, thus it suffices to demonstrate that $H_1^{\infty}$ and $H_2^{\infty} $ are both strictly definite. For $H_1^{\infty}$, similar as that in the proof of Lemma 2, let $\mathcal{H}$ be the Hilbert space of integrable function on $\mathbb{R}^{d+1}$, i.e., $f\in \mathcal{H}$ if $\mathbb{E}_{\tilde{w}\sim \mathcal{N}(\bm{0}, \bm{I}) }[|f(\tilde{w})|^2]<\infty$. Now to prove $H_1^{\infty}$ is strictly positive definite, it is equivalent to show that $\phi(u_1)(\tilde{w}),\cdots, \phi(u_{n_1})(\tilde{w})\in \mathcal{H}$ are linearly independent, where $\phi(u_i)(\tilde{w})=\sigma(\tilde{w}^T u_i)$. It has been proved in \cite{8}, we provide a different proof for completeness and this proof also indicates the strictly positive definiteness of $\tilde{H}^{\infty}$. Suppose that there are $\alpha_1,\cdots, \alpha_{n_1}\in \mathbb{R}$ such that
\[ \alpha_1 \phi(u_1) +\cdots+\alpha_{n_1}\phi(u_{n_1})=0 \ in \ \mathcal{H},\]
which implies that 
\[ \alpha_1 \phi(u_1)(\tilde{w}) +\cdots+\alpha_{n_1}\phi(u_{n_1})(\tilde{w})=0 \] 
holds for all $\tilde{w}\in \mathbb{R}^{d+1}$ due to the continuity of $\phi(u_1)(\cdot)$.

Let $\tilde{D}_i=\{\tilde{w}\in \mathbb{R}^{d+1}: \tilde{w}^T u_i=0\}$ for $i\in [n_1]$, then Lemma A.1 in \cite{7} implies that when no two samples in $\{u_i\}_{i=1}^{n_1}$ are parallel, $D_i \not\subset \cup_{j\neq i} D_j$ for any $i\in [n_1]$. Thus, we can choose $\tilde{w}_0 \in D_i\backslash  \cup_{j\neq i} D_j$. Since $\cup_{j\neq i} D_j$ is a closed set, there is positive constant $r_0$ such that $B_{r_0}(\tilde{w}_0)\cap( \cup_{j\neq i} D_j)= \emptyset$. This fact implies that $\phi(u_j)(\cdot)$ is differentiable in $B_{r_0}(\tilde{w}_0)$ for each $j\neq i$. Thus $\alpha_i \phi(u_i)(\cdot)$ is also differentiable in $B_{r_0}(\tilde{w}_0)$. However $\phi(u_i)(\tilde{w})=\sigma(\tilde{w}^T u_i)$ is not differentiable in $B_{r_0}(\tilde{w}_0)$. Thus we can deduce that $\alpha_i=0$. Similarly, we have $\alpha_j= 0$ for all $j\in [n_1]$, which implies that $H_1^{\infty}$ is strictly positive definite. For $H_2^{\infty}$, it can be seen as a Gram matrix of PINN, Lemma 3.2 in \cite{10} implies that $H_2^{\infty}$ is strictly positive definite. 

Second, for $\tilde{H}^{\infty}$, recall that $\tilde{H}^{\infty}=\tilde{H}_1^{\infty} \otimes \tilde{H}_2^{\infty}$. Note that the $(i,j)$-th entry of $\tilde{H}_1^{\infty}$ is $\mathbb{E}[u_i^T u_j I\{\tilde{w}^T u_i\geq 0,\tilde{w}^T u_i\geq 0\}] $. Thus, Theorem 3.1 in \cite{17} implies that $\tilde{H}_1^{\infty}$ is strictly positive definite. For $H_2^{\infty}$, let
\[\phi_1(w)=\mathcal{L}(\sigma_3(w^T y_1)),\cdots, \phi_{n_2}(w)=\mathcal{L}(\sigma_3(w^T y_{n_2})), \psi_1(w)=\sigma_3(w^T\tilde{y}_1),\cdots,  \psi_{n_3}(w)=\sigma_3(w^T\tilde{y}_{n_3})\]
and $H$ be the Hilbert space of integrable $(d+1)$-dimensional vector fields on $\mathbb{R}^{d+1}$. Suppose that there are $\alpha_1, \cdots, \alpha_{n_2},\beta_1,\cdots,\beta_{n_3}$ such that 
\[ \alpha_1 \phi_1+\cdots+\alpha_{n_2}\phi_{n_2}+\beta_1 \psi_1+\cdots+\beta \psi_{n_3}=0 \ in \ \mathcal{H},\]
which yields that 
\[ \alpha_1 \phi_1+\cdots+\alpha_{n_2}\phi_{n_2}+\beta_1 \psi_1+\cdots+\beta \psi_{n_3}=0 \]
holds for all $w\in \mathbb{R}^{d+1}$.

Let $D_i=\{w\in \mathbb{R}^{d+1}: w^T y_i =0 \}$ for $i \in [n_2]$ and $\bar{D}_i =  \{w\in \mathbb{R}^{d+1}: w^T \tilde{y}_i =0 \}$ for $i \in [n_3]$. Thus $D_i \not\subset (\cup_{j\neq i} D_j)\cup (\cup_{j} \bar{D}_j )$ for any $i\in [n_2]$. Similarly, we can choose $w_0\in D_i$ and $r_0>0$ such that $B_{r_0}(w_0) \cap \left((\cup_{j\neq i} D_j)\cup (\cup_{j} \bar{D}_j )\right)=\emptyset$. Note that $\phi_j$ $(j\neq i)$ and $\psi_j$ are differentiable in $B_{r_0}(w_0)$, thus $\alpha_i  \phi_i $ is also differentiable in $B_{r_0}(w_0)$, implying that $\alpha_i=0$. Therefore, $\alpha_i=0$ for all $i\in [n_2]$. Moreover, similar to the proof of the strictly positive definiteness of $H_1^{\infty}$, we can also deduce that $\beta_i=0$ for all $i\in [n_3]$. Finally, $\tilde{H}$ is strictly positive definite.

\end{proof}

\subsection{Proof of Lemma 7}

\begin{proof}
First, for $H(0)-H^{\infty}$, we consider its $(i,j)$-th block, whose entry has the following form
\[ \frac{1}{m}\sum\limits_{r=1}^m X_r Y_r - \mathbb{E}[X_1 Y_1],\]
where
\[  X_r = \left[\frac{1}{\sqrt{p}} \sum\limits_{k=1}^p \tilde{a}_{rk}(0)\sigma(\tilde{w}_{rk}(0)^T u_i)  \right]\left[\frac{1}{\sqrt{p}} \sum\limits_{k=1}^p \tilde{a}_{rk}(0)\sigma(\tilde{w}_{rk}(0)^T u_j)  \right],\]
for $j_1\in [n_1], j_2\in [n_1]$,
\[ Y_r=\frac{1}{n_2}\left\langle \frac{\partial \mathcal{L}(\sigma_3(w_r(0)^T y_{j_1})) }{\partial w_r}, \frac{\partial \mathcal{L}(\sigma_3(w_r(0)^T y_{j_2})) }{\partial w_r}  \right\rangle,\]
for $j_1\in [n_1], n_2+j_2\in [n_2, n_2+n_3]$,
\[ Y_r=\frac{1}{\sqrt{n_2n_3}\ }\left\langle \frac{\partial \mathcal{L}(\sigma_3(w_r(0)^T y_{j_1})) }{\partial w_r}, \frac{\partial \sigma_3(w_r(0)^T \tilde{y}_{j_2}) }{\partial w_r}  \right\rangle,\]
for $n_2+j_1\in [n_2,n_2+n_3], n_2+j_2\in [n_2, n_2+n_3]$,
\[ Y_r=\frac{1}{n_3}\left\langle \frac{\partial \sigma_3(w_r(0)^T \tilde{y}_{j_1}) }{\partial w_r}, \frac{\partial \sigma_3(w_r(0)^T \tilde{y}_{j_2}) }{\partial w_r}  \right\rangle.\]

To use the concentration inequality, we need to clarify the order of the sub-Weil random variable $X_rY_r-\mathbb{E}[X_1Y_1]$. Note that Lemma 18 implies that
\[ \|X_r\|_{\psi_1} \leq \left\|\frac{1}{\sqrt{p}} \sum\limits_{k=1}^p \tilde{a}_{rk}(0)\sigma(\tilde{w}_{rk}(0)^T u_i) \right\|_{\psi_2} \left\|\frac{1}{\sqrt{p}} \sum\limits_{k=1}^p \tilde{a}_{rk}(0)\sigma(\tilde{w}_{rk}(0)^T u_j) \right\|_{\psi_2}=\mathcal{O}(1).\]

On the other hand, from 
\[\mathcal{L}(\sigma_3(w_r^T y))=w_{r0}\sigma_2(w_r^T y)+\|w_{r1}\|_2^2\sigma(w_r^T y)+\sigma_3(w_r^T y),\]
we have
\[ \frac{\partial \mathcal{L}(\sigma_3(w_r^T y)) }{\partial w_r}=\begin{pmatrix} 1\\0_d \end{pmatrix}\sigma_2(w_r^T y)+ w_{r0}y \sigma(w_r^Ty)+2\begin{pmatrix} 0\\w_{r1} \end{pmatrix} \sigma(w_r^T y)+\|w_{r1}\|_2 y I\{w_r^T y\geq 0\}+y\sigma_2(w_r^T y)\]
and
\[ \frac{\partial \sigma_3(w_r^T y) }{\partial w_r}= y\sigma_2(w_r^T y). \]

Therefore, we can deduce that
\[\left\|\frac{\partial \mathcal{L}(\sigma_3(w_r^T y)) }{\partial w_r} \right\|_2\lesssim |w_r^T y|^2+\|w_r\|_2|w_r^T y|+|w_r^T y|^2 \lesssim \|w_r\|_2|w_r^T y|\]
and
\[ \left\|\frac{\partial \sigma_3(w_r^T y) }{\partial w_r} \right\|_2\lesssim |w_r^T y|^2\lesssim \|w_r\|_2 |w_r^Ty|.\]

Note that $\|w_r^T y\|_{\psi_2}=\mathcal{O}(1)$ for $\|y\|_2=\mathcal{O}(1)$, thus $\|\sigma_2(w_r^T y)\|_{\psi_{1}}=\mathcal{O}(1)$, $\|\sigma_3(w_r^T y)\|_{\psi_{\frac{2}{3} }}=\mathcal{O}(1)$. Thus Lemma 20 implies that 
\[ \|\|w_r\|_2^2|w_r^T y_{j_1}||w_r^T y_{j_2}|\|_{\psi_{\frac{1}{2} } }\leq \|\|w_r\|_2^2\|_{\psi_1} \||w_r^T y_{j_1}||w_r^T y_{j_2}|\|_{\psi_1}=\mathcal{O}(d). \]

On the other hand, Lemma 21 implies that 
\[\| X_rY_r-\mathbb{E}[X_1Y_1]\|_{\psi_{\alpha}} \lesssim \| X_rY_r\|_{\psi_{\alpha}}+\| \mathbb{E}[X_1Y_1]\|_{\psi_{\alpha}} \lesssim \| X_rY_r\|_{\psi_{\alpha}}+\mathbb{E}[|X_1|]\mathbb{E}[|Y_1|].\]

Note that $\mathbb{E}[|X_1|]\leq |X_1|_{\psi_1}=\mathcal{O}(1)$ and $|Y_1|_{\psi_{\frac{1}{2} }}\lesssim \|\|w_r\|_2^2|w_r^T y_{j_1}||w_r^T y_{j_2}|\|_{\psi_{\frac{1}{2} } }=\mathcal{O}(d)$. From the Taylor expansion of the function $e^x$, we have that for any $C>0$,
\[\mathbb{E}[e^{(\frac{|Y_1|}{C} )^{\frac{1}{2}}}-1]\geq \mathbb{E}\left[\frac{1}{2!}\frac{|Y|}{C}\right],\]
which implies that $\mathbb{E}[|Y_1|]=\mathcal{O}(d)$. Therefore, $\|X_rY_r-\mathbb{E}[X_1Y_1]\|_{\psi_{\frac{1}{2} }}=\mathcal{O}(d)$. Finally, applying Lemma 17 leads to that with probability at least $1-\delta$, 
\[\left|\frac{1}{m}\sum\limits_{r=1}^m X_rY_r-\mathbb{E}[X_1Y_1] \right|\lesssim \frac{d}{n_2\sqrt{m}} \sqrt{\log(\frac{1}{\delta}) }+ \frac{d}{n_2m} \log^2(\frac{1}{\delta}) .\]

Taking a union bound yields that with probability at least $1-\delta$,
\begin{align*}
\|H(0)-H^{\infty}\|_F &\lesssim \frac{dn_1}{\sqrt{m}} \log(\frac{n_1(n_2+n_3)}{\delta}).
\end{align*}

First, for $\tilde{H}(0)-\tilde{H}^{\infty}$, we consider its $(i,j)$-th block, whose entry has the following form
\[ \frac{1}{m}\sum\limits_{r=1}^m X_r Y_r - \mathbb{E}[X_1 Y_1],\]
where
\[  X_r =\frac{1}{p} \sum\limits_{k=1}^p u_i^T u_jI\{\tilde{w}_{rk}(0)^T u_i\geq 0,\tilde{w}_{rk}(0)^T u_j\geq 0\}  ,\]
for $j_1\in [n_1], j_2\in [n_1]$,
\[ Y_r=\frac{1}{n_2} \mathcal{L}(\sigma_3(w_r(0)^T y_{j_1}))  \mathcal{L}(\sigma_3(w_r(0)^T y_{j_2}))\]
for $j_1\in [n_1], n_2+j_2\in [n_2, n_2+n_3]$,
\[ Y_r=\frac{1}{\sqrt{n_2n_3}\ } \mathcal{L}(\sigma_3(w_r(0)^T y_{j_1}))  \sigma_3(w_r(0)^T \tilde{y}_{j_2}),\]
for $n_2+j_1\in [n_2,n_2+n_3], n_2+j_2\in [n_2, n_2+n_3]$,
\[ Y_r=\frac{1}{n_3}\sigma_3(w_r(0)^T \tilde{y}_{j_1})\sigma_3(w_r(0)^T \tilde{y}_{j_2}).\]

From Lemma 21, we have
\[ \left\| \frac{1}{p} \sum\limits_{k=1}^p u_i^T u_jI\{\tilde{w}_{rk}(0)^T u_i\geq 0,\tilde{w}_{rk}(0)^T u_j\geq 0\}\right\|_{\psi_2} =\mathcal{O}(1).\]

Note that 
\[ |\mathcal{L}(\sigma_3(w_r^T y)) |\lesssim \|w_r\|_2 |w_r^Ty|^2+ \|w_r\|_2^2 |w_r^Ty|+|w_r^T y|^3\lesssim \|w_r\|_2^2 |w_r^Ty| \]
and
\[ \ |\sigma_3(w_r^Ty)|\leq |w_r^T y|^3\lesssim \|w_r\|_2^2 |w_r^Ty|.\]

From Lemma 21, we can deduce that 
\[\| \|w_r\|_2^4 \|_{\psi_{\frac{1}{2} }}\leq \| \|w_r\|_2^2 \|_{\psi_{1}}^2=\mathcal{O}(d^2)\]
and $\| |w_r^Ty|^2\|_{\psi_1}=\mathcal{O}(1)$, thus
\[ \| \|w_r\|_2^4|w_r^Ty|^2\|_{\psi_{\frac{1}{3}} }=\mathcal{O}(d^2) .\]

Therefore, with probability at least $1-\delta$, 
\begin{align*}
\|\tilde{H}(0)-\tilde{H}^{\infty}\|_F &\lesssim \frac{d^2n_1}{\sqrt{m}} \log(\frac{n_1(n_2+n_3)}{\delta}).
\end{align*}
\end{proof}

\subsection{Proof of Lemma 8}
\begin{proof}
For $H-H(0)$, from the form of $(j_1,j_2)$-th entry of the $(i,j)$-th block, we focus on the form $a_rb_r-a_r(0)b_r(0)$, where
\[ a_r=\left[\frac{1}{\sqrt{p}} \sum\limits_{k=1}^p \tilde{a}_{rk}\sigma(\tilde{w}_{rk}^T u_i) \right] \left[\frac{1}{\sqrt{p}} \sum\limits_{k=1}^p \tilde{a}_{rk}\sigma(\tilde{w}_{rk}^T u_j)\right],\]
\[b_r=\left\langle \frac{\partial \mathcal{L}(\sigma_3(w_r^T y_{j_1})) }{\partial w_r}, \frac{\partial \mathcal{L}(\sigma_3(w_r^T y_{j_2})) }{\partial w_r}  \right\rangle or \left\langle \frac{\partial \mathcal{L}(\sigma_3(w_r^T y_{j_1})) }{\partial w_r}, \frac{\partial \sigma_3(w_r^T \tilde{y}_{j_2}) }{\partial w_r}  \right\rangle or \left\langle \frac{\partial \sigma_3(w_r^T \tilde{y}_{j_1}) }{\partial w_r}, \frac{\partial \sigma_3(w_r^T \tilde{y}_{j_2}) }{\partial w_r}  \right\rangle\]
and the notation $a_r(0),b_r(0)$ means replacing $w_r,\tilde{w}_{rk}$ in the definitions of $a_r$ and $b_r$ with $w_r(0)$ and $\tilde{w}_{rk}(0)$, respectively.

For $a_r-a_r(0)$, (25) implies that
\begin{equation}
|a_r-a_r(0)|\lesssim p\tilde{R}^2+\sqrt{p}\tilde{R}\sqrt{\log\left( \frac{mn_1}{\delta}\right)}.
\end{equation} 

For $b_r-b_r(0)$, when $b_r= \left\langle \frac{\partial \mathcal{L}(\sigma_3(w_r^T y_{j_1})) }{\partial w_r}, \frac{\partial \mathcal{L}(\sigma_3(w_r^T y_{j_2})) }{\partial w_r}  \right\rangle$, we have
\begin{equation}
\begin{aligned}
&\left|\left\langle \frac{\partial \mathcal{L}(\sigma_3(w_r^T y_{j_1})) }{\partial w_r}, \frac{\partial \mathcal{L}(\sigma_3(w_r^T y_{j_2})) }{\partial w_r}  \right\rangle-\left\langle \frac{\partial \mathcal{L}(\sigma_3(w_r(0)^T y_{j_1})) }{\partial w_r}, \frac{\partial \mathcal{L}(\sigma_3(w_r(0)^T y_{j_2})) }{\partial w_r}  \right\rangle \right|\\
&\leq \left\| \frac{\partial \mathcal{L}(\sigma_3(w_r^T y_{j_1})) }{\partial w_r}- \frac{\partial \mathcal{L}(\sigma_3(w_r(0)^T y_{j_1})) }{\partial w_r}  \right\|_2 \left\|\frac{\partial \mathcal{L}(\sigma_3(w_r(0)^T y_{j_2})) }{\partial w_r} \right\|_2 \\
&\quad + \left\| \frac{\partial \mathcal{L}(\sigma_3(w_r^T y_{j_2})) }{\partial w_r}- \frac{\partial \mathcal{L}(\sigma_3(w_r(0)^T y_{j_2})) }{\partial w_r}  \right\|_2 \left\|\frac{\partial \mathcal{L}(\sigma_3(w_r(0)^T y_{j_1})) }{\partial w_r} \right\|_2 \\
&\quad +  \left\| \frac{\partial \mathcal{L}(\sigma_3(w_r^T y_{j_1})) }{\partial w_r}- \frac{\partial \mathcal{L}(\sigma_3(w_r(0)^T y_{j_1})) }{\partial w_r}  \right\|_2\left\| \frac{\partial \mathcal{L}(\sigma_3(w_r^T y_{j_2})) }{\partial w_r}- \frac{\partial \mathcal{L}(\sigma_3(w_r(0)^T y_{j_2})) }{\partial w_r}  \right\|_2 .
\end{aligned}
\end{equation}

Note with probability at least $1-\delta$, we have $\|w_r(0)\|_2\lesssim B_1$, $|w_r(0)^T y_j|\lesssim B_2$, $|w_r(0)^T \tilde{y}_{j_1}|\lesssim B_2$ holds for all $r\in [m]$, $j\in [n_2]$, $j_1\in [n_3]$ where
\[ B_1=\sqrt{d\log\left( \frac{m}{\delta}\right)},\ B_2=\sqrt{\log\left( \frac{m(n_2+n_3)}{\delta}\right)}.\]

Under these events, we can deduce that 
\[\left\|\frac{\partial \mathcal{L}(\sigma_3(w_r(0)^T y_{j_1})) }{\partial w_r} \right\|_2, \left\|\frac{\partial \mathcal{L}(\sigma_3(w_r(0)^T y_{j_2})) }{\partial w_r} \right\|_2 \lesssim B_1B_2\]
and
\begin{align*}
\left\| \frac{\partial \mathcal{L}(\sigma_3(w_r^T y_{j_1})) }{\partial w_r}- \frac{\partial \mathcal{L}(\sigma_3(w_r(0)^T y_{j_1})) }{\partial w_r}  \right\|_2&\lesssim (B_1+B_2)R+B_1|I\{w_r(0)^T y_{j_1}\geq 0\}-I\{ w_r^T y_{j_1}\geq 0 \}| \\
&\lesssim (B_1+B_2)R+B_1I\{A_{r,j_1}\}.
\end{align*}
Thus, for $b_r-b_r(0)$, we have
\begin{equation}
\begin{aligned}
|b_r-b_r(0)|&\lesssim B_1B_2(B_1+B_2)R+B_1^2B_2 [I\{A_{r,j_1}\}+I\{A_{r,j_2}\}].
\end{aligned}
\end{equation}

From Bernstein's inequality, we have with probability at least $1-n_2\exp(-mR)$,
\[\frac{1}{m}\sum\limits_{r=1}^m I\{A_{r,j}\}\lesssim R\]
holds for all $j\in [n_2]$.

Thus, summing $r$ yields that 
\begin{equation}
\begin{aligned}
\frac{1}{m}\sum\limits_{r=1}^m |b_r-b_r(0)|&\lesssim B_1B_2(B_1+B_2)R+B_1^2B_2R \\
&\lesssim B_1^2B_2R.
\end{aligned}
\end{equation}

When $b_r=\left\langle \frac{\partial \mathcal{L}(\sigma_3(w_r^T y_{j_1})) }{\partial w_r}, \frac{\partial \sigma_3(w_r^T \tilde{y}_{j_2}) }{\partial w_r}  \right\rangle or \left\langle \frac{\partial \sigma_3(w_r^T \tilde{y}_{j_1}) }{\partial w_r}, \frac{\partial \sigma_3(w_r^T \tilde{y}_{j_2}) }{\partial w_r}  \right\rangle$, we can obtain the same estimation, since
\[\left\|\frac{\partial \sigma_3(w_r(0)^T y_{j_1})}{\partial w_r}\right\|_2\lesssim B_2^2\lesssim B_1B_2\]
and 
\[\left\|\frac{\partial \sigma_3(w_r^T y_{j_1})}{\partial w_r}-\frac{\partial \sigma_3(w_r(0)^T y_{j_1})}{\partial w_r}\right\|_2\lesssim B_2 R.\]

Note that we can decompose $a_rb_r -a_r(0)b_r(0)$ as follows
\[a_r(0)[b_r-b_r(0)]+b_r(0)[a_r-a_r(0)]+[a_r-a_r(0)][b_r-b_r(0)]. \] 

Therefore, we have
\begin{equation}
\begin{aligned}
&\frac{1}{m}\sum\limits_{r=1}^m |a_rb_r -a_r(0)b_r(0)| \\
&= \frac{1}{m}\sum\limits_{r=1}^m \left|a_r(0)[b_r-b_r(0)]+b_r(0)[a_r-a_r(0)]+[a_r-a_r(0)][b_r-b_r(0)] \right| \\
&\leq  \frac{1}{m}\sum\limits_{r=1}^m |a_r(0)||b_r-b_r(0)|+\frac{1}{m}\sum\limits_{r=1}^m|b_r(0)||a_r-a_r(0)|+\frac{1}{m}\sum\limits_{r=1}^m |a_r-a_r(0)||b_r-b_r(0)|\\
&\lesssim B_1^2B_2R\log\left( \frac{mn_1}{\delta}\right)+B_1^2B_2^2\left(p\tilde{R}^2+\sqrt{p}\tilde{R}\sqrt{\log\left( \frac{mn_1}{\delta}\right)}\right).
\end{aligned}
\end{equation}

Summing $i,j,i_1,j_1$ yields that
\begin{align*}
\|H-H(0)\|_F&\lesssim n_1B_1^2B_2R\log\left( \frac{mn_1}{\delta}\right)+n_1B_1^2B_2^2\left(p\tilde{R}^2+\sqrt{p}\tilde{R}\sqrt{\log\left( \frac{mn_1}{\delta}\right)}\right).
\end{align*}

For $\tilde{H}-\tilde{H}(0)$, from the form of $(j_1,j_2)$-th entry of the $(i,j)$-th block, we focus on the form $a_rb_r-a_r(0)b_r(0)$, where
\[a_r=\frac{1}{p}\sum\limits_{k=1}^p u_i^Tu_j I\{\tilde{w}_{rk}^T u_i\geq 0, \tilde{w}_{rk}^Tu_j\geq 0\},\]
\[ b_r=\mathcal{L}(\sigma_3(w_r^T y_{j_1}))\mathcal{L}(\sigma_3(w_r^T y_{j_2})) \ or \ \mathcal{L}(\sigma_3(w_r^T y_{j_1}))\sigma_3(w_r^T \tilde{y}_{j_2})\ or \ \sigma_3(w_r^T \tilde{y}_{j_1})\sigma_3(w_r^T \tilde{y}_{j_2}).\]

Similarly, we can deduce that 
\begin{equation}
\begin{aligned}
&\frac{1}{m}\sum\limits_{r=1}^m |a_r-a_r(0)| \\
&\lesssim \frac{1}{m}\frac{1}{p} \sum\limits_{r=1}^m \sum\limits_{k=1}^p |I\{\tilde{w}_{rk}^T u_i\geq 0, \tilde{w}_{rk}^Tu_j\geq 0\}-I\{\tilde{w}_{rk}(0)^T u_i\geq 0, \tilde{w}_{rk}(0)^Tu_j\geq 0\}| \\
&\lesssim \frac{1}{m}\frac{1}{p} \sum\limits_{r=1}^m \sum\limits_{k=1}^p I\{\tilde{A}_{r,k}^{i} \}+I\{\tilde{A}_{r,k}^{j} \} \\
&\lesssim \tilde{R},
\end{aligned}
\end{equation}
where the last inequality holds with probability at least $1-n_1\exp(-mp\tilde{R})$ due to the use of Bernstein inequality.

For $b_r-b_r(0)$, note that 
\[ |\mathcal{L}(\sigma_3(w_r(0)^T y_j))|\lesssim B_1^2B_2, \ |\sigma_3(w_r(0)^T \tilde{y}_j)|\lesssim B_2^3\lesssim B_1^2B_2\]
and
\[\|\mathcal{L}(\sigma_3(w_r^Ty_j))-\mathcal{L}(\sigma_3(w_r(0)^Ty_j))\|_2\lesssim B_1B_2R, \ |\sigma_3(w_r^T\tilde{y}_j)-\sigma_3(w_r(0)^T\tilde{y}_j)|\lesssim B_2^2R\lesssim B_1B_2R.\]

Thus, similar to (66), we have
\begin{equation}
|b_r-b_r(0)|\lesssim B_1^3B_2^2 R.	
\end{equation}	 

Therefore, similar to (69), we have
\begin{equation}
\begin{aligned}
&\frac{1}{m}\sum\limits_{r=1}^m |a_rb_r -a_r(0)b_r(0)| \\
&= \frac{1}{m}\sum\limits_{r=1}^m \left|a_r(0)[b_r-b_r(0)]+b_r(0)[a_r-a_r(0)]+[a_r-a_r(0)][b_r-b_r(0)] \right| \\
&\lesssim B_1^2B_2 \tilde{R}+ B_1^3 B_2^2 R.
\end{aligned}
\end{equation}

Summing $i,j,i_1,j_1$ yields that
\begin{align*}
\|\tilde{H}-\tilde{H}(0)\|_F&\lesssim n_1B_1^2B_2 \tilde{R}+n_1 B_1^3 B_2^2 R.
\end{align*}

\end{proof}

\subsection{Proof of Lemma 8}
\begin{proof}
Note that
\begin{equation}
\begin{aligned}
&s^{t+1}(u_i)(y_j)-s^{t}(u_i)(y_j) \\
&= s^{t+1}(u_i)(y_j)-s^{t}(u_i)(y_j)-\left\langle \frac{\partial s^{t}(u_i)(y_j)}{\partial w}, w(t+1)-w(t) \right\rangle-\left\langle \frac{\partial s^{t}(u_i)(y_j)}{\partial \tilde{w}}, \tilde{w}(t+1)-\tilde{w}(t) \right\rangle \\
&\quad + \left\langle \frac{\partial s^{t}(u_i)(y_j)}{\partial w}, w(t+1)-w(t) \right\rangle+\left\langle \frac{\partial s^{t}(u_i)(y_j)}{\partial \tilde{w}}, \tilde{w}(t+1)-\tilde{w}(t) \right\rangle.
\end{aligned}
\end{equation}
From the updating formula of gradient, we have
\begin{align*}
	&\left\langle \frac{\partial s^{t}(u_i)(y_j)}{\partial w}, w(t+1)-w(t) \right\rangle \\
	&= \sum\limits_{r=1}^m \left\langle \frac{\partial s^{t}(u_i)(y_j)}{\partial w_r}, w_r(t+1)-w_r(t) \right\rangle \\
	&= -\eta \sum\limits_{r=1}^m \sum\limits_{i_1=1}^{n_1} \sum\limits_{j_1=1}^{n_2} \left\langle \frac{\partial s^{t}(u_i)(y_j)}{\partial w_r}, \frac{\partial s^{t}(u_{i_1})(y_{j_1})}{\partial w_r} \right\rangle s^{t}(u_{i_1})(y_{j_1}) \\
	&\quad -\eta \sum\limits_{r=1}^m \sum\limits_{i_1=1}^{n_1} \sum\limits_{j_2=1}^{n_3} \left\langle \frac{\partial s^{t}(u_i)(y_j)}{\partial w_r}, \frac{\partial h^{t}(u_{i_1})(\tilde{y}_{j_2})}{\partial w_r} \right\rangle h^{t}(u_{i_1})(\tilde{y}_{j_2}).
\end{align*}

Similarly, we can obtain that 
\begin{align*}
&\left\langle \frac{\partial s^{t}(u_i)(y_j)}{\partial \tilde{w} }, \tilde{w}(t+1)-\tilde{w}(t) \right\rangle \\
&= \sum\limits_{r=1}^m \sum\limits_{k=1}^p\left\langle \frac{\partial s^{t}(u_i)(y_j)}{\partial \tilde{w}_{rk}}, \tilde{w}_{rk}(t+1)-\tilde{w}_{rk}(t) \right\rangle \\
&=-\eta \sum\limits_{r=1}^m \sum\limits_{k=1}^p  \sum\limits_{i_1=1}^{n_1} \sum\limits_{j_1=1}^{n_2} \left\langle \frac{\partial s^{t}(u_i)(y_j)}{\partial \tilde{w}_{rk}}, \frac{\partial s^{t}(u_{i_1})(y_{j_1})}{\partial \tilde{w}_{rk}} \right\rangle s^{t}(u_{i_1})(y_{j_1}) \\
&\quad -\eta \sum\limits_{r=1}^m \sum\limits_{k=1}^p  \sum\limits_{i_1=1}^{n_1} \sum\limits_{j_2=1}^{n_3} \left\langle \frac{\partial s^{t}(u_i)(y_j)}{\partial \tilde{w}_{rk}}, \frac{\partial h^{t}(u_{i_1})(\tilde{y}_{j_2})}{\partial \tilde{w}_{rk}} \right\rangle h^{t}(u_{i_1})(\tilde{y}_{j_2}).
\end{align*}

For $h^{t+1}(u_i)(\tilde{y}_j)-h^{t}(u_i)(\tilde{y}_j)$, we can derive similar result, which is omitted for simplicity. Similar to the derivation in the section on neural operators, we have
\begin{equation}
G^{t+1}(u)-G^{t}(u)=-\eta(H(t)+\tilde{H}(t))G^{t}(u)+ I(t),	
\end{equation}
where $I(t)\in \mathbb{R}^{n_1(n_2+n_3)}$, $I(t)$ can be divided into $n_1$ blocks, where each block is an $(n_2+n_3)$dimensional vector. The $j_1$-th ($j_1\in [n_2]$) component of $i$-th block is 
\[ s^{t+1}(u_i)(y_{j_1})-s^{t}(u_i)(y_{j_1})-\left\langle \frac{\partial s^{t}(u_i)(y_{j_1})}{\partial w}, w(t+1)-w(t) \right\rangle-\left\langle \frac{\partial s^{t}(u_i)(y_{j_1})}{\partial \tilde{w}}, \tilde{w}(t+1)-\tilde{w}(t) \right\rangle,\]  
The $n_2+j_2$-th ($j_1\in [n_3]$) component of $i$-th block is
\[ h^{t+1}(u_i)(\tilde{y}_{j_2})-h^{t}(u_i)(\tilde{y}_{j_2})-\left\langle \frac{\partial h^{t}(u_i)(\tilde{y}_{j_2})}{\partial w}, w(t+1)-w(t) \right\rangle-\left\langle \frac{\partial h^{t}(u_i)(\tilde{y}_{j_2})}{\partial \tilde{w}}, \tilde{w}(t+1)-\tilde{w}(t) \right\rangle.\]

Finally, applying a simple algebraic transformation to (76), we have  
\[ G^{t+1}(u)=[I-\eta(H(t)+\tilde{H}(t))]G^{t}(u)+I(t).\]
\end{proof}

\subsection{Proof of Corollary 2}
\begin{proof}
Let $B_1=2\sqrt{d\log(m/\delta)}, B_2=2\sqrt{\log(m(n_2+n_3)/\delta)}$, we first estimate $\|\tilde{w}_{rk}(t+1)-\tilde{w}_{rk}(t)\|_2$. The gradient updating rule yields that 
\begin{align*}
\tilde{w}_{rk}(t+1)-\tilde{w}_{rk}(t)&= -\eta \frac{\partial L(W(t),\tilde{W}(t))}{\partial \tilde{w}_{rk} } \\
&=-\eta \sum\limits_{i=1}^{n_1}\sum\limits_{j_1=1}^{n_2} s^{t}(u_i)(y_{j_1}) \frac{\partial s^{t}(u_i)(y_{j_1})}{\partial \tilde{w}_{rk}}- -\eta \sum\limits_{i=1}^{n_1}\sum\limits_{j_2=1}^{n_3} h^{t}(u_i)(\tilde{y}_{j_2}) \frac{\partial h^{t}(u_i)(\tilde{y}_{j_2})}{\partial \tilde{w}_{rk}}.
\end{align*}

For the gradient term, we have 
\begin{align*}
	&\left\| \frac{\partial s^{t}(u_i)(y_{j_1})}{\partial \tilde{w}_{rk}}\right\|_2 \\
	&= \left\| \frac{1}{\sqrt{n_2}}\frac{1}{\sqrt{m}}\frac{1}{\sqrt{p}}\tilde{a}_{rk} u_i I\{\tilde{w}_{rk}(t)^T u_i \geq 0\} \mathcal{L}(\sigma_3(w_r(t)^T y_{j_1})) \right\|_2 \\
	&\lesssim \frac{1}{\sqrt{n_2}}\frac{B_1^2B_2}{\sqrt{mp}}
\end{align*}
and
\begin{align*}
	&\left\| \frac{\partial h^{t}(u_i)(\tilde{y}_{j_2})}{\partial \tilde{w}_{rk}}\right\|_2 \\
	&= \left\|\frac{1}{\sqrt{n_3}} \frac{1}{\sqrt{m}}\frac{1}{\sqrt{p}}\tilde{a}_{rk} u_i I\{\tilde{w}_{rk}(t)^T u_i \geq 0\} \sigma_3(w_r(t)^T\tilde{y}_{j_2}) \right\|_2 \\
	&\lesssim \frac{1}{\sqrt{n_3}}\frac{B_2^3}{\sqrt{mp}}.
\end{align*}

Theorefore, we obtain that
\begin{equation}
\begin{aligned}
\|\tilde{w}_{rk}(t+1)-\tilde{w}_{rk}(t)\|_2&\lesssim \frac{\eta\sqrt{n_1}B_1^2B_2}{\sqrt{mp}} \sqrt{\sum\limits_{i=1}^{n_1} \|s(u_i)\|_2^2}+\frac{\eta\sqrt{n_1}B_2^3}{\sqrt{mp}} \sqrt{\sum\limits_{i=1}^{n_1} \|h(u_i)\|_2^2} \\
&\lesssim \frac{\eta\sqrt{n_1}B_1^2B_2}{\sqrt{mp}} \|G^{t}(u)\|_2,
\end{aligned}
\end{equation}
where the first inequality follows from Cauchy's inequality.

Summing $t$ from $0$ to $T$ yields that
\begin{equation}
\begin{aligned}
\|\tilde{w}_{rk}(T)-\tilde{w}_{rk}(0)\|_2 &\leq \sum\limits_{t=0}^{T} \|\tilde{w}_{rk}(t+1)-\tilde{w}_{rk}(t)\|_2 \\
&\lesssim \sum\limits_{t=1}^{T}\frac{\eta\sqrt{n_1}B_1^2B_2}{\sqrt{mp}} \left(1-\frac{\eta(\lambda_0+\tilde{\lambda}_0)}{2} \right)^{t/2} \|G^{0}(u)\|_2\\
&\lesssim   \frac{\sqrt{n_1}B_1^2B_2\|G^{0}(u)\|_2}{\sqrt{mp}(\lambda_0+\tilde{\lambda}_0)},
\end{aligned}
\end{equation}
where the second inequality follows from the induction hypothesis.

Then, we estimate $\|w_r(t+1)-w_r(t)\|_2$, the gradient descent updating rule yields that

\begin{equation}
\begin{aligned}
\|w_r(t+1)-w_r(t)\|_2 &= \left\|-\eta \frac{\partial L(W(t), \tilde{W}(t))}{\partial w_r} \right\|_2 \\
&= \eta \left\| \sum\limits_{i=1}^{n_1}\sum\limits_{j_1=1}^{n_2} s^{t}(u_i)(y_{j_1}) \frac{\partial s^{t}(u_i)(y_{j_1})}{\partial w_r}+ \sum\limits_{i=1}^{n_1}\sum\limits_{j_2=1}^{n_3} h^{t}(u_i)(\tilde{y}_{j_2}) \frac{\partial h^{t}(u_i)(\tilde{y}_{j_2})}{\partial w_r}\right\|_2.
\end{aligned}
\end{equation}

Recall that
\begin{align*}
&\left\| \frac{\partial s^{t}(u_i)(y_{j_1})}{\partial w_r}\right\|_2 \\
&= \left\| \frac{1}{\sqrt{n_2}}\frac{1}{\sqrt{m}}\left[\frac{1}{\sqrt{p}}\sum\limits_{k=1}^p\tilde{a}_{rk} \sigma(\tilde{w}_{rk}(t)^T u_i)\right] \frac{\partial \mathcal{L}(\sigma_3(w_r(t)^T y_{j_1}))}{\partial w_r} \right\|_2 .
\end{align*}

Note that 
\[ \left|\left[\frac{1}{\sqrt{p}}\sum\limits_{k=1}^p\tilde{a}_{rk} \sigma(\tilde{w}_{rk}(t)^T u_i)\right]-\left[\frac{1}{\sqrt{p}}\sum\limits_{k=1}^p\tilde{a}_{rk} \sigma(\tilde{w}_{rk}(0)^T u_i)\right] \right|\lesssim \sqrt{p}\tilde{R}^{'}\]
and 
\[\left\|\frac{\partial \mathcal{L}(\sigma_3(w_r(t)^T y_{j_1}))}{\partial w_r}\right\|_2\lesssim  B_1B_2.\]

Thus, we have
\begin{equation}
\begin{aligned}
&\left\| \frac{\partial s^{t}(u_i)(y_{j_1})}{\partial w_r}\right\|_2 \\
&\leq \frac{1}{\sqrt{n_2m}} \left|\left[\frac{1}{\sqrt{p}}\sum\limits_{k=1}^p\tilde{a}_{rk} \sigma(\tilde{w}_{rk}(t)^T u_i)\right]-\left[\frac{1}{\sqrt{p}}\sum\limits_{k=1}^p\tilde{a}_{rk} \sigma(\tilde{w}_{rk}(0)^T u_i)\right] \right|\left\|\frac{\partial \mathcal{L}(\sigma_3(w_r(t)^T y_{j_1}))}{\partial w_r}\right\|_2 \\
&\quad +\frac{1}{\sqrt{n_2m}} \left|\frac{1}{\sqrt{p}}\sum\limits_{k=1}^p\tilde{a}_{rk} \sigma(\tilde{w}_{rk}(0)^T u_i) \right|\left\|\frac{\partial \mathcal{L}(\sigma_3(w_r(t)^T y_{j_1}))}{\partial w_r}\right\|_2 \\
&\lesssim \frac{\sqrt{p}R^{'}B_1B_2}{\sqrt{mn_2}}+\frac{B_1B_2}{\sqrt{mn_2}}\sqrt{\log\left(\frac{mn_1}{\delta}\right)}\\
&\lesssim \frac{B_1B_2}{\sqrt{mn_2}}\sqrt{\log\left(\frac{mn_1}{\delta}\right)},
\end{aligned}
\end{equation}
where in the last inequality, we assume that $\sqrt{p}\tilde{R}^{'}\lesssim \sqrt{\log(mn_1/\delta)}$.

Similarly, since 
\[\left\|\frac{\partial\sigma_3(w_r(t)^T \tilde{y}_{j_2})}{\partial w_r}\right\|_2\lesssim  B_2^2,\]
we can obtain that 
\begin{equation}
\begin{aligned}
\left\| \frac{\partial h^{t}(u_i)(\tilde{y}_{j_2})}{\partial w_r}\right\|_2 \lesssim \frac{\sqrt{p}R^{'}B_2^2}{\sqrt{mn_3}}+\frac{B_2^2}{\sqrt{mn_3}}\sqrt{\log\left(\frac{mn_1}{\delta}\right)}.
\end{aligned}
\end{equation}

Combining (79), (80) and (81) yields that
\begin{equation}
\begin{aligned}
\|w_{r}(t+1)-w_{r}(t)\|_2&\lesssim \frac{\eta\sqrt{n_1}B_1B_2}{\sqrt{m}} \sqrt{\log\left(\frac{mn_1}{\delta}\right)}\sqrt{\sum\limits_{i=1}^{n_1} \|s(u_i)\|_2^2}+\frac{\eta\sqrt{n_1}B_2^2}{\sqrt{m}}\sqrt{\log\left(\frac{mn_1}{\delta}\right)} \sqrt{\sum\limits_{i=1}^{n_1} \|h(u_i)\|_2^2} \\
&\lesssim \frac{\eta\sqrt{n_1}B_1B_2}{\sqrt{m}}\sqrt{\log\left(\frac{mn_1}{\delta}\right)} \|G^{t}(u)\|_2,	
\end{aligned}
\end{equation}
where the first inequality follows from Cauchy's inequality.

Summing $t$ from $0$ to $T$ yields that
\begin{equation}
\begin{aligned}
\|w_r(T+1)-w_r(0)\|_2 &\leq \sum\limits_{t=1}^{T} \|w_r(t+1)-w_r(t)\|_2 \\
&\lesssim \sum\limits_{t=1}^{T-1}\frac{\eta\sqrt{n_1}B_1^2B_2}{\sqrt{mp}}\sqrt{\log\left(\frac{mn_1}{\delta}\right)} \left(1-\frac{\eta(\lambda_0+\tilde{\lambda}_0)}{2} \right)^{t/2} \|G^{0}(u)\|_2\\
&\lesssim   \frac{\sqrt{n_1}B_1B_2\|G^{0}(u)\|_2}{\sqrt{mp}(\lambda_0+\tilde{\lambda}_0)}\sqrt{\log\left(\frac{mn_1}{\delta}\right)}.
\end{aligned}
\end{equation}

\end{proof}

\subsection{Proof of Lemma 10}
\begin{proof}
From the form of the residual $I(t)$, it suffices to estimate 
\[s^{t+1}(u_i)(y_j)-s^{t}(u_i)(y_j) -\left\langle \frac{\partial s^{t}(u_i)(y_j) }{\partial w} , w(t+1)-w(t)\right\rangle -\left\langle \frac{\partial s^{t}(u_i)(y_j) }{\partial \tilde{w}} , \tilde{w}(t+1)-\tilde{w}(t)\right\rangle\]
and
\[h^{t+1}(u_i)(\tilde{y}_j)-h^{t}(u_i)(\tilde{y}_j) -\left\langle \frac{\partial h^{t}(u_i)(\tilde{y}_j) }{\partial w} , w(t+1)-w(t)\right\rangle -\left\langle \frac{\partial h^{t}(u_i)(\tilde{y}_j) }{\partial \tilde{w}} , \tilde{w}(t+1)-\tilde{w}(t)\right\rangle,\]
which we denote by $I_{i,j}^{s}(t)$ and $I_{i,j}^{h}(t)$, respectively. In fact, we only need to estimate $I_{i,j}^{s}(t)$, since $\mathcal{L}(u)$ includes the term $u$, which is the same as the boundary term.

Recall that the shallow neural operator has the form
\[G(u)(y) = \frac{1}{\sqrt{m}}\sum\limits_{r=1}^m \left[\frac{1}{\sqrt{p}} \sum\limits_{k=1}^p \tilde{a}_{rk} \sigma(\tilde{w}_{rk}^T u ) \right]\sigma_3(w_r^T y).\]

We first estimate $I_{i,j}^{s}(t)$. We can explicitly express the difference $s^{t+1}(u_i)(y_j)-s^{t}(u_i)(y_j)$ as follows:
\begin{equation}
\begin{aligned}
&s^{t+1}(u_i)(y_j)-s^{t}(u_i)(y_j)\\
&= \frac{1}{\sqrt{n_2}} \left[\mathcal{L}G^{t+1}(u_i)(y_j)-\mathcal{L}G^t(u_i)(y_j) \right] \\
&= \frac{1}{\sqrt{n_2}}\frac{1}{\sqrt{m}}\sum\limits_{r=1}^m \left(\left[\frac{1}{\sqrt{p}} \sum\limits_{k=1}^p \tilde{a}_{rk} \sigma(\tilde{w}_{rk}(t+1)^T u_i ) \right] \mathcal{L}(\sigma_3(w_r(t+1)^T y_j))   \right. \\
&\quad \left.-\left[\frac{1}{\sqrt{p}} \sum\limits_{k=1}^p \tilde{a}_{rk} \sigma(\tilde{w}_{rk}(t)^T u_i ) \right]\mathcal{L}(\sigma_3(w_r(t)^T y_j)) \right) \\
&= 	\frac{1}{\sqrt{n_2}}\frac{1}{\sqrt{m}}\sum\limits_{r=1}^m \left[\frac{1}{\sqrt{p}} \sum\limits_{k=1}^p \tilde{a}_{rk} (\sigma(\tilde{w}_{rk}(t+1)^T u_i )-\sigma(\tilde{w}_{rk}(t)^T u_i )) \right] \mathcal{L}(\sigma_3(w_r(t)^T y_j)) \\
&\quad+\frac{1}{\sqrt{n_2}}\frac{1}{\sqrt{m}}\sum\limits_{r=1}^m \left[\frac{1}{\sqrt{p}} \sum\limits_{k=1}^p \tilde{a}_{rk} \sigma(\tilde{w}_{rk}(t+1)^T u_i ) \right] \left[\mathcal{L}(\sigma_3(w_r(t+1)^T y_j))-\mathcal{L}(\sigma_3(w_r(t)^T y_j))\right], 
\end{aligned}
\end{equation}
where in the last equality, we split $s^{t+1}(u_i)(y_j)-s^{t}(u_i)(y_j)$ into two terms in order to estimate them separately later.

On the other hand, from the form of neural operator, we can obtain that
\begin{equation}
\begin{aligned}
&\left\langle \frac{\partial s^{t}(u_i)(y_j) }{\partial w} , w(t+1)-w(t)\right\rangle \\
&= \sum\limits_{r=1}^m \left\langle \frac{\partial s^{t}(u_i)(y_j) }{\partial w_r} , w_r(t+1)-w_r(t)\right\rangle \\
&=  \frac{1}{\sqrt{n_2}}\frac{1}{\sqrt{m}}\sum\limits_{r=1}^m \left[\frac{1}{\sqrt{p}} \sum\limits_{k=1}^p \tilde{a}_{rk} \sigma(\tilde{w}_{rk}(t)^T u_i ) \right] \left\langle\frac{\partial \mathcal{L}(\sigma_3(w_r(t)^T y_j)) }{\partial w_r} ,w_r(t+1)-w_r(t)\right\rangle ,
\end{aligned}
\end{equation}
and
\begin{equation}
\begin{aligned}
&\left\langle \frac{\partial s^{t}(u_i)(y_j) }{\partial \tilde{w} } , \tilde{w}(t+1)-\tilde{w}(t)\right\rangle \\
&= \sum\limits_{r=1}^m \sum\limits_{k=1}^p \left\langle \frac{\partial s^{t}(u_i)(y_j) }{\partial \tilde{w}_{rk} } , \tilde{w}_{rk}(t+1)-\tilde{w}_{rk}(t)\right\rangle \\
&=\frac{1}{\sqrt{n_2}}\frac{1}{\sqrt{m}} \sum\limits_{r=1}^m  \frac{1}{\sqrt{p}} 
\left[\sum\limits_{k=1}^p\tilde{a}_{rk}I\{\tilde{w}_{rk}(t)^T u_i\geq 0\} (\tilde{w}_{rk}(t+1)-\tilde{w}_{rk}(t))^T u_i\right] \mathcal{L}(\sigma_3(w_r(t)^T y_j)).
\end{aligned}
\end{equation}

With the explicit expressions for each term of $I_{i,j}^{s}(t)$, namely (84), (85), and (86), we can split $I_{i,j}^{s}(t)$ into two parts: the first part is the second term of (84) minus (86), and the second part is the first term of (84) minus (85). Specifically, let 
\begin{equation}
\begin{aligned}
I_1^r(t) &= \left( \frac{1}{\sqrt{p}} \sum\limits_{k=1}^p \tilde{a}_{rk} \left(\sigma(\tilde{w}_{rk}(t+1)^T u_i )-\sigma(\tilde{w}_{rk}(t)^T u_i )-I\{\tilde{w}_{rk}(t)^T u_i\geq 0\} (\tilde{w}_{rk}(t+1)-\tilde{w}_{rk}(t))^T u_i\right)\right) \\
&\quad \cdot \mathcal{L}(\sigma_3(w_r(t)^T y_j))
\end{aligned}
\end{equation}
and
\begin{equation}
\begin{aligned}
I_2^r(t) &=\left[\frac{1}{\sqrt{p}} \sum\limits_{k=1}^p \tilde{a}_{rk} \sigma(\tilde{w}_{rk}(t+1)^T u_i ) \right] \left[\mathcal{L}(\sigma_3(w_r(t+1)^T y_j))-\mathcal{L}(\sigma_3(w_r(t)^T y_j))\right] \\
&\quad -\left[\frac{1}{\sqrt{p}} \sum\limits_{k=1}^p \tilde{a}_{rk} \sigma(\tilde{w}_{rk}(t)^T u_i ) \right] \left\langle\frac{\partial \mathcal{L}(\sigma_3(w_r(t)^T y_j)) }{\partial w_r} ,w_r(t+1)-w_r(t)\right\rangle,
\end{aligned}
\end{equation}
where in the definition, we have omitted the indices $i,j,s$ for simplicity.

Then 
\begin{equation}
I_{i,j}^s(t)= \frac{1}{\sqrt{n_2}}\frac{1}{\sqrt{m}} \sum\limits_{r=1}^m \left[I_1^r(t)+I_2^r(t)\right].
\end{equation}

To estimate $I_1(r)$, since 
\[ |\mathcal{L}(\sigma_3(w_r(t)^T y_j))|\lesssim B_1^2B_2,\]
it suffices to estimate
\[\sigma(\tilde{w}_{rk}(t+1)^T u_i )-\sigma(\tilde{w}_{rk}(t)^T u_i )-I\{\tilde{w}_{rk}(t)^T u_i\geq 0\} (\tilde{w}_{rk}(t+1)-\tilde{w}_{rk}(t))^T u_i. \]

With a little abuse of notation, we let  $\tilde{A}_{r,k}^{i}=\{\tilde{w}\in \mathbb{R}^{d+1}: \|\tilde{w}-\tilde{w}_{rk}(0)\|\leq \tilde{R}^{'}, I\{ \tilde{w}^T u_i\geq 0\}  \neq I\{ \tilde{w}_{rk}(0)^T u_i\geq 0\}  \}$, $\tilde{S}_i =\{(r,k) \in [m]\times[q]: I\{\tilde{A}_{r,k}^{i}\}=0\}$
and $\tilde{S}_i^{\perp} = [m]\times[p] \backslash \tilde{S}_i$. Then we have that $P(\tilde{A}_{r,k}^{i})\lesssim \tilde{R}^{'}$. Note that $\|\tilde{w}_{rk}(t+1)-\tilde{w}_{rk}(0)\|_2\leq R^{'}, \|\tilde{w}_{rk}(t)-\tilde{w}_{rk}(0)\|_2\leq R^{'}$, thus for $(r,k) \in \tilde{S}_{i}$, we have $I\{\tilde{w}_{rk}(t+1)^T u_i\geq 0\}=I\{\tilde{w}_{rk}(t)^T u_i\geq 0\}$. At this point, 
\begin{equation}
\begin{aligned}
&\sigma(\tilde{w}_{rk}(t+1)^T u_i )-\sigma(\tilde{w}_{rk}(t)^T u_i ) -I\{\tilde{w}_{rk}(t)^T u_i\geq 0\} (\tilde{w}_{rk}(t+1)-\tilde{w}_{rk}(t))^T u_i \\
&= (\tilde{w}_{rk}(t+1)^T u_i ) I\{\tilde{w}_{rk}(t+1)^T u_i\geq 0\}-(\tilde{w}_{rk}(t)^T u_i ) I\{\tilde{w}_{rk}(t)^T u_i\geq 0\} \\
&\quad -I\{\tilde{w}_{rk}(t)^T u_i\geq 0\} (\tilde{w}_{rk}(t+1)-\tilde{w}_{rk}(t))^T u_i \\
&=  (\tilde{w}_{rk}(t+1)^T u_i ) I\{\tilde{w}_{rk}(t)^T u_i\geq 0\}-(\tilde{w}_{rk}(t)^T u_i ) I\{\tilde{w}_{rk}(t)^T u_i\geq 0\} \\
&\quad -I\{\tilde{w}_{rk}(t)^T u_i\geq 0\} (\tilde{w}_{rk}(t+1)-\tilde{w}_{rk}(t))^T u_i\\
&= 0.
\end{aligned}
\end{equation}
On the other hand, for all $(r,k)\in [m]\times[p]$,
\[|\sigma(\tilde{w}_{rk}(t+1)^T u_i )-\sigma(\tilde{w}_{rk}(t)^T u_i ) -I\{\tilde{w}_{rk}(t)^T u_i\geq 0\} (\tilde{w}_{rk}(t+1)-\tilde{w}_{rk}(t))^T u_i|\lesssim \|\tilde{w}_{rk}(t+1)-\tilde{w}_{rk}(t)\|_2. \]

Therefore, for $I_1^r(t)$, we have
\begin{align*}
	|I_1^r(t)|&\lesssim \frac{1}{\sqrt{p}}\sum\limits_{k=1}^{p} B_1^2B_2\|\tilde{w}_{rk}(t+1)-\tilde{w}_{rk}(t)\|_2 I\{(r,k)\in \tilde{S}_i^{\perp}\} .
\end{align*}

Combining with (77) yields that 
\begin{equation}
\begin{aligned}
\frac{1}{\sqrt{m}} \sum\limits_{r=1}^m |I_1^r(t)|&\lesssim \frac{1}{\sqrt{m}}\frac{1}{\sqrt{p}} \sum\limits_{r=1}^m \sum\limits_{k=1}^{p} B_1B_2\|\tilde{w}_{rk}(t+1)-\tilde{w}_{rk}(t)\|_2 I\{(r,k)\in \tilde{S}_i^{\perp}\} \\
&=\frac{1}{\sqrt{m}}\frac{1}{\sqrt{p}} \sum\limits_{r=1}^m \sum\limits_{k=1}^{p} B_1B_2\|\tilde{w}_{rk}(t+1)-\tilde{w}_{rk}(t)\|_2 I\{ \tilde{A}_{r,k}^{i}\} \\
&\lesssim \eta \sqrt{n_1} B_1^4B_2^2 \|G^{t}(u)\|_2 \frac{1}{mp}\sum\limits_{r=1}^m \sum\limits_{k=1}^{p} I\{ \tilde{A}_{r,k}^{i}\}\\
&\lesssim \eta \sqrt{n_1} B_1^4B_2^2 \|G^{t}(u)\|_2 \tilde{R}^{'}\\
&\lesssim \frac{\eta n_1 B_1^6B_2^3 \|G^{0}(u)\|_2}{\sqrt{mp}(\lambda_0+\tilde{\lambda}_0)}\|G^{t}(u)\|_2,
\end{aligned}
\end{equation}
where the last inequality follows from the Bernstein inequality and holds with probability at least $1-n_1\exp(-mp\tilde{R}^{'})$.

For $I_2^{r}(t)$, we can rewrite it as follows:
\begin{equation}
\begin{aligned}
I_2^r(t)&= \left[\frac{1}{\sqrt{p}} \sum\limits_{k=1}^p \tilde{a}_{rk} \sigma(\tilde{w}_{rk}(t+1)^T u_i ) \right] \left[\mathcal{L}(\sigma_3(w_r(t+1)^T y_j))-\mathcal{L}(\sigma_3(w_r(t)^T y_j))\right] \\
&\quad-\left[\frac{1}{\sqrt{p}} \sum\limits_{k=1}^p \tilde{a}_{rk} \sigma(\tilde{w}_{rk}(t)^T u_i ) \right] \left\langle \frac{\partial \mathcal{L}(\sigma_3(w_r(t)^T y_j)) }{\partial w_r}, w_r(t+1)-w_r(t)\right\rangle \\
&= \left[\frac{1}{\sqrt{p}} \sum\limits_{k=1}^p \tilde{a}_{rk} \left[\sigma(\tilde{w}_{rk}(t+1)^T u_i )-\sigma(\tilde{w}_{rk}(0)^T u_i )\right] \right]\left[\mathcal{L}(\sigma_3(w_r(t+1)^T y_j))-\mathcal{L}(\sigma_3(w_r(t)^T y_j))\right] \\
&-\left[\frac{1}{\sqrt{p}} \sum\limits_{k=1}^p \tilde{a}_{rk} \left[\sigma(\tilde{w}_{rk}(t)^T u_i )-\sigma(\tilde{w}_{rk}(0)^T u_i )\right]\right] \left\langle \frac{\partial \mathcal{L}(\sigma_3(w_r(t)^T y_j)) }{\partial w_r}, w_r(t+1)-w_r(t)\right\rangle \\
&+\left[ \mathcal{L}(\sigma_3(w_r(t+1)^T y_j))-\mathcal{L}(\sigma_3(w_r(t)^T y_j))- \left\langle \frac{\partial \mathcal{L}(\sigma_3(w_r(t)^T y_j)) }{\partial w_r}, w_r(t+1)-w_r(t)\right\rangle \right]\\
&\quad \cdot \left[\frac{1}{\sqrt{p}} \sum\limits_{k=1}^p \tilde{a}_{rk} \sigma(\tilde{w}_{rk}(0)^T u_i )\right] \\
&:=I_{2,1}^r(t)+I_{2,2}^r(t)+I_{2,3}^r(t).
\end{aligned}
\end{equation}

In the following, we will estimate the three items $I_{2,1}^r(t),I_{2,2}^r(t)$ and $I_{2,3}^r(t)$ separately.

For $I_{2,1}^r(t)$ and $I_{2,2}^r(t)$, note that for $s=t,t+1$, we have
\begin{equation}
\left|\left[\frac{1}{\sqrt{p}} \sum\limits_{k=1}^p \tilde{a}_{rk} \sigma(\tilde{w}_{rk}(s)^T u_i ) \right]-\left[\frac{1}{\sqrt{p}} \sum\limits_{k=1}^p \tilde{a}_{rk} \sigma(\tilde{w}_{rk}(0)^T u_i ) \right] \right|\lesssim \sqrt{p}\tilde{R}^{'}.	
\end{equation}
Moreover, we can deduce that
\begin{equation}
\left| \mathcal{L}(\sigma_3(w_r(t+1)^T y_j))-\mathcal{L}(\sigma_3(w_r(t)^T y_j))\right|\lesssim (B_1^2+B_1B_2) \|w_r(t+1)-w_r(t)\|_2\lesssim B_1^2 \|w_r(t+1)-w_r(t)\|_2
\end{equation}
and 
\begin{equation}
\left\|\frac{\partial \mathcal{L}(\sigma_3(w_r(t)^T y_j)) }{\partial w_r} \right\|_2\lesssim B_1B_2.
\end{equation}

Combining (93), (94) and (95) yields that
\begin{equation}
|I_{2,1}^r(t)|\lesssim \sqrt{p}R^{'}B_1^2\|w_r(t+1)-w_r(t)\|_2
\end{equation}
and
\begin{equation}
	|I_{2,2}^r(t)|\lesssim \sqrt{p}R^{'}B_1B_2\|w_r(t+1)-w_r(t)\|_2\lesssim \sqrt{p}R^{'}B_1^2\|w_r(t+1)-w_r(t)\|_2.
\end{equation}

It remains to estimate $I_{2,3}^r(t)$. From its form, it suffices to estimate 
\begin{equation}
\begin{aligned}
&\mathcal{L}(\sigma_3(w_r(t+1)^T y_j))-\mathcal{L}(\sigma_3(w_r(t)^T y_j))-\left\langle \frac{\partial \mathcal{L}(\sigma_3(w_r(t)^T y_j)) }{\partial w_r}, w_r(t+1)-w_r(t) \right\rangle\\
&:= \tilde{I}_1(t)+\tilde{I}_2(t)+\tilde{I}_3(t),
\end{aligned}
\end{equation}
where  $\tilde{I}_1(t), \tilde{I}_2(t)$ and $\tilde{I}_3(t)$ are respectively related to the first-order term, the second-order term, and the zeroth-order term of the PDE. Sepecifically, 
\[\tilde{I}_1(t)=w_{r0}(t+1)\sigma_3^{'}(w_r(t+1)^Ty_j)-w_{r0}(t)\sigma_3^{'}(w_r(t)^Ty_j)-\left\langle \frac{\partial w_{r0}(t)\sigma_3^{'}(w_r(t)^Ty_j) }{\partial w_r}, w_r(t+1)-w_r(t) \right\rangle,\]
\begin{align*}
\tilde{I}_2(t)&=-\left[ \|w_{r1}(t+1)\|_2^2\sigma_3^{''}(w_r(t+1)^T y_j)-\|w_{r1}(t)\|_2^2\sigma_3^{''}(w_r(t)^T y_j) \right. \\
&\quad \left. -\left\langle \frac{\partial \|w_{r1}(t)\|_2^2\sigma_3^{''}(w_r(t)^T y_j) }{\partial w_r}, w_r(t+1)-w_r(t) \right\rangle \right]
\end{align*}
and
\[\tilde{I}_{3}(t)=\sigma_3(w_r(t+1)^Ty_j)-\sigma_3(w_r(t)^Ty_j)-\left\langle \frac{\partial\sigma_3(w_r(t)^Ty_j) }{\partial w_r}, w_r(t+1)-w_r(t) \right\rangle.\]

Note that both $\tilde{I}_1(t)$ and $\tilde{I}_2(t)$ have the form 
\[ f(w^{'})g(w^{'})-f(w)g(w)-\left\langle\frac{\partial f(w)g(w)}{\partial w}, w^{'}-w\right\rangle.\]

However, due to the non-differentiability of the ReLU function, we cannot perform a second-order expansion; instead, we decompose it into

\begin{equation}
\begin{aligned}
&f(w^{'})g(w^{'})-f(w)g(w)-\left\langle\frac{\partial f(w)g(w)}{\partial w}, w^{'}-w \right\rangle \\
&=f(w^{'})g(w^{'})-f(w)g(w)-\left\langle\frac{\partial f(w)}{\partial w} g(w), w^{'}-w \right\rangle -\left\langle\frac{\partial g(w)}{\partial w} f(w), w^{'}-w \right\rangle \\
&=  f(w^{'}) \left[ g(w^{'})-g(w)- \left\langle\frac{\partial g(w)}{\partial w} , w^{'}-w \right\rangle\right]+g(w)\left[f(w^{'})-f(w)-\left\langle\frac{\partial f(w)}{\partial w} , w^{'}-w \right\rangle \right] \\
&\quad + [f(w^{'})-f(w)]\left\langle \frac{\partial g(w)}{\partial w}  ,w^{'}-w\right\rangle.
\end{aligned}
\end{equation}

Thus, for $\tilde{I}_1(t)$, we have
\begin{equation}
\begin{aligned}
\tilde{I}_1(t)&= w_{r0}(t+1)\sigma_3^{'}(w_r(t+1)^Ty_j)-w_{r0}(t)\sigma_3^{'}(w_r(t)^Ty_j) \\
&\quad -[\sigma_3^{'}(w_r(t)^T y_j)(w_{r0}(t+1)-w_{r0}(t))+w_{r0}(t)\sigma_3^{''}(w_r(t)^T y_j) (w_{r}(t+1)-w_{r}(t))^T y_{j} ] \\
&= w_{r0}(t+1)[\sigma_3^{'}(w_r(t+1)^Ty_j)-\sigma_3^{'}(w_r(t)^Ty_j)] -w_{r0}(t)\sigma_3^{''}(w_r(t)^T y_j) (w_{r}(t+1)-w_{r}(t))^T y_{j}\\
&=[w_{r0}(t+1)-w_{r0}(t)][\sigma_3^{'}(w_r(t+1)^Ty_j)-\sigma_3^{'}(w_r(t)^Ty_j)]\\
&\quad + w_{r0}(t)[\sigma_3^{'}(w_r(t+1)^Ty_j)-\sigma_3^{'}(w_r(t)^Ty_j)-\sigma_3^{''}(w_r(t)^T y_j) (w_{r}(t+1)-w_{r}(t))^T y_{j}    ].
\end{aligned}
\end{equation}

We apply mean value theorem for the first term in (100) and obtain that
\begin{align*}
	|\sigma_3^{'}(w_r(t+1)^Ty_j)-\sigma_3^{'}(w_r(t)^Ty_j)|&=|w_r(t+1)^Ty_j-w_r(t)^Ty_j| |\sigma_3^{''}(\xi)|\\
	&\lesssim B_2\|w_r(t+1)-w_r(t)\|_2.
\end{align*}

Similarly, for the second term, we have
\begin{align*}
&|\sigma_3^{'}(w_r(t+1)^Ty_j)-\sigma_3^{'}(w_r(t)^Ty_j)-\sigma_3^{''}(w_r(t)^T y_j) (w_{r}(t+1)-w_{r}(t))^T y_{j}|\\
&=|\sigma_3^{''}(\xi)-\sigma_3^{''}(w_r(t)^T y_j)| |(w_{r}(t+1)-w_{r}(t))^T y_{j}| \\
&\lesssim \|w_r(t+1)-w_r(t)\|_2^2.	
\end{align*}

Thus, for $\tilde{I}_1(t)$, we have
\begin{equation}
\begin{aligned}
	|\tilde{I}_1(t)|&\lesssim B_2\|w_r(t+1)-w_r(t)\|_2^2+ B_1 \|w_r(t+1)-w_r(t)\|_2^2 \\
	&\lesssim B_1 \|w_r(t+1)-w_r(t)\|_2^2.
\end{aligned}
\end{equation}

For $\tilde{I}_2(t)$, with same decomposition in (99), we have
\begin{equation}
\begin{aligned}
\tilde{I}_2(t)&= \|w_{r1}(t+1)\|_2^2\sigma_3^{''}(w_r(t+1)^T y_j)-\|w_{r1}(t)\|_2^2\sigma_3^{''}(w_r(t)^T y_j) \\
&\quad -[ 2\sigma_3^{''}(w_r(t)^T y_j)(w_{r1}(t+1)-w_{r1}(t))^T w_{r1}(t)  +\|w_{r1}(t)\|_2^2 \sigma_3^{'''}(w_r(t)^T y_j)(w_r(t+1)-w_r(t))^T y_j ]\\
&= \|w_{r1}(t+1)\|_2^2 \left[ \sigma_3^{''}(w_r(t+1)^T y_j)-\sigma_3^{''}(w_r(t)^T y_j)-\sigma_3^{'''}(w_r(t+1)^T y_j)(w_r(t+1)-w_r(t))^T y_j\right] \\
&\quad + \|w_{r1}(t)\|_2^2\left[ \|w_{r1}(t+1)\|_2^2- \|w_{r1}(t)\|_2^2-2w_{r1}(t)^T (w_{r1}(t+1)-w_{r1}(t)) \right] \\
&\quad +  [\|w_{r1}(t+1)\|_2^2-\|w_{r1}(t)\|_2^2]\sigma_3^{'''}(w_r(t)^T y_j)(w_r(t+1)-w_r(t))^T y_j.
\end{aligned}
\end{equation}

For the first term, similar to (90), for $r\in S_j$, we have
\[ \sigma(w_r(t+1)^T y_j)-\sigma(w_r(t)^T y_j)-I\{w_r(t+1)^T y_j\geq 0\}(w_r(t+1)-w_r(t))^T y_j=0,\]
thus
\[ |\sigma(w_r(t+1)^T y_j)-\sigma(w_r(t)^T y_j)-I\{w_r(t+1)^T y_j\geq 0\}(w_r(t+1)-w_r(t))^T y_j|\lesssim \|w_r(t+1)-w_r(t)\|_2  I\{r\in S_j^{\perp}\}.\]

For the second term, the mean value theorem yields that
\[| \|w_{r1}(t+1)\|_2^2- \|w_{r1}(t)\|_2^2-2w_{r1}(t)^T (w_{r1}(t+1)-w_{r1}(t))|\lesssim \|w_{r1}(t+1)-w_{r1}(t)\|_2 .\]

For the third term, we have
\[|\|w_{r1}(t+1)\|_2^2-\|w_{r1}(t)\|_2^2]\sigma_3^{'''}(w_r(t)^T y_j)(w_r(t+1)-w_r(t))^T y_j|\lesssim B_1 \|w_r(t+1)-w_r(t)\|_2^2.\]

Combining these result for $\tilde{I}_1(t), \tilde{I}_2(t)$ and $\tilde{I}_3(t)$ yields that 
\begin{equation}
\begin{aligned}
|I_{2,3}^r(t)|&\leq \left|\left[\frac{1}{\sqrt{p}} \sum\limits_{k=1}^p \tilde{a}_{rk} \sigma(\tilde{w}_{rk}(0)^T u_i ) \right]\right|[|\tilde{I}_1(t)|+|\tilde{I}_2(t)|+|\tilde{I}_3(t)|]  \\
&\lesssim \sqrt{\log\left(\frac{mn_1}{\delta} \right) }\left[B_1^2\| w_r(t+1)-w_r(t)\|_2^2+B_1^2\|w_r(t+1)-w_r(t)\|_2 I\{r\in S_j^{\perp}\} \right].
\end{aligned}
\end{equation}

With these estimations for $I_{2,1}^r(t),I_{2,2}^r(t)$ and $I_{2,3}^r(t)$, i.e., (96), (97) and (98), we obtain that
\begin{equation}
\begin{aligned}
|I_2^r(t)|&\leq |I_{2,1}^r(t)|+|I_{2,2}^r(t)|+|I_{2,3}^r(t)|\\
&\lesssim \sqrt{p}\tilde{R}^{'}B_1^2\|w_r(t+1)-w_r(t)\|_2\\
&\quad +\sqrt{\log\left(\frac{mn_1}{\delta} \right) }\left[B_1^2\| w_r(t+1)-w_r(t)\|_2^2+B_1^2\|w_r(t+1)-w_r(t)\|_2 I\{r\in S_j^{\perp}\} \right].
\end{aligned}
\end{equation}

Recall that (82) shows that
\[\|w_r(t+1)-w_r(t)\|_2\lesssim  \frac{\eta\sqrt{n_1}B_1B_2}{\sqrt{m}}\sqrt{\log\left(\frac{mn_1}{\delta}\right)} \|G^{t}(u)\|_2\]
and
\[\tilde{R}^{'}=\frac{\sqrt{n_1}B_1^2B_2 \|G^{0}(u)\|_2}{\sqrt{mp}(\lambda_0+\tilde{\lambda}_0)}, \ R^{'}=\frac{\sqrt{n_1}B_1B_2\|G^{0}(u)\|_2}{\sqrt{m}(\lambda_0+\tilde{\lambda}_0)} \sqrt{\log\left( \frac{mn_1}{\delta}\right)}.\]

Therefore, combining with the use of Bernstein inequality, summing $r$ yields that
\begin{equation}
\begin{aligned}
\frac{1}{\sqrt{m}}\sum\limits_{r=1}^m |I_2^r(t)|&\lesssim \frac{\eta n_1B_1^5B_2^2\|G^{0}(u)\|_2}{\sqrt{m}(\lambda_0+\tilde{\lambda}_0)}\sqrt{\log\left( \frac{mn_1}{\delta}\right)} \|G^t(u)\|_2+\frac{\eta n_1B_1^4B_2^2\|G^{0}(u)\|_2}{\sqrt{m}(\lambda_0+\tilde{\lambda}_0)}\log^{3/2}\left( \frac{mn_1}{\delta}\right) \|G^t(u)\|_2\\
&\lesssim \frac{\eta n_1B_1^5B_2^2\|G^{0}(u)\|_2}{\sqrt{m}(\lambda_0+\tilde{\lambda}_0)}\log^{3/2}\left( \frac{mn_1}{\delta}\right) \|G^t(u)\|_2.
\end{aligned} 
\end{equation}

Recall that (91) implies that
\[\frac{1}{\sqrt{m}}\sum\limits_{r=1}^m |I_1^r(t)|\lesssim \frac{\eta n_1 B_1^6B_2^3 \|G^{0}(u)\|_2}{\sqrt{mp}(\lambda_0+\tilde{\lambda}_0)}\|G^{t}(u)\|_2.\]

Therefore, we have
\begin{equation}
\begin{aligned}
\|I(t)\|_2&\lesssim \frac{\eta (n_1)^{3/2} B_1^6B_2^3 \|G^{0}(u)\|_2}{\sqrt{mp}(\lambda_0+\tilde{\lambda}_0)}\|G^{t}(u)\|_2+\frac{\eta  (n_1)^{3/2}B_1^5B_2^2\|G^{0}(u)\|_2}{\sqrt{m}(\lambda_0+\tilde{\lambda}_0)}\log^{3/2}\left( \frac{mn_1}{\delta}\right) \|G^t(u)\|_2\\
&\lesssim  \frac{\eta  (n_1)^{3/2}B_1^6B_2^3\|G^{0}(u)\|_2}{\sqrt{m}(\lambda_0+\tilde{\lambda}_0)}\log^{3/2}\left( \frac{mn_1}{\delta}\right) \|G^t(u)\|_2.
\end{aligned}
\end{equation}

\end{proof}

\subsection{Proof of Theorem 3}
\begin{proof}
It suffices to show that Condition 2 also holds for $s=T+1$. From the iteration formula in Lemma 9, we have
\begin{equation}
\begin{aligned}
\|G^{T+1}(u)\|_2^2 &= \|[I-\eta(H(T)+\tilde{H}(T))]G^{T}(u)+I(T)\|_2^2 \\
&= \|[I-\eta(H(T)+\tilde{H}(T))]G^{T}(u)\|_2^2+\|I(T)\|_2^2+2\left\langle [I-\eta(H(T)+\tilde{H}(T))]G^{T}(u),I(T)\right\rangle.
\end{aligned}
\end{equation}

From the stability of the Gram matrices, i.e., Lemma 8, when $R,\tilde{R}$ satisfy that
\[ R\lesssim \frac{\lambda_0}{n_1B_1^3B_2^2}, \ \tilde{R}\lesssim \frac{\tilde{ \lambda}_0}{n_1\sqrt{p}B_1^2B_2^2},\]
we have
\[ \|H(0)-H(T)\|_2\leq \frac{\lambda_0}{4}, \|\tilde{H}(0)-\tilde{H}(T)\|_2\leq \frac{\tilde{\lambda}_0}{4}.\]

Thus, with Lemma 7, we can deduce that 
\[ \|H(T)-H^{\infty}\|_2\leq \frac{\lambda_0}{2}, \|\tilde{H}(T)-\tilde{H}^{\infty}\|_2\leq \frac{\tilde{\lambda}_0}{2},\]
implying that $\lambda_{min}(H(T))\geq \lambda_0/2$ and $\lambda_{min}(\tilde{H}(T))\geq \tilde{\lambda}_0/2$. 

Therefore, when $\eta=\mathcal{O}(1/(\|H^{\infty}\|_2+\|\tilde{H}^{\infty}\|_2))$, we have that $I-\eta(H(T)+\tilde{H}(T))$ is positive definite and then
\begin{equation}
\|I-\eta(H(T)+\tilde{H}(T))\|_2\leq \frac{\eta(\lambda_0+\tilde{\lambda}_0)}{2}.
\end{equation}

Let $I(T)\leq \bar{R} \|G^{T}(u)\|_2$, then combining (107) and (108) yields that
\begin{equation}
\begin{aligned}
\|G^{T+1}(u)\|_2^2&\leq \|[I-\eta(H(T)+\tilde{H}(T))]G^{T}(u)\|_2^2+\|I(T)\|_2^2+2\|[I-\eta(H(T)+\tilde{H}(T))]G^{T}(u)\|_2 \|I(T)\|_2\\
&\leq \left[\left(1-\frac{\eta(\lambda_0+\tilde{\lambda}_0)}{2} \right)^2 +\bar{R}^2+2\bar{R}\left(1-\frac{\eta(\lambda_0+\tilde{\lambda}_0)}{2} \right)\right]\|G^{T}(u)\|_2^2\\
&\leq \left(1-\frac{\eta(\lambda_0+\tilde{\lambda}_0)}{2} \right)\|G^{T}(u)\|_2^2,
\end{aligned}
\end{equation} 
where the last inequality requires that $\bar{R}\lesssim \eta (\lambda_0+\tilde{\lambda}_0)$.

Finally, we need to specify the requirements for $m$ to ensure that the aforementioned conditions are satisfied. Recall that first, $m$ needs to satisfy that $R^{'}\leq R$ and $\tilde{R}^{'}\leq \tilde{R}$, i.e.,
\begin{equation}
\frac{\sqrt{n_1}B_1B_2\|G^{0}(u)\|_2}{\sqrt{m}(\lambda_0+\tilde{\lambda}_0)}\sqrt{\log\left(\frac{mn_1}{\delta}\right) }\lesssim \frac{\lambda_0}{n_1B_1^3B_2^2}, \ \frac{\sqrt{n_1}B_1^2B_2\|G^{0}(u)\|_2}{\sqrt{mp}(\lambda_0+\tilde{\lambda}_0)}\lesssim \frac{\tilde{\lambda}_0}{n_1\sqrt{p}B_1^2B_2^2}.
\end{equation}
Simple algebraic operations yield that
\begin{equation}
m =\Omega\left( \frac{n_1^3B_1^8B_2^6\|G^{0}\|_2^2}{(\lambda_0+\tilde{\lambda}_0)^2 min(\lambda_0^2, \tilde{\lambda}_0^2)}\log \left(\frac{mn_1}{\delta} \right) \right) .
\end{equation}

Second, $m$ needs to satisfy that $\bar{R}\lesssim \eta(\lambda_0+\tilde{\lambda}_0)$, i.e.,
\[ \frac{\eta  (n_1)^{3/2}B_1^6B_2^3\|G^{0}(u)\|_2}{\sqrt{m}(\lambda_0+\tilde{\lambda}_0)}\log^{3/2}\left( \frac{mn_1}{\delta}\right)\lesssim \eta(\lambda_0+\tilde{\lambda}_0) ,\]
implying that
\begin{equation}
m =\Omega\left( \frac{n_1^3B_1^{12}B_2^{6} \|G^{0}(u)\|_2^2 }{(\lambda_0+\tilde{\lambda}_0)^4}\log^3 \left(\frac{mn_1}{\delta} \right) \right) .
\end{equation}

Finally, combining (111), (112) and the estimation of $\|G^{0}(u)\|_2^2$ in Lemma 17, we have that
\[ m=\tilde{\Omega}\left( \frac{n_1^4d^7}{(\lambda_0+\tilde{\lambda}_0)^2 min(\lambda_0^2, \tilde{\lambda}_0^2)} \right),\]
where $\tilde{\Omega}$ indicates that some terms involving $\log(n_1)$, $\log(n_2)$ and $\log(m)$ are omitted. 

\end{proof}

\section{Auxiliary Lemmas}	

\begin{lemma}[Anti-concentration of Gaussian distribution]
Let $X \sim \mathcal{N}(0, \sigma^2)$, then for any $t>0$,
\[ \frac{2}{3} \frac{t}{\sigma} < P(|X|\leq t) < \frac{4}{5}\frac{t}{\sigma}.\]
\end{lemma}

\begin{lemma}[Bernstein inequality, Theorem 3.1.7 in \cite{13}]
Let $X_i$, $1\leq i \leq n$ be independent centered random variables a.s. bounded by $c<\infty$ in absolute value. Set $
\sigma^2=1/n\sum_{i=1}^n \mathbb{E}X_i^2$ and $S_n=1/n \sum_{i=1}^n X_i$. Then, for all $t\geq 0$,
\[P\left(S_n\geq \sqrt{\frac{2\sigma^2 t}{n}} +\frac{ct}{3n}\right)\leq e^{-u}.\]
\end{lemma}

First, we provide some preliminaries about Orlicz norms. 

Let $g:[0,\infty) \rightarrow [0,\infty)$ be a non-decreasing convex function with $g(0)=0$. The $g$-Orlicz norm of a real-valued random variable $X$ is given by 
\[\|X\|_{g}:=\inf \left\{C>0:\mathbb{E}\left[ g\left(\frac{|X|}{C}\right) \right]\leq 1\right\}.\]
If $\|X\|_{\psi_{\alpha}}<\infty$, we say that $X$ is sub-Weibull of order $\alpha>0$, where
\[\psi_{\alpha}(x):=e^{x^{\alpha}}-1.\]
Note that when $\alpha\geq 1$, $\|\cdot\|_{\psi_{\alpha}}$ is a norm and when $0< \alpha< 1$, $\|\cdot\|_{\psi_{\alpha}}$ is a quasi-norm. In the related proofs, we may frequently use the fact that for real-valued random variable $X\sim \mathcal{N}(0,1)$, we have $\|X\|_{\psi_2} \leq \sqrt{6}$ and $\| X^2\|_{\psi_1}=\|X\|_{\psi_2}^2 \leq 6$. Moreover, when $\|X\|_{\psi_2}<\infty, \|Y\|_{\psi_2}<\infty$, we have $\|XY\|_{\psi_1}\leq \|X\|_{\psi_2}\|Y\|_{\psi_2}$. Since without loss of generality, we can assume that $\|X\|_{\psi_2}=\|Y\|_{\psi_2}=1$, then
\[ \mathbb{E}\left[e^{|XY| }\right]\leq \mathbb{E}\left[e^{\frac{|X|^2}{2}+\frac{|Y|^2}{2} } \right]=\mathbb{E}\left[e^{\frac{|X|^2}{2}}e^{\frac{|Y|^2}{2} } \right]  \leq \frac{1}{2}\mathbb{E} \left[ e^{|X|^2}+e^{|Y|^2 } \right]\leq 1,\]
where the first inequality and the second inequality follow from the inequality $2ab\leq a^2+b^2$ for $a\geq 0, b\geq 0$.

\begin{lemma}[Theorem 3.1 in \cite{11}]
If $X_1,\cdots, X_n$ are independent mean zero random variables with $\|X_i\|_{\psi_{\alpha}}<\infty$ for all $1\leq i\leq n$ and some $\alpha>0$, then for any vector $a=(a_1,\cdots,a_n)\in \mathbb{R}^n$, the following holds true:
\[P\left(\left|\sum\limits_{i=1}^n a_i X_i \right|\geq 2eC(\alpha)\|b\|_2\sqrt{t}+2eL_n^{*}(\alpha)t^{1/\alpha}\|b\|_{\beta(\alpha)}  \right)\leq 2e^{-t}, \ for \ all \ t\geq 0,\]
where $b=(a_1\|X_1\|_{\psi_{\alpha}},\cdots, a_n\|X_n\|_{\psi_{\alpha}}) \in \mathbb{R}^n$,
\begin{equation*}
C(\alpha):=\max\{\sqrt{2},2^{1/\alpha}\}\left\{
\begin{aligned}
\sqrt{8}(2\pi)^{1/4}e^{1/24}(e^{2/e}/\alpha)^{1/\alpha} & , & if \ \alpha<1, \\
4e+2(\log 2)^{1/\alpha} & , & if \ \alpha\geq 1.
\end{aligned}
\right.
\end{equation*}
and for $\beta(\alpha)=\infty$ when $\alpha \leq 1$ and $\beta(\alpha)=\alpha/(\alpha-1)$ when $\alpha>1$,
\begin{equation*}
L_n(\alpha):=\frac{4^{1/\alpha}}{\sqrt{2}\|b\|_2}\times
\left\{
\begin{aligned}
&\|b\|_{\beta(\alpha)} , & if \ \alpha<1, \\
&4e \|b\|_{\beta(\alpha)}/C(\alpha), & if \alpha \geq 1.
\end{aligned}
\right.
\end{equation*}
and $L^{*}_n(\alpha)=L_n(\alpha)C(\alpha)\|b\|_2/\|b\|_{\beta(\alpha)}$.
\end{lemma}

\begin{lemma}
For any $r\in [m]$, we have that with probability at least $1-\delta$, 
\[ \left|\frac{1}{\sqrt{p}} \sum\limits_{k=1}^p \tilde{a}_{rk} \sigma(\tilde{w}_{rk}(0)^T u ) \right| \lesssim \sqrt{\log\left( \frac{1}{\delta}\right) } .\]
Moreover, its $\psi_2$-norm is a universal constant.
\end{lemma}

\begin{proof}
Note that $\mathbb{E}[\tilde{a}_{rk} \sigma(\tilde{w}_{rk}(0)^T u )]=0$ and $\|\tilde{a}_{rk} \sigma(\tilde{w}_{rk}(0)^T u )\|_{\psi_2}\leq \|  |\tilde{w}_{rk}(0)^T u|\|_{\psi_2}=\mathcal{O}(1)$, then applying Lemma 17 yields that with probability at least $1-\delta$,

\[ \left|\frac{1}{\sqrt{p}} \sum\limits_{k=1}^p \tilde{a}_{rk} \sigma(\tilde{w}_{rk}(0)^T u ) \right|\lesssim \sqrt{\log\left( \frac{1}{\delta}\right) } .  \]	

Moreover, from the equivalence of the $\psi_2$ norm and the concentration inequality (see Lemma 2.2.1 in \cite{12}), it follows that
\[ \left\|\frac{1}{\sqrt{p}} \sum\limits_{k=1}^p \tilde{a}_{rk} \sigma(\tilde{w}_{rk}(0)^T u ) \right \|_{\psi_2}=\mathcal{O}(1).\]
\end{proof}

\begin{lemma}
With probability at least $1-\delta$, we have
\[L(0)\lesssim  n_1n_2\log\left(\frac{n_1n_2}{\delta} \right).\]
\end{lemma}

\begin{proof}

\begin{align*}
L(0)&=\frac{1}{2}\sum\limits_{i=1}^{n_1}\sum\limits_{j=1}^{n_2} (G^{0}(u_i)(y_j)-z_j^i)^2 \\
&\leq \sum\limits_{i=1}^{n_1}\sum\limits_{j=1}^{n_2} (G^{0}(u_i)(y_j))^2+(z_j^i)^2,
\end{align*}
Recall that
\[G^{0}(u)(y) = \frac{1}{\sqrt{m}}\sum\limits_{r=1}^m \left[\frac{1}{\sqrt{p}} \sum\limits_{k=1}^p \tilde{a}_{rk} \sigma(\tilde{w}_{rk}(0)^T u ) \right]\sigma(w_r(0)^T y).\]

Note that 
\[ \left\|\frac{1}{\sqrt{p}} \sum\limits_{k=1}^p \tilde{a}_{rk} \sigma(\tilde{w}_{rk}(0)^T u ) \right\|_{\psi_2}=\mathcal{O}(1), \ \left\|\sigma(w_r(0)^T y_j) \right\|_{\psi_2}=\mathcal{O}(1) ,\]
thus
\[ \left\|\left[\frac{1}{\sqrt{p}} \sum\limits_{k=1}^p \tilde{a}_{rk} \sigma(\tilde{w}_{rk}(0)^T u ) \right]\sigma(w_r(0)^T y_j) \right\|_{\psi_1}=\mathcal{O}(1)\]

Then from Lemma 12, we have that with probability at least $1-\delta$,
\[L(0)\lesssim n_1n_2\left(\sqrt{\log\left(\frac{n_1n_2}{\delta}\right)} + \frac{\log(\frac{n_1n_2}{\delta})}{\sqrt{m}} \right)^2\lesssim n_1n_2\log\left(\frac{n_1n_2}{\delta} \right).\]	
\end{proof}	

\begin{lemma}
	If $\|X\|_{\psi_{\alpha}}, \|Y\|_{\psi_{\beta}}<\infty$ with $\alpha, \beta>0$, then we have $\|XY\|_{\psi_{\gamma}}\leq \|X\|_{\psi_{\alpha}}\|Y\|_{\psi_{\beta}}$, where $\gamma$ satisfies that
	\[ \frac{1}{\gamma} = \frac{1}{\alpha}+\frac{1}{\beta}.\] 
\end{lemma}

\begin{proof}
Without loss of generality, we can assume that $\|X\|_{\psi_{\alpha}}= \|Y\|_{\psi_{\beta}}=1$. To prove this, let us use Young’s inequality, which states that
\[ xy \leq \frac{x^p}{p}+\frac{y^q}{q}, for \ x, y \geq 0, p,q>1.\]

Let $p=\alpha/\gamma,q=\beta/\gamma$, then 
\begin{align*}
\mathbb{E} [\exp(|XY|^{\gamma})] &\leq \mathbb{E}\left[\exp\left(\frac{|X|^{\gamma p}}{p}+\frac{|Y|^{\gamma q}}{q} \right)\right] \\
&= \mathbb{E}\left[\exp\left(\frac{|X|^{\alpha}}{p} \right)\exp \left( \frac{|Y|^{\beta}}{q}\right) \right] \\
&\leq \mathbb{E}\left[ \frac{ \exp (|X|^{\alpha}) }{p}+ \frac{ \exp (|Y|^{\beta}) }{q}\right] \\
&\leq \frac{2}{p}+\frac{2}{q} \\
&= 2,
\end{align*}
where the first and second inequality follow from Young's inequality. From this, we have that  $\|XY\|_{\psi_{\gamma}}\leq \|X\|_{\psi_{\alpha}}\|Y\|_{\psi_{\beta}}$.

\end{proof}

\begin{lemma}
With probability at least $1-\delta$, we have
\[\|G^{0}(u)\|_2^2=L(W(0),\tilde{W}(0))\lesssim n_1d\log\left( \frac{n_1(n_2+n_3)}{\delta}\right).\]
\end{lemma}

\begin{proof}
Recall that the loss of function of PINN is 
\[ L(W, \tilde{W})=\sum\limits_{i=1}^{n_1} \sum\limits_{j_1=1}^{n_2} \frac{1}{n_2} (\mathcal{L}G(u_i)(y_{j_1})-f(y_{j_1}) )^2 +\sum\limits_{i=1}^{n_1} \sum\limits_{j_2=1}^{n_3} \frac{1}{n_3} (G(u_i)(\tilde{y}_{j_2})-g(\tilde{y}_{j_2}) )^2 \]
and the shallow neural operator has the following form
\[G(u)(y) = \frac{1}{\sqrt{m}}\sum\limits_{r=1}^m \left[\frac{1}{\sqrt{p}} \sum\limits_{k=1}^p \tilde{a}_{rk} \sigma(\tilde{w}_{rk}^T u) \right]\sigma_3(w_r^T y).\]

In order to estimate the initial value, it suffices to consider
\[\frac{1}{\sqrt{m}}\sum\limits_{r=1}^m \left[\frac{1}{\sqrt{p}} \sum\limits_{k=1}^p \tilde{a}_{rk}(0) \sigma(\tilde{w}_{rk}(0)^T u) \right]\mathcal{L}(\sigma_3(w_r(0)^T y))\]
and
 \[\frac{1}{\sqrt{m}}\sum\limits_{r=1}^m \left[\frac{1}{\sqrt{p}} \sum\limits_{k=1}^p \tilde{a}_{rk}(0) \sigma(\tilde{w}_{rk}(0)^T u) \right]\sigma_3(w_r(0)^T y).\]

Note that $|\mathcal{L}(\sigma_3(w_r(0)^T y))|\lesssim \|w_r(0)\|_2^2|w_r(0)^Ty|$, thus Lemma 20 implies that 
\[ \|\mathcal{L}(\sigma_3(w_r(0)^T y))\|_{\psi_{\frac{2}{3} } } \lesssim \| \|w_r(0)\|_2^2|w_r(0)^Ty| \|_{\psi_{\frac{2}{3} } }\leq \|\|w_r(0)\|_2^2\|_{\psi_1} \||w_r(0)^Ty|\|_{\psi_2}=\mathcal{O}(d).\]

Therefore, combining with Lemma 18 yields that 
\[ \left\|\left[\frac{1}{\sqrt{p}} \sum\limits_{k=1}^p \tilde{a}_{rk}(0) \sigma(\tilde{w}_{rk}(0)^T u) \right]\mathcal{L}(\sigma_3(w_r(0)^T y))\right\|_{\psi_{\frac{1}{2} }}\lesssim \mathcal{O}(d).\]

Similarly, we can deduce that 
\[ \left\|\left[\frac{1}{\sqrt{p}} \sum\limits_{k=1}^p \tilde{a}_{rk}(0) \sigma(\tilde{w}_{rk}(0)^T u) \right]\sigma_3(w_r(0)^T y)\right\|_{\psi_{\frac{1}{2} }}\lesssim \mathcal{O}(d).\]

Finally, applying Lemma 17 leads to that with probability at least $1-\delta$, 
\[\|G^{0}(u)\|_2^2=L(W(0),\tilde{W}(0))\lesssim n_1d\log\left( \frac{n_1(n_2+n_3)}{\delta}\right).\]
\end{proof}

\end{document}